\documentclass[twoside]{article}

\usepackage{blindtext}

\usepackage{amsfonts, amsmath, amssymb, amsthm}
\usepackage{yfonts} 
\usepackage{gensymb} % Has the \degree symbol.

\usepackage{physics}    % Has macros for differentials, derivatives, partial derivatives.

\usepackage{xpatch} 
\usepackage{mathtools} % Has the DeclarePairedDelimeter macro and maybe more.

\usepackage{etoc}

\usepackage{mathalpha}
\usepackage{bm}
\usepackage{verbatim}
\usepackage{euscript}
\usepackage{graphicx} % Amongs possibly others, used for short equal to sign.

\usepackage{paralist}
\usepackage{cancel} % to strike-out terms in an equation.

\usepackage{bbm} % Package for indicator random variable?! 

\usepackage{nicefrac}

\usepackage{thmtools}
\usepackage{thm-restate} % Restatable claim/lemma/theorems

\usepackage{enumitem}
\usepackage{graphicx}
\usepackage{soul}
\usepackage{float}

\usepackage{xspace}
\usepackage{wrapfig}

\usepackage{subcaption}
\usepackage{caption} % Custom captions for tables etc.

\usepackage[usenames,dvipsnames]{color}
\usepackage[dvipsnames]{xcolor}

\usepackage{csvsimple}
\usepackage{etoolbox}  % For \ifblank
\usepackage{pgf}

\usepackage[colorinlistoftodos]{todonotes}
\usepackage[colorlinks=true, allcolors=black]{hyperref} 
\hypersetup{
  colorlinks   = true, %Colours links instead of ugly boxes
  urlcolor     = blue, %Colour for external hyperlinks
  linkcolor    = blue, %Colour of internal links
  citecolor   = blue %Colour of citations, could be ``red''
}
\usepackage[noabbrev,capitalize,nameinlink]{cleveref}

\usepackage{natbib}

\usepackage{booktabs} % since you're using \toprule etc.
\usepackage{tabls} % This has \tablinesep=2ex that gives space between rows.
\usepackage{array}
\usepackage{longtable}

\usepackage[normalem]{ulem}
\usepackage{soul}  % Use as \st{abc} in text mode.

\usepackage{pifont}% http://ctan.org/pkg/pifont  % For ticks and crosses.

\usepackage{array} % For new columns specifiers in a Table.

\DeclareMathOperator*{\argmax}{arg\,max}

\DeclareMathSymbol{\shortminus}{\mathbin}{AMSa}{"39}
\makeatletter
\newcommand{\shorteq}{%
  \settowidth{\@tempdima}{--}% Width of hyphen
  \resizebox{\@tempdima}{\height}{=}%
}
\makeatother

\newcommand{\inprod}[2]{\left\langle #1,#2 \right\rangle}	
\newcommand{\cardinality}[1]{ \left\lvert #1 \right\rvert }

\newcommand{\assign}{\leftarrow}

\newcommand{\prob}[1]{\mathbb{P} \left\{ #1 \right\}}
\newcommand{\probover}[2]{\underset{#1}{\mathbb{P}} \left[ #2 \right]}
\newcommand{\probunder}[2]{\mathbb{P}_{#1} \left\{ #2 \right\}}

\newcommand{\given}{\,\middle|\,}

\newcommand{\expectationover}[2]{\underset{#1}{\mathbb{E}} \left[ #2 \right]}

\newcommand{\indicator}[1]{\mathbbm{1} \left\{ #1 \right\}}

\newcommand{\gaussian}[2]{\mathcal{N} \left(#1, #2\right)}

\newcommand{\bigO}[1]{O \left( #1 \right)}

\newcommand{\numberthis}{\addtocounter{equation}{1}\tag{\theequation}}

\newcommand{\thetahat}{\hat{\theta}}

\newtheorem{theorem}{Theorem}
\newtheorem{corollary}{Corollary}
\newtheorem{lemma}{Lemma}
\newtheorem{claim}[lemma]{Claim}
\crefname{claim}{Claim}{Claims}

\newtheorem{example}{Example}

\newtheorem{axiom}[lemma]{Axiom}
\crefname{axiom}{Axiom}{Axioms}

\newtheorem{definition}{Definition}

\crefname{assumption}{Assumption}{Assumptions}

\newtheorem{remark}{Remark}

\makeatletter
\newcommand{\saveequation}[2]{% #1 = label, #2 = math
  #2 \label{#1}
  \protected@write\@mainaux{}{\string\SAVEEQUATION{#1}{\unexpanded{\unexpanded{#2}}}}%
}
\newcommand{\SAVEEQUATION}[2]{%
  \global\@namedef{SAVEDEQUATION@#1}{#2}%
}
\newcommand{\repeatequation}[1]{%
  \ifcsname SAVEDEQUATION@#1\endcsname
    \@nameuse{SAVEDEQUATION@#1}\tag{\ref{#1}}%
  \else
    ?? \notag
  \fi
}
\makeatother

\newenvironment{remark-nolabel}{\noindent{\bf Remark}\hspace*{1em}}{\bigskip}

\newcolumntype{L}[1]{>{\raggedright\let\newline\\\arraybackslash\hspace{0pt}}m{#1}}
\newcolumntype{C}[1]{>{\centering\let\newline\\\arraybackslash\hspace{0pt}}m{#1}}
\newcolumntype{R}[1]{>{\raggedleft\let\newline\\\arraybackslash\hspace{0pt}}m{#1}}

\definecolor{rowgray}{RGB}{245,245,245} % very light gray
\definecolor{rowwhite}{RGB}{255,255,255}

\newcommand{\Real}{\mathbb{R}}

\long\def\comment#1{}

\newcounter{relctr} %% <- counter for relations
\everydisplay\expandafter{\the\everydisplay\setcounter{relctr}{0}} %% <- reset every eq
\newcommand\labelrel[2]{%
  \begingroup
    \refstepcounter{relctr}%
    \stackrel{\textnormal{(\alph{relctr})}}{\mathstrut{#1}}%
    \originallabel{#2}%
  \endgroup
}
\AtBeginDocument{\let\originallabel\label} %% <- store original definition

\newcommand{\handout}[6]{
   \renewcommand{\thepage}{\arabic{page}} % {#1-\arabic{page}}
   \noindent
   \begin{center}
   \framebox{
      \vbox{
    \hbox to 5.78in { {\bf #2}
     	 \hfill #3 }
       \vspace{4mm}
       \hbox to 5.78in { {\Large \hfill #6  \hfill} }
       \vspace{2mm}
       \hbox to 5.78in { {\it #4 \hfill #5} }
      }
   }
   \end{center}
   \vspace*{4mm}
}

\usepackage{algorithm}% http://ctan.org/pkg/algorithm
\usepackage{algpseudocode}% http://ctan.org/pkg/algorithmicx

\usepackage{etoolbox}
\usepackage{lipsum}

\makeatletter
\newcommand*{\algrule}[1][\algorithmicindent]{%
  \hspace*{.2em}% <------------- This is where the rule starts from
  \vrule %height .75\baselineskip depth .25\baselineskip
  \hspace*{\dimexpr#1-.2em-.4pt}%
}

\newcommand{\StatePar}[1]{%
  \State\parbox[t]{\dimexpr\linewidth-\ALG@thistlm}{\strut #1\strut}%
}
\renewcommand{\ALG@beginalgorithmic}{\offinterlineskip}% Remove all interline skips

\newcount\ALG@printindent@tempcnta
\def\ALG@printindent{%
  \ifnum \theALG@nested > 0% is there anything to print
    \ifx\ALG@text\ALG@x@notext% is this an end group without any text?
    \else
      \unskip
      \ALG@printindent@tempcnta=1
      \loop
        \algrule[\csname ALG@ind@\the\ALG@printindent@tempcnta\endcsname]%
        \advance \ALG@printindent@tempcnta 1
        \ifnum \ALG@printindent@tempcnta<\numexpr\theALG@nested+1\relax
      \repeat
        \fi
    \fi
}
\patchcmd{\ALG@doentity}{\noindent\hskip\ALG@tlm}{\ALG@printindent}{}{\errmessage{failed to patch}}
\makeatother

\algrenewcommand\algorithmicend{\strut\textbf{end}}
\algrenewcommand\algorithmicdo{\strut\textbf{do}}
\algrenewcommand\algorithmicwhile{\strut\textbf{while}}
\algrenewcommand\algorithmicfor{\strut\textbf{for}}
\algrenewcommand\algorithmicforall{\strut\textbf{for all}}
\algrenewcommand\algorithmicloop{\strut\textbf{loop}}
\algrenewcommand\algorithmicrepeat{\strut\textbf{repeat}}
\algrenewcommand\algorithmicuntil{\strut\textbf{until}}
\algrenewcommand\algorithmicprocedure{\strut\textbf{procedure}}
\algrenewcommand\algorithmicfunction{\strut\textbf{function}}
\algrenewcommand\algorithmicif{\strut\textbf{if}}
\algrenewcommand\algorithmicthen{\strut\textbf{then}}
\algrenewcommand\algorithmicelse{\strut\textbf{else}}

\algrenewcommand\algorithmicrequire{\strut\textbf{Input:}}
\algrenewcommand\algorithmicensure{\strut\textbf{Output:}}

\let\oldState\State
\renewcommand{\State}{\oldState\strut}

\usepackage[accepted]{aistats2026}
\usepackage{microtype}

\begin{document}

\runningtitle{Creator Incentives in Recommender Systems:
A Cooperative Game-Theoretic Approach}

\runningauthor{Krishnamurthy, Agarwal, Subramanian, and Nickel}

\twocolumn[

\aistatstitle{Creator Incentives in Recommender Systems:
A Cooperative Game-Theoretic Approach for Stable and Fair Collaboration in Multi-Agent Bandits}

\aistatsauthor{Ramakrishnan Krishnamurthy \And Arpit Agarwal}

\aistatsaddress{Courant Institute, New York University \And  Indian Institute of Technology Bombay}

\aistatsauthor{Lakshminarayanan Subramanian \And Maximilian Nickel}

\aistatsaddress{Courant Institute, New York University \And FAIR, Meta AI}
]

\newif\iftwocolumnmode 
\twocolumnmodetrue

% SHORT-HANDS.
\newcommand{\agents}{M}
\newcommand{\coalset}{\mathcal{C}}
\newcommand{\alg}{\textsc{Alg}}
\newcommand{\idenalg}{\textsc{Mul}}
\newcommand{\singalg}{\textsc{Sin}}
\newcommand{\muletc}{\textsc{M-ETC}}

\newcommand{\inst}{\mathcal{I}}
\newcommand{\instfix}{\inst^{\text{fixed}}}

\newcommand{\taumax}{\overline{\tau}}

\newcommand{\hist}{\mathcal{H}}
\newcommand{\lsalg}{\textsc{LsAlg}}

\newcommand{\pq}[1]{\mathbb{P}^{Q} \left\{ #1 \right\}}
\newcommand{\ps}[1]{\mathbb{P}^{S} \left\{ #1 \right\}}

% for experiments.
\newcommand{\mlu}{\textsc{LinUcb-M}}
\newcommand{\greedy}{\textsc{Greedy}}
\newcommand{\iocc}{I_{\text{occ}}}
\newcommand{\iloc}{I_{\text{geo}}}
\newcommand{\iage}{I_{\text{age}}}
\newcommand{\igen}{I_{\text{gen}}}

\begin{abstract}

User interactions in online recommendation platforms create interdependencies among content creators: feedback on one creator's content influences the system's learning and, in turn, the exposure of other creators' contents.
To analyze incentives in such settings, we model collaboration as a multi-agent stochastic linear bandit problem 
with a transferable utility (TU) cooperative game formulation, 
where a coalition's value equals the negative sum of its members' cumulative regrets.

We show that, for identical (homogenous) agents with fixed action sets, the induced TU game is convex under mild algorithmic conditions, implying a non-empty core that contains the Shapley value and ensures both stability and fairness.
For heterogeneous agents, the game still admits a non-empty core,
though convexity and Shapley value core-membership are no longer guaranteed.
To address this, we propose a simple regret-based payout rule that satisfies three out of the four Shapley axioms and also lies in the core. 
Experiments on MovieLens-100k dataset illustrate when the empirical payout aligns with---and diverges from---the Shapley fairness across different settings and algorithms. 
\end{abstract}

% INTRODUCTION.
\section{Introduction}
% Recommendation systems.
Collaborative learning has emerged as a powerful paradigm in machine learning, enabling multiple agents to improve overall task performance by 
% pooling
data or computational resources \citep{kairouz2021advances,zhang2021multi}.
This framework is particularly salient in large-scale online recommender systems, where interactions between users and content creators serve as implicit signals for model optimization.
These platforms leverage user feedback to jointly model user preferences and content characteristics, thereby learning latent structures that generalize across the population \citep{ko2022survey,zhang2019deep}.
% In a two-sided system with content creators and users, the platform recommends to users the content that is relevant and interesting to the user, and can collect user feedback on this recommendation, which reveals insights about both the content and the user.
In many of these systems, a creator's revenue is directly tied to user engagement metrics, which in turn depend on how well their content is recommended.
Consequently, the learning and reward dynamics are coupled---what the system learns from feedback on one creator's content can affect how another creator's content is recommended, and ultimately, monetized.
This introduces a subtle but critical interplay between data-sharing, learning dynamics, and economic outcomes for creators 
\citep{qian2024digital,hronmodeling,zhu2023online}.
To vividly illustrate this interdependence, consider the following simplified scenario: 
\begin{example}[Online Recommender Platform]
Let Alice and Bob be content creators, producing content on apples and bananas, respectively.
A new user visits the platform, and the system initially recommends Alice's content about apples. Two outcomes are possible:
\begin{enumerate}
    \item \textbf{Positive feedback.} The user engages positively with Alice's content.
    The system infers a preference for fruit-related topics and subsequently recommends Bob's content about bananas.
    \item \textbf{Negative feedback.} The user disengages or responds negatively. 
    The system infers a disinterest in fruit content and refrains from recommending Bob's content to the user.
\end{enumerate}
In both cases, Bob's reward hinges on how the platform generalizes the user's preference from Alice's content:
\begin{enumerate}
\item If the generalization aligns with the user's true interests, 
Bob either gains a satisfied user or avoids an unsatisfactory impression---both favourable.
\item However, if the generalization is inaccurate, Bob either suffers from unwarranted exposure or misses out on a potentially interested user.
\end{enumerate}
\end{example}
Thus, the system's inference from one creator's data directly impacts another's opportunity and reward, highlighting the entangled fates of collaborating agents in such platforms.

This simplified example abstracts away many practical complexities.
In real-world systems, creators may be \emph{homogeneous}---equally capable of producing various content types for the same user population---or \emph{heterogeneous}, with specialized content expertise and/or access to segmented audience.
Nonetheless, the example drives home the central research question:
\begin{quote}
\centering
\emph{
How does collaboration among learning agents impact each other?}
\\ 
% \; \\
\vspace{5pt}
\emph{
Is it possible to ensure collectively rational participation from all agents?
}
\end{quote}

\subsection{Our Contributions}

We model this learning scenario as a multi-agent collaborative bandit problem, 
which captures the exploration-exploitation trade-off (mirroring the tension between learning and revenue generation in recommendation system).
To quantify how an agent's learning affects others,
we model the inter-agent dynamics using a transferable utility (TU) coalition game,
where the value of a coalition reflects the collective regret reduction achieved through data sharing among the members of the coalition.
The formal setup is introduced in \cref{sec:setting}.

Using tools from cooperative game theory, we analyze how properties of the learning problems and algorithmic behaviour/assumptions influence coalition formations and collaborative incentives.
To this end, our key contributions are as follows.
\begin{enumerate}
\item \textbf{Homogeneous Agents with Fixed Action Sets (\cref{sec:ident-agents}): } 
We consider a symmetric setting where all agents share the same action space across time. Under mild conditions on the learning bandit algorithm, 
we prove in \cref{thm:identical-agent-actions-convex-game} that the induced TU game is convex, ensuring the core is non-empty and contains the Shapley value.
This implies that full collaboration (i.e., grand coalition formation) is both stable and equitable.

\item \textbf{Heterogeneous Agents with Diverse Action Sets (\cref{sec:non-ident-agents}):} 
We then analyse more realistic settings where agents differ in their available actions (e.g., content specializations).
We show in  \cref{thm:non-ident-stable-game} that, under some algorithmic assumptions, the resulting game still has a non-empty core, but may not contain the Shapley value.
To address this, we propose a simple, regret-based payout scheme and prove in \cref{thm:shapley-axioms} that, under the assumptions, it satisfies all but one of the Shapley value axioms, providing a practical and principled solution.
\item \textbf{Empirical Validation (\cref{sec:experiments}):}
We conduct numerical simulations on problem instances derived from MovieLens-100k dataset, illustrating how the empirical payout structure aligns with or diverges from the Shapley value in these settings.
\end{enumerate}

\subsection{Related Work}
For a general survey on linear bandits, see \cite{lattimore2020bandit}, 
and for a text-book treatment of cooperative game theoery, see \cite{osborne1994course}.

\paragraph{Collaborative Multi-Agent Bandits.} 
In the regret minimization setting, several algorithms have been proposed to upper bound different regret notions across agents. 
A common goal is minimizing the sum (arithmetic mean) of individual regrets, for which optimal bounds are known \citep{wang2019distributed}.
\cite{baek2021fair} study `grouped bandits', where users arriving over time belong to groups (comparable to agents, in our setting) and propose UCB algorithms optimizing Nash Social Welfare (NSW), equivalent to maximizing the product (geometric mean) of regret reduction that the different groups get by collaboration.
% Product regret minimization has also been considered in standard multi-armed bandits \citep{barman2023fairness}.
% RAMA: But this NSW is a product over time. No notion of agents.
For identical agents sharing a common action set, \cite{yang2023cooperative} derive optimal bounds on \emph{any} agent's regret (maximum), conservatively bounding the performance of the worst agent. 
When the agents are rational and self-interested, equilibrium guarantees for their behaviour in both the collaborative setup \citep{bolton_strategic_1999,ramakrishnan2024collaborative} and the competitive setup \citep{aridor_competing_2024} have been shown.
% aridor_competing_2024 - consider competing bandits. not collaborative bandits.

\paragraph{Heterogeneous Agents.} For agents with heterogeneous action sets (non-identical), general instance-dependent bounds remain elusive to the best of our knowledge.
\cite{raghavan2018externalities} construct an example wherein collaboration using LinUCB \citep{abbasi2011improved} increases regret of an agent. 
Nonetheless, such heterogeneity is also shown to naturally help in exploration. 
Specifically, \cite{kannan2018smoothed} show that a myopic greedy algorithm (albeit, after some initial exploration) achieves non-trivial regret bounds, when some heterogeneity is ensured by random perturbations of actions. \cite{wang2023exploration} further consider agents with `free' arms---actions that incur no regret---which enables effective exploration for others and yielding regret bounds leveraging this heterogeneity.

\paragraph{Shapley Values in Machine Learning and Social Systems.} 
Originating from game theory, Shapley values have been widely adopted to model fairness and marginal impact in ML and social systems. 
Applications include identifying influential nodes in social networks \citep{narayanam2010shapley}, incentivizing truthful data sharing \citep{chessa2017cooperative}, and attributing value in online services like surveys and recommendations \citep{kleinberg2001value}. 
In explainable AI, they quantify feature importance in model outputs \citep{lundberg2017unified} and evaluate training data contributions \citep{ghorbani2019data, jia2019towards}. 
Since exact computation is exponential in the number of players/quantities, efficient polynomial-time approximations are used in practice \citep{mann1960values, musco2024provably}.

\paragraph{Strategic Aspects in Recommender Systems.}
There is a rich literature that studies how rational and self-interested content creators (supply side) and viewers/users (demand side) behave in two-sided markets such as recommender system.
A line of work models content creation as a strategic \emph{game} in which creators choose/create content to maximize exposure to user demand under a given recommendation rule \citep{ben-porat_game-theoretic_2018,jagadeesan_supply-side_2023,yao_how_2023}, 
% I would ideally like to say more about the Ben-Porat paper, but he seems to have both equilibrium and fairness stuff together, and for some reason, I'm not sure how to compare it against ours.
and how coordinated creators may jointly adjust their strategies \citep{yu_beyond_2025}.
These works fix a specific learning/recommendation algorithm and analyze the content creation behaviour and choices.
In contrast, we take creators' content as a given, and instead ask how the algorithm should behave so that the resulting utilities satisfy desirable notions of fairness and stability.

% RAMA: The below two references are added only because a reviewer suggested them. Otherwise, I don't feel its relevant.
Empirical evidence suggests that although the precise recommendation algorithm is typically opaque to creators, they anthropomorphize the algorithm and attribute to it distinct `personas', adapting their content strategies accordingly to maximize exposure \citep{wu_agent_2019}.
On the demand side, 
\cite{fedorova_altruistic_2025} study how groups of users (content viewers) stategically interact with the contents to amplify or counteract algorithmic suppresssion to tune their future recommendations.

\section{Setting \& Preliminaries}
\label{sec:setting}

\paragraph{Multi-agent bandit problem.}
There is a set $\agents$ of agents who all play a common linear bandit instance for a time period of $T$.
(By a slight abuse of notation, we also use $M$ to denote the total number of agents.)

We define a multi-agent linear bandits problem instance by the tuple $I=(\theta^*, \mathrm{X})$, where $\theta^* \in \Real^d$ is an unknown model parameter in ambient dimension $d$, 
and $\mathrm{X} = (X_{a,t})_{a \in \agents, t \in [T]}$ denotes the complete profile of action sets across agents and time.
At every time-step $t \in [T]$, 
each agent $a \in \agents$ is presented with a set of actions $X_{a,t} \subset \Real^d$.
The agent chooses and plays an action $x_{a,t} \in X_{a,t}$, and observes a stochastic reward $y_{a,t} = \inprod{\theta^*}{x_{a,t}} + \eta_{a,t} \in \Real$, 
where $\eta_{a,t}$s are i.i.d. zero-mean sub-gaussian random variables as is standard in the bandit literature. 

We introduce some useful notations next. 
Write $H_{a,t}~:=~(x_{a,s},y_{a,s})_{s=1}^{t}$
to be the history of all actions played and rewards observed by agent $a$ up to time $t$.
Write $x^*_{a,t}:= \argmax_{x \in X_{a,t}} \inprod{\theta^*}{x}$ to be the optimal action that maximizes the expected reward for agent $a$ at time $t$.

\paragraph{Nature of collaboration.} 
We permit agents to communicate and collaborate amongst themselves by forming \emph{coalitions}.
Before bandit playing commences, the agents partition themselves into a collection $\coalset$ of disjoint coalitions.
Then, for each coalition $C \in \coalset$, all agents in the coalition shall reveal all their played actions and observed rewards to all other agents within their coalition. 
% All agents can make use of other agents' information freely and instantaneously within the coalition, i.e., 
So, each agent $a \in C$ can decide action $x_{a,t}$  at time $t$ using coalition history $H^C_{t-1}:=\{H_{b,t-1} : b \in C\}$ upto the previous time-step $t-1$.
Inversely, no information is shared between any two agents who are not in the same coalition. 
The coalition shall make use of a multi-agent bandit algorithm $\alg$
% ---a product of agent-specific bandit algorithms, possibly---
that at every time $t$, takes in the coalition history $H^C_{t-1}$ as input, and outputs a profile of actions $(x_{a,t})_{a \in C}$ for all agents in the coalition to play. 
% We have $\alg: $ \thought{tbd}

Next, we define the expected pseudo-regret (simply called `regret' henceforth) of agent $a$ in coalition $C$ that uses $\alg$ on a problem instance 
$I=(\theta^*,X)$ for a time period $T$ as follows: 
\begin{align}
R^C_a (\alg, I, T) := \sum_{t=1}^{T} \inprod{\theta^*}{x^*_{a,t}} - \expectationover{I, \alg}{\inprod{\theta^*}{x_{a,t}}}, \label{eqn:coalition-regret}  
\end{align}
where the expectation is over both the stochasticity in the rewards and any internal randomness used by the algorithm.
We also use some simpler notations:
$R_a(.)$ denotes $R^{ \{a\} }_a(.)$ of a singleton coalition,
and when the context is clear, 
the dependence on the problem instance, the algorithm used, and/or time 
are implicitly baked into the expression $R^C_a$ or $R^C_a(\alg, T)$  in place of $R^C_a (\alg, I, T)$.

\paragraph{Transferable Utility Game.}
 
We use the Transferable Utility (TU) game from Co-operative Game Theory to model our setting here.
% Write $u_a = - R^C_a - \ell^C_a$ to be the utility of an agent $a \in C$.
We define the characteristic/value function $v$---that
intrinsically depends on the time horizon $T$, problem instance $I$, and the multi-agent bandit algorithm $\alg$ used---as follows:
for every coalition $C \in 2^{\agents}$,
\begin{align}
    % v(C) = - \sum_{a \in C} \ell^C_a - \max_{I \in \inst} \sum_{a \in C} R^C_a (\alg,I)
    v_{\alg,I,T}(C) = - \sum_{a \in C} R^C_a (\alg,I, T) \label{eqn:value-function}
\end{align}
to be the sum of negative (expected) regrets of all agents in the coalition.
A higher value $v_{\cdot}(C)$ corresponds to a lower total regret for agents in coalitions $C$ and is therefore preferable.
\footnote{Instead of a value function, one could also define the TU game with a \emph{cost} function that equals the (positive) sum of regrets of all agents, and lower cost is desireable.
Nevertheless, we use value functions as it is more widely used.}
When the context is clear, we simply use $v(C)$ to denote $v_{\alg,I,T}(C)$.

We denote $(M,v_{\alg,I,T})$ to be the \emph{collaboration game} and shall study this game for different classes of instances $I$ and algorithms $\alg$.

% \section{Mapping our model to Recommender Systems}

\subsection{Mapping our model to Recommender Systems}

We provide a part-by-part comparison of our model setting and its real-world interpretation.

\textbf{Model: } Agent 1 takes an action and observes a reward. Agent 1 then shares details about both the action and reward with Agent 2, who uses this information to decide their own action.

\textbf{Direct translation: } Creator 1 (e.g., a YouTube channel) selects and recommends a piece of content to a user, observes the user’s engagement (e.g., a like or a click), and shares both the content and feedback with Creator 2 (another channel), who then uses this to choose what content of his to recommend.

\textbf{Application reality: } However, in practice, the platform (e.g., YouTube) is the one that recommends content from Creator 1 to a user, observes user feedback, and internally uses this data to improve recommendations, including recommending content from Creator 2 to another user.

In our model, we abstract away the platform and represent its centralized learning and coordination as if agents (creators) were directly sharing full information with each other. This abstraction/simplification allows us to study the system’s learning and strategic dynamics more transparently (as a multi-agent bandit problem and a cooperative game among the agents), 
while still capturing the essence of how feedback from one creator’s content can influence the recommendations for others.

\section{Fixed Action Sets}
\label{sec:ident-agents}
In this section, we restrict our attention to a simple family of multi-agent bandit instances with fixed action sets.
Let $\instfix \subset \inst$ be the set of all problem instances with a single finite action set shared throughout the time horizon by all agents.
That is, for all instances $I \in \instfix$, 
there exists some $X' \subset \Real^d$ of finite cardinality s.t.
$X_{a,t} = X'$ for all agents $a \in \agents$ and time-steps $t \in [T]$.

We describe a multi-agent bandit algorithm $\idenalg$ to play these instances $I \in \instfix$, and analyse the resulting collaboration game $(M,v_{\idenalg, I, T})$.

\paragraph{Algorithm description.} We consider $\idenalg$, a simple collaborative multi-agent bandit algorithm that each coalition shall independently use/execute.
It is based on the essence of \cite{howson2024quack},  
it is a meta-algorithm that uses a given single-agent bandit algorithm $\singalg$ as a black-box decision maker to play the multi-agent bandit problem instance.
% It functions by having all agents play identical actions at a given time step, and uses the feedback to feed the single-agent algorithm as and when it is required.
% push, pull difference. Single-agent algo pulls samples.

Described in \cref{alg:ident-meta-alg}, $\idenalg$ comprises of two conceptual components:
First, all of the multiple agents interact with the original multi-agent bandit instance (iterated using real time $t$), observe rewards, and put the reward into a common reward buffer.  
Second, the $\singalg$ plays a `simulated' single-agent bandit instance (iterated using virtual time $\tau$) by interacting with this reward buffer.

Specifically, at a given step $\tau$, given the single-agent history $\overline{H}_{\tau-1}$, $\singalg$ chooses an action $\overline{x}_\tau \in X$ to play (lines \ref{aline:sing-agent-first-action},\ref{aline:sing-agent-next-action}), seeks and obtains (if available) from the buffer $B_{\overline{x}_\tau}$ a reward $\overline{y}_\tau$ for this action (line \ref{aline:get-from-buffer}). 
It appends this action-reward tuple $(\overline{x}_\tau,\overline{y}_\tau)$ to its single-agent history (line \ref{aline:sing-agent-append-history}) and proceeds to the next step $\tau + 1$.
Here, if the reward for action $\overline{x}_\tau$ is not available in the buffer (the condition in line \ref{aline:while-buff-available} fails), then the first component is triggered,
wherein all the agents play \emph{this} action $\overline{x}_\tau$ on the original multi-agent bandit instance (say, in time-step $t$), observe rewards (line \ref{aline:play-bandit}), and put them into the reward buffer (line \ref{aline:replenish-buffer}). After that, the second component resumes.

\begin{algorithm}[ht]  %[h]
\small
\caption{ $\idenalg$, a multi-agent bandit meta-algorithm.}
\label{alg:ident-meta-alg}
{\bf Input}: A set of collaborating agents $C$, a fixed and finite action set $X$, a single-agent bandit algorithm $\singalg$.
\begin{algorithmic}[1]
\StatePar{Initialize multi-agent reward buffers : $B_x \assign \emptyset, \forall x \in X$, and single-agent history $\overline{H}_0 \assign \emptyset$, step $\tau \assign 1$.} 
\StatePar{Get initial action to play, $\overline{x}_\tau \assign \singalg (\overline{H}_{\tau-1}, X)$.}
\label{aline:sing-agent-first-action}
\For{time-step $t \assign 1,2,\dots,T$}
    % \Comment{Reward for $\overline{x}_\tau$ unavailable.}
    \StatePar{Every agent $a \in C$ plays same action $x_{a,t} \assign \overline{x}_\tau$ and observes reward $y_{a,t}$.}
    \label{aline:play-bandit}
    \StatePar{Replenish buffer $B_{\overline{x}_\tau} \assign \{y_{a,t} : a \in C\}$ with all agents' rewards.}
    \label{aline:replenish-buffer}
    \While{$B_{\overline{x}_\tau} \neq \emptyset$}
    \label{aline:while-buff-available}
    % \Comment{Reward for $\overline{x}_\tau$ available in buffer.}
        \StatePar{Remove arbitrarily a reward $\overline{y}_\tau$ from $B_{\overline{x}_\tau}$.}
        \label{aline:get-from-buffer}
        \StatePar{Append to history $\overline{H}_{\tau+1} \assign \overline{H}_{\tau} \cup \{(\overline{x}_\tau,\overline{y}_\tau)\}$, move to next step $\tau \assign \tau + 1$.}
        \label{aline:sing-agent-append-history}
        \StatePar{Get next action to play, $\overline{x}_\tau \assign \singalg (\overline{H}_{\tau-1},X).$}
        \label{aline:sing-agent-next-action}
    \EndWhile 
\EndFor
\end{algorithmic}
\end{algorithm}

\noindent On the choice of the single-agent black-box algorithm $\singalg$ to be used, we do not mandate any specific algorithm; instead, we only introduce a mild assumption on the regret behavior (in \cref{assm:regret-strictly-concave}) it needs to have.
Specifically, we assume that the algorithm improves over time: its expected instantaneous regret decreases as time $t$ grows.
Howerver, it cannot decrease too rapidly: the cumulative regret curve can not converge faster than logarithmically at any time.

To formally state the assumption, we introduce additional notation.
For brevity, let $R(t):=R_a(\singalg, I, t)$ denote the cumulative regret at time $t$.

We define two quantities that capture the temporal behaviour of this regret trajectory.

First, for $h>0$, define the discrete first derivative $R'(t,h):=\nicefrac{1}{h} \left( R(t+h)-R(t) \right)$, which represents the average regret accumulated over the interval $(t,t+h]$. 
Second, for $g>0$, define the discrete second derivative 
$R{''}(t,g,h):=\nicefrac{1}{g}\left( R'(t+g,h) - R'(t,h) \right)$ that measures the rate of change of this average regret. 
These quantities serve as discrete analogues of the first and second derivatives of the cumulative regret, respectively.
\begin{restatable}{assumption}{assmSingleAgentConcaveRegret}[Regret rate achievable]
\label{assm:regret-strictly-concave}
When run on any problem instance $I \in \instfix$,
% (i.e., when faced with fixed action set $X$),
the expected regret of $\singalg$ obeys the following:
\begin{enumerate}[itemsep=0pt]
    \item \textbf{Strict concavity.} The second derivative of the regret is strictly negative at all times $t$, i.e., 
    \begin{align}
        R{''}(t,g,h) \leq - \upsilon_t, \label{eqn:assm-sing-upper}
    \end{align}
    for some strictly positive sequence of $\upsilon_t > 0$, for all $t,g,h$.
    \item \textbf{Logarithmic limitation.} The second derivative of the regret is bounded from below by that of a logarithmic curve at all times $t$, i.e.,
    \begin{align}
        R{''}(t,g,h) \geq - c t^{-2 + \varepsilon}, \label{eqn:assm-sing-lower}
    \end{align}
    for some arbitrarily small constant $\varepsilon > 0$, for all $t,g,h$.
\end{enumerate}  
\end{restatable}

We discuss in \cref{appn-subsec:justify-strictly-concave-assm} how these assumptions about concavity of the regret and logarithmic learning limitation are very natural and arise from well known theoretical results and empirical observations that apply to several learning algorithms.

Next, we state in \cref{thm:identical-agent-actions-convex-game} that TU games induced by such algorithms on problem instances with fixed action sets are \emph{convex}; 
that is, the marginal contribution that any agent brings to a coalition's value does not decrease as the coalition grows as in \cref{eqn:super-modularity}.
% We also optionally refer the reader to
For completeness, we refer the reader to \cref{sec:fundamentals} for the precise definitions of cooperative game theoretic concepts (such as convex games, balanced games, core of a game, Shapley value axioms etc.) used in the remainder of the paper.
\begin{restatable}{theorem}{fixedActionsConvexGame}
\label{thm:identical-agent-actions-convex-game}
When the meta-algorithm $\idenalg$ is run on any problem instance $I \in \instfix$ for a sufficiently large time horizon $T$,
then, if the black-box single-agent algorithm $\singalg$ used obeys \cref{assm:regret-strictly-concave}, the resultant collaboration game $(M,v_{\idenalg,I,T})$ is convex.
    % time horizon is sufficiently large $T > 2 \sqrt{\nicefrac{MK}{\upsilon}}$, with $\upsilon = \min_{t} \upsilon_t$, and $K = \cardinality{X}$ is the size of the fixed action set. 
    % time horizon $T > (4MK)^{\nicefrac{1}{\epsilon_1}}$ 
    % is sufficiently large and finite.
    % $T \tendsto \infty$.
\end{restatable}
Deferring the formal proof to \cref{appn-sec:fixed-proofs}, we give a proof sketch here.
\begin{proof}[Proof sketch.]
The result is shown in two steps.

\paragraph{Step 1: Relating multi-agent regret to single-agent regret.}

Fix a coalition $C$ of size $m$. In \cref{lem:idenalg-single-multi-regret-inequality}, we show that the total regret incurred by the agents in $C$ under $\idenalg$ over horizon $T$ is tightly controlled on both sides by the regret of the underlying single-agent algorithm $\singalg$ run for $mT$ steps. 

\begin{restatable}{lemma}{multSingBounds}
\label{lem:idenalg-single-multi-regret-inequality}
For any time-horizon $T$ and problem instance $I \in \instfix$, for any agent $a \in C$,
\iftwocolumnmode
\begin{align*}
    R_a(\singalg, I, mT) - mK & \leq \sum_{a \in C} R^C_a(\idenalg, I, T) \\
    & \leq R_a(\singalg, I, mT) + mK,
\end{align*}
\else
\begin{align}
    R_a(\singalg, I, mT) - mK \leq \sum_{a \in C} R^C_a(\idenalg, I, T) \leq R_a(\singalg, I, mT) + mK,
\end{align}
\fi
where $m = \cardinality{C}$ is the number of agents in coalition $C$, and $K = \cardinality{X}$ is the size of the action set $X$.
\end{restatable}

The key feature of this bound is that the discrepancy between the multi-agent regret and the single-agent regret is an additive term of order $mK$, which is independent of the time horizon $T$. Thus, asymptotically in $T$, the coalition regret behaves like the regret of a single agent run for $mT$ rounds.

\paragraph{Step 2: Supermodularity of the value function.}

To establish convexity of the game, it suffices to show supermodularity of the value function:
for all coalitions $S \subseteq Q \subseteq \agents$ and any agent $a \notin Q$,
\begin{align}
v(S \cup \{a\}) - v(S) 
\; \leq \;
v(Q \cup \{a\}) - v(Q).
\label{eqn:super-modularity}
\end{align}
By definition of our value function (\cref{eqn:value-function}), 
the above requires comparison of regret quantities across four different coalitions. Using \cref{lem:idenalg-single-multi-regret-inequality}, we re-express the above inequality in terms of single-agent regret quantities that only differ in terms of time horizon for which they are run:
\begin{align*}
& R_a(\singalg, I, qT{+}T) - R_a(\singalg, I, qT) \\ \leq \;  & R_a(\singalg, I, sT{+}T) - R_a(\singalg, sT) - 4MK,
% \label{eqn:super-modularity-2}
\end{align*}
where $q = \cardinality{Q}, s = \cardinality{S}$ are the sizes of the coalitions.
We complete the proof by showing that the above inequality is satisfied when $\singalg$ adheres to \cref{assm:regret-strictly-concave}.
\end{proof}

The convexity of the collaboration game shown by the Theorem immediately leads to the following corollary:
\begin{corollary}
\label{thm:identical-agent-actions-shap-in-core}
Under the assumptions of \cref{thm:identical-agent-actions-convex-game}, 
the collaboration game $(M, v_{\idenalg, I,T})$ has a non-empty core. 
Moreover, the Shapley value of the game lies in the core.
\end{corollary}
\noindent This is based on standard results about convex games (recapped in \cref{appn-subsec:shapley-value}).
As the core is non-empty, we have that the grand coalition,
i.e., the coalition of the set of all agents $\agents$, is stable. 

% From \cref{thm:identical-agent-actions-shap-in-core}, we observe that it is possible to...

\section{Heterogeneous action sets}
\label{sec:non-ident-agents}
In this section, we consider the more general setting of problem instances $I \in \inst$ with heterogeneous action sets (non-identical agents).
Without postulating specific algorithms, we prescribe some assumptions on the behaviour and performance that a multi-agent bandit algorithm $\alg$ needs to satisfy for our results to hold.

First, we assume the algorithm shall benefit from other agents only through their action-reward tuples (samples) shared:
\begin{restatable}{assumption}{assmSamplesPooled}[Anonymized data consumption]
\label{assm:samples-pooled}
The algorithm $\alg$, when run on a set of agents $C$, shall determine the action to be played by each agent $a \in C$ at time $t$ as a function of 
\begin{enumerate}
    \item the agent's own past action-reward history, $P^a_{t-1} = \left( (x_{a,s},y_{a,s}) \right)_{s \in [t-1]}$,
    \item the sequence of anonymized multisets/pool of past action–reward pairs generated by the other agents in the coalition
    \begin{align*}
      P^{-a}_{t-1} = \left( \bigcup_{b \in C \setminus \{a\} } \{(x_{b,s}, y_{b,s})\} \right)_{s \in [t-1]} , 
    \end{align*}
    % $$,
    where no agent identities are retained, and
    \item the set of actions $X_{a,t}$ of the agent at present.
\end{enumerate}
\end{restatable}

% This guarantees that collaboration pools data in an anonymized way, so outcomes depend only on action-reward data and not on agent identities, enabling symmetry. 
It can be seen that these data sequences from self-play and from other agents, $P^a_t$ and $P^{-a}_t$, are functions of the coalition history $H^C_t$.
Second, the (expected) regret of an agent does not increase when more agents join the coalition:
\begin{restatable}{assumption}{assmLargeCoalitionLessRegret}[The more (agents) the merrier]
\label{assm:large-coalition-less-regret}
For any problem instance $I \in \inst$,
for any two coalitions $S \subseteq Q \subseteq \agents$, and any agent $a \in S$, it holds that 
\begin{align*}
    R^{Q}_a(\alg,I,T) \; \leq \; R^{S}_a(\alg,I,T).
\end{align*}
\end{restatable}
In \cref{appn-subsec:justify-samples-pooled-assm,appn-subsec:justify-lclr-assm}, we motivate why these assumptions are reasonable to expect from multi-agent bandit algorithms. We also provide an Explore-Then-Commit algorithm that satisfies these assumptions. We defer it to \cref{appn-subsec:metc-for-assm} in the interest of space.

Next, we establish that these two assumptions are sufficient for the collaboration game to have several desirable properties.
\begin{theorem}
\label{thm:non-ident-stable-game}
Consider the collaboration game $(M, v_{\alg,I,T})$ induced by any problem instance $I \in \inst$. 
If the bandit algorithm $\alg$ satisfies \cref{assm:large-coalition-less-regret}, 
then the grand coalition $M$ is \emph{stable}; equivalently, the game has a non-empty core.
\end{theorem}

\begin{proof}
By the Bondareva-Shapley theorem, 
% The grand coalition forms when the allocation/imputation belongs to the core. And 
the core is non-empty if and only if the game is \emph{balanced}.
(Recapped as \cref{thm:bondareva-shapley} in \cref{sec:fundamentals}).
We shall show our result by showing that our collaboration game is balanced.
For brevity, we write $v$ and $R^C_a$ to denote $v_{\alg,I,T}$ and $R^C_a(\alg,I,T)$, respectively, for any coalition $C \subseteq M$.

For any balancing mapping $w$ (i.e., $\sum_{S \ni a} w(S) = 1$ for every $a \in \agents$; see \cref{defn:balanced-game}), we have 
\begin{align*}
    & \sum_{S \subseteq M : S \neq \emptyset}  w(S) v(S)
    \; = \; \sum_{S \subseteq M : S \neq \emptyset} w(S) \left( - \sum_{a \in S} R^S_a \right) \\
    =& - \sum_{a \in M} \sum_{S \subseteq M : a \in S} w(S)  R^S_a
    \; \labelrel{\leq}{i:1a} \; - \sum_{a \in M} \sum_{S \subseteq M : a \in S} w(S)  R^M_a \\
    \labelrel{=}{i:1b} & - \sum_{a \in M} R^M_a  = v(M). \numberthis \label{eqn:identical-no-cost-game-balanced}
\end{align*}

Here, \eqref{i:1a} uses $R^S_a \geq R^M_a$ due to \cref{assm:large-coalition-less-regret}, and \eqref{i:1b} uses the fact that $w$ is a balancing mapping. \cref{eqn:identical-no-cost-game-balanced} shows that the game $(M,v)$ is balanced, and thus it has a non-empty core.
\end{proof}

We have established that even when the action sets are arbitrarily heterogeneous (agents are non-identical), the grand coalition is \emph{stable}, i.e., every agent shall prefer to collaborate together as the set of all agents (the grand coalition).
This stability, however, holds only when agents receive a payout/allocation that lies within this non-empty core. 
In the case of fixed (identical) action sets, we established that the collaboration game is convex (\cref{thm:identical-agent-actions-convex-game}), ensuring that the core contains the Shapley value
(\cref{thm:identical-agent-actions-shap-in-core}), which can serve as both a stable and fair allocation.
For heterogeneous agents, by contrast, it remains unclear whether the collaboration game is convex and whether the Shapley value lies in the core. 
Consequently, the natural question arises: what allocation rule can guarantee both stability and equity/fairness in this more general setting? 

\subsection{On a Fair Payout Profile}
In this subsection, we inquire if there are specific point solutions (allocation profiles) that obey desireable properties of fairness---such as efficiency, dummy-player, symmetry, and linearity---as in the axioms of the Shapley value.

\paragraph{`Grand coalition regret' allocation.}
% Recall from \cref{eqn:value-function} the value function $v_{\alg,T,I}$ induced by the regrets of subsets of agents playing bandit problem instance $I$ for $T$ time-steps using algorithm $\alg$.
For the collaboration game $(M,v_{\alg,I,T})$, consider the allocation profile
\begin{align}
    p = \left( p_a = - R^M_a(\alg, I, T) \right)_{a \in \agents}, \label{eqn:gcr-payout}
\end{align}
where each agent is given a payout that equals the negative regret he shall incur as a part of the grand coalition from the said bandit scenario.
We investigate what coalitions shall form under this payout structure. 

\begin{remark}[Practicality]
Before the coalitions of agents form, 
it might appear impracticable to offer a payout $p_a$ to agent $a \in \agents$ that depends on the regret in one specific coalition (the grand coalition $M$)
as in \cref{eqn:gcr-payout}.
However, this is not the case.
We shall go on to show that under the promise of this payout, all agents shall indeed form the grand coalition. 
Thus, payout $p$ is not counterfactual and is a realisable quantity.
\end{remark} 

We find out that this payout profile satisfies three of the four Shapley axioms---efficiency, dummy-player, and symmetry---and it also belongs to the core of this game:
\begin{restatable}{theorem}{payoutShapleyAxioms}
\label{thm:shapley-axioms}
For any bandit instance $I \in \inst$, 
% (permitting heterogeneous agents), 
if $\alg$ satisfies \cref{assm:samples-pooled,assm:large-coalition-less-regret}, 
then, the allocation 
$p=\left( p_a = - R^M_a(\alg, I, T) \right)_{a \in \agents}$ 
obeys the axioms of efficiency, dummy-player, and symmetry, and also belongs to the core of the collaboration game $(M, v_{\alg, I,T})$.
\end{restatable}
Deferring the formal proof to \cref{appn-sec:hetero-proofs}, we provide a proof sketch here.
For brevity, we write $v$ and $R^C_a$ to denote $v_{\alg,I,T}$ and $R^C_a(\alg,I,T)$, respectively, for any coalition $C \subseteq M$.

\begin{proof}[Proof sketch]
First, the payout satisfies the efficiency axiom by design; the value of the grand coalition $ -\sum_{a \in \agents} R^M_a$ is exhaustively divided among the agents.
Building upon the efficiency, we have for any coalition $S \subseteq \agents$ that
\begin{align}
\sum_{a \in S} p_a = - \sum_{a \in S} R^M_a \labelrel{\geq}{i:2aa} - \sum_{a \in S} R^S_a = v(S)
\end{align}
% $ \sum_{a \in S} p_a = - \sum_{a \in S} R^M_a \labelrel{\geq}{i:2aa} - \sum_{a \in S} R^S_a = v(S)$,
where \eqref{i:2aa} is due to \cref{assm:large-coalition-less-regret}.
This shows that no coalition `blocks' the payout $p$ and that it belongs to the core.
The dummy-player axiom can also be similarly shown by building from the definition of a dummy player.

The symmetry axiom mandates that the payout of an agent---in our context, the regret of an agent $a$ in the grand coalition---doesn't change when all the agents are relabeled. 
% We shall establish this by showing that the distribution over trajectories of bandit play remains unaltered when agents are relabeled.
We shall rely on \cref{assm:samples-pooled} and argue that the learning algorithm removes all the agent-specific information during the union operation to pool the data, and it is this pool of anonymized data that is used to determine the actions to play. 
Conditioned on this pool, the action choice only depends on the agent's action set and not the agent's identity.
As a result, the distribution of bandit play trajectories remains unchanged under relabeling of agents, and thus the regret, and by extension, the payout $p_a = -R^{\agents}_a$ remains unchanged under relabeling of agents.
\end{proof}

While we have shown that the payout $p$ obeys three of four Shapley axioms, we believe it cannot satisfy the final axiom in general.
In \cref{appn-subsec:on-linearity}, we give some thoughts on why this payout $p$ can not satisfy the final Shapley axiom of Linearity.
\paragraph{Comparison with actual Shapley value.}
Achieving a payout profile $p^*$
that coincides exactly with the Shapley value is often impractical for two key reasons. 
First, computing the Shapley value has exponential complexity in the number of agents, making it infeasible beyond small-scale settings. 
Second, even for a modest number of agents, its computation requires access to each agent’s regret under \emph{all} possible coalitions in order to evaluate 
$v$ in \cref{eqn:shap-value-definition}.
These regret quantities are counterfactual and generally not directly observable. 

In that context, \cref{thm:shapley-axioms} shows, perhaps surprisingly, that the payout profile 
$p$ satisfies three of the four Shapley axioms without requiring any explicit computation or any redistribution of regret at the end of bandit play. 
In particular, the allocation 
$p_a = -R^M_a$, i.e., 
the regret incurred by agent $a$ when participating in the grand coalition,
emerges naturally as the payoff when the grand coalition forms.

\section{Numerical Simulations}
\label{sec:experiments}

\begin{figure*}[ht]
    \centering
    \begin{subfigure}{.49\linewidth}
        \centering
        \includegraphics[width=\linewidth]{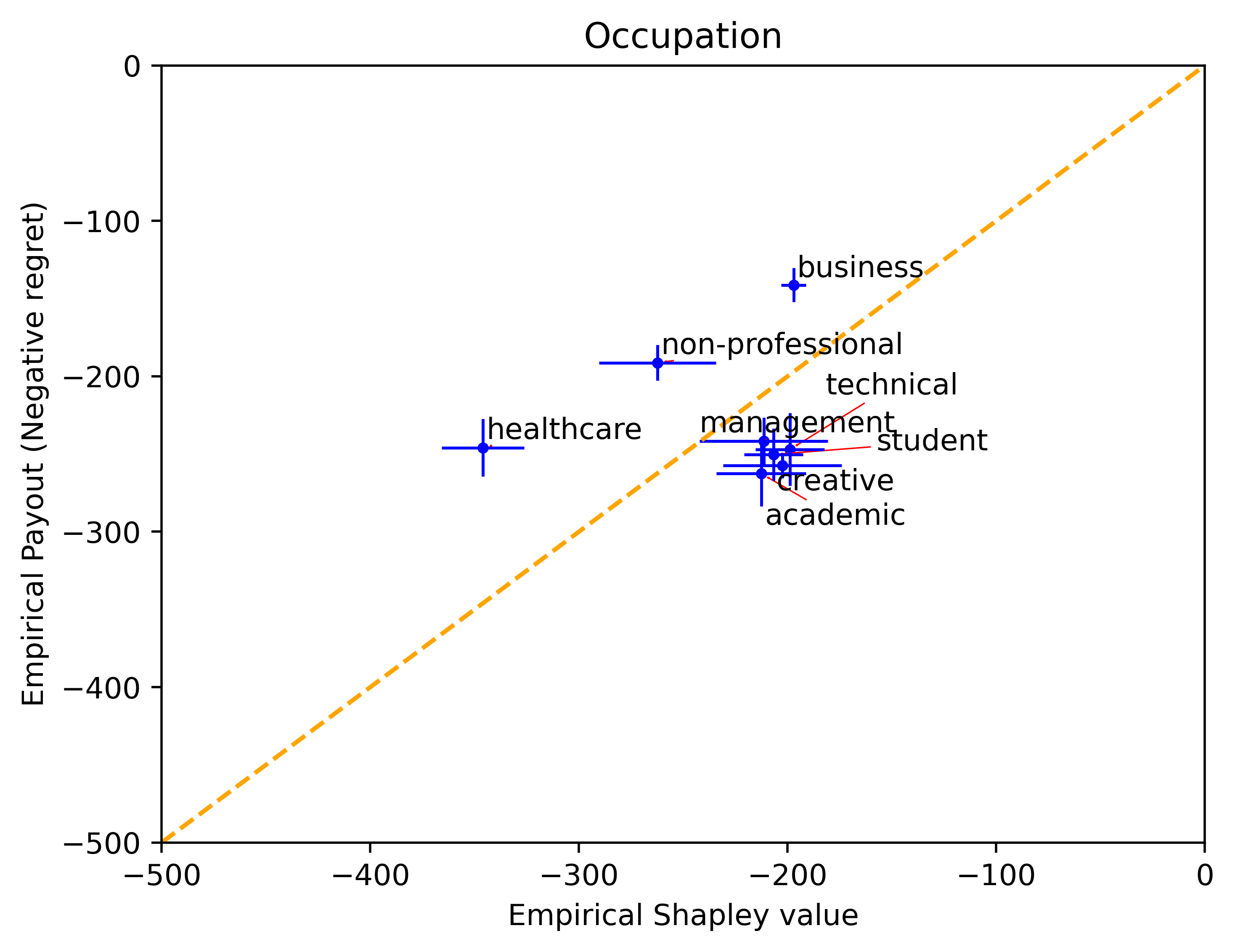}
        \caption{$(\iocc, \mlu)$ }\label{fig:ml-occu}
    \end{subfigure}%
    \begin{subfigure}{.49\linewidth}
        \centering
        \includegraphics[width=\linewidth]{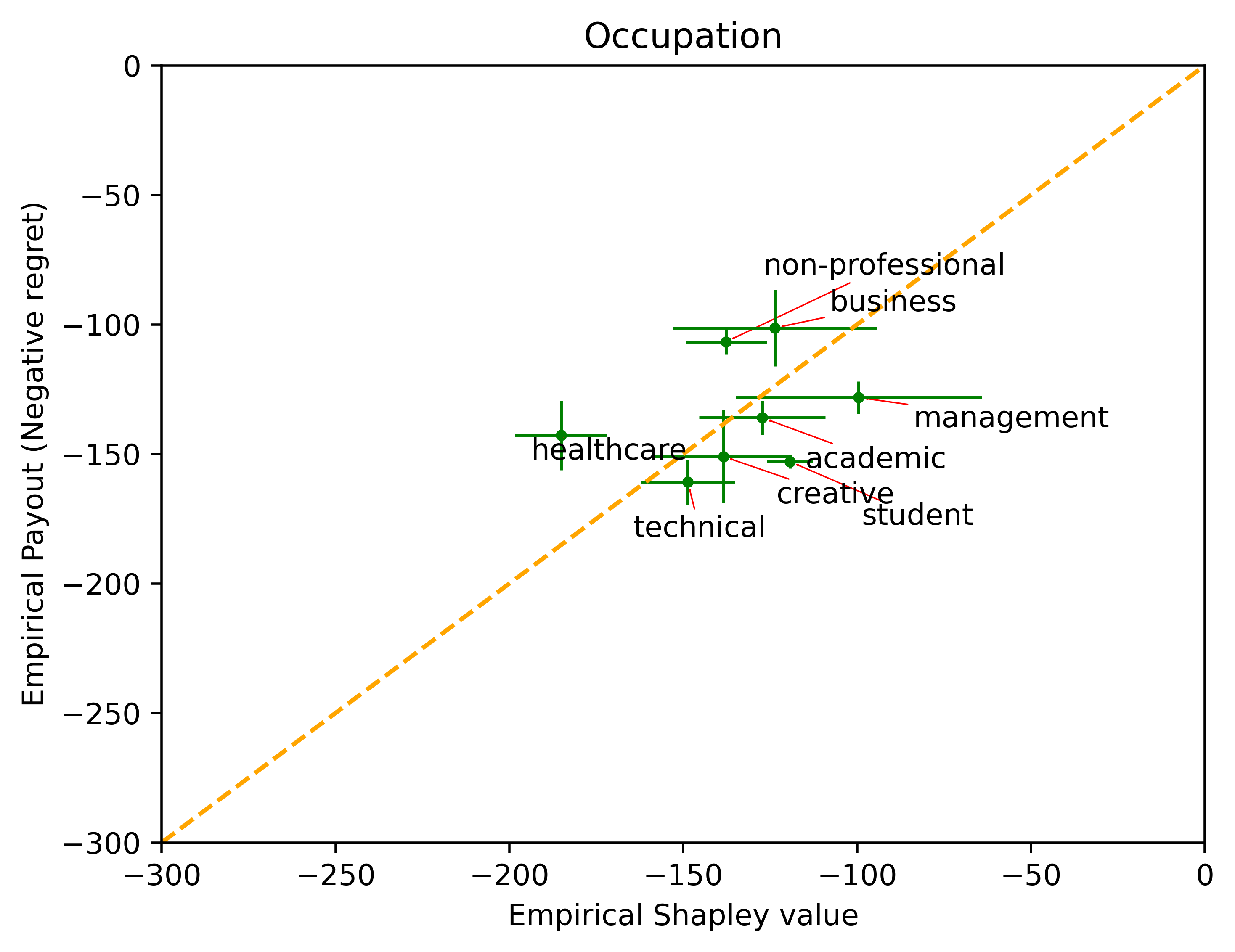}
        \caption{$(\iocc, \greedy)$ } \label{fig:ml-occu-greedy}
    \end{subfigure}
    
    \begin{subfigure}{.49\linewidth}
        \centering
        \includegraphics[width=\linewidth]{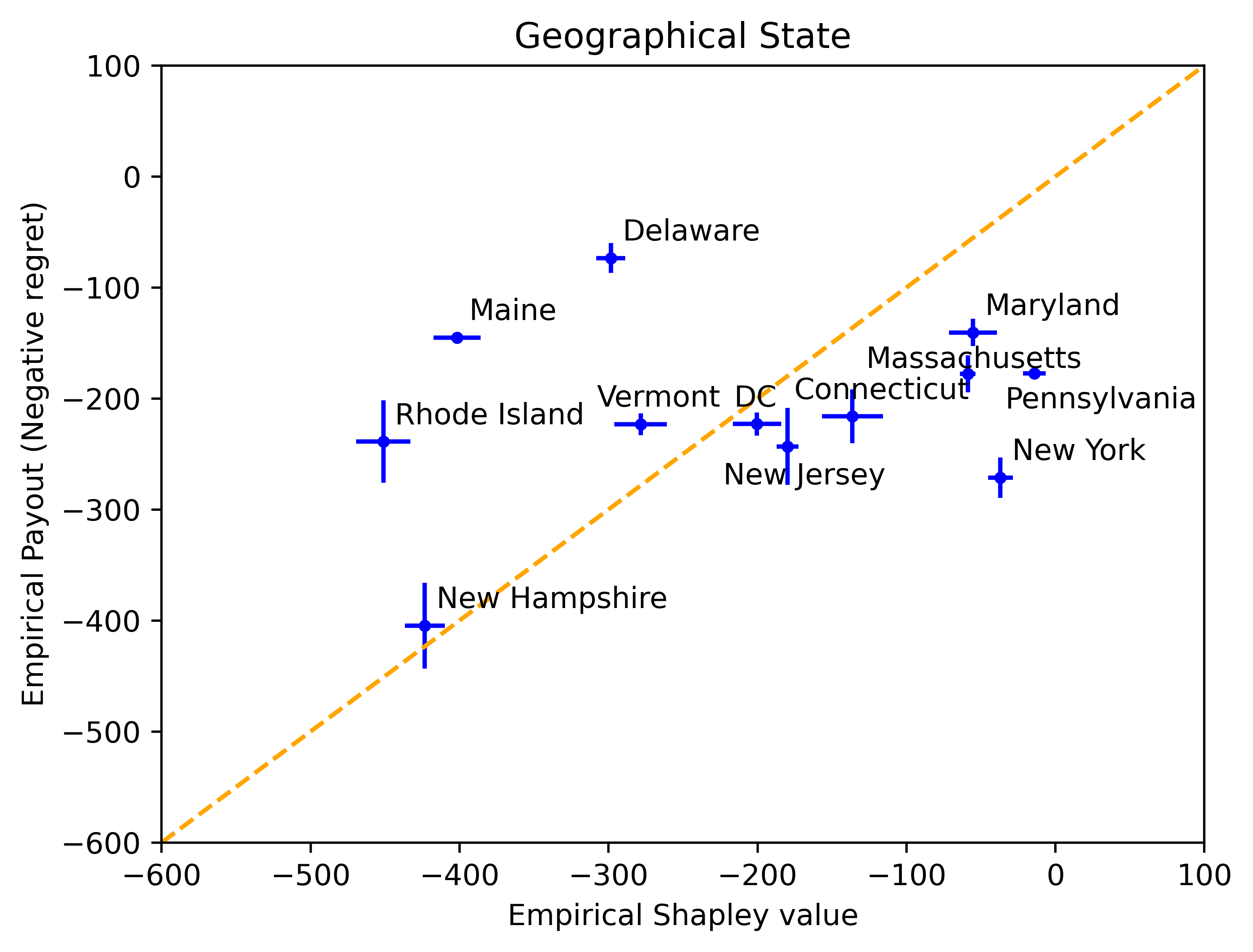}
        \caption{$(\iloc , \mlu)$}\label{fig:ml-states}
    \end{subfigure}
    \begin{subfigure}{.49\linewidth}
        \centering
        \includegraphics[width=\linewidth]{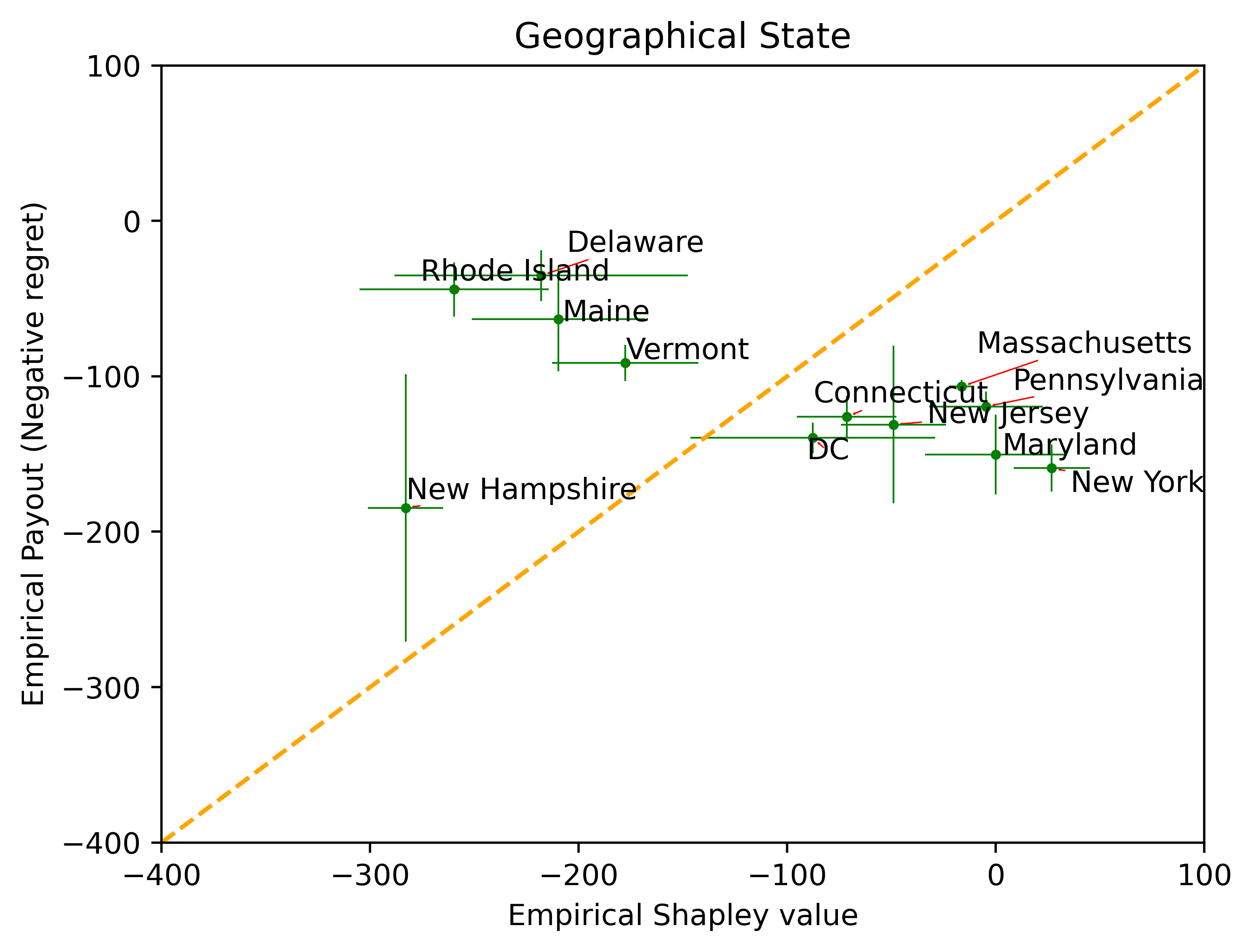}
        \caption{$(\iloc, \greedy)$}\label{fig:ml-states-greedy}
    \end{subfigure}
    \caption{\textbf{Movielens experiments}}
    \label{fig:movielesn-occu-state}
\end{figure*}

We showed in \cref{thm:shapley-axioms} that under some assumptions, 
the payout structure $p=(p_a = -R^\agents_a(\alg,I,T))_{a \in \agents}$ obeys Shapley's axioms of dummy-player, symmetry, and efficiency, but may not obey linearity.
% Rama: belongs to core as well. verify that from experiments?
We run numerical simulations to check how much the payout of an agent and his Shapley value align empirically on experiments using the MovieLens dataset below, and present some results on some synthetic instances in \cref{appn-subsec:synth-exp}.

\paragraph{Bandit environment setup.} 

We consider the MovieLens-100k dataset~\citep{HarperKo15}, motivated by a real-world movie recommendation platform. 
% where users are recommended movies, wherein the recommendation depends on movie ratings given in the past by users on the platform.
% Here, the ratings/data from some users (or groups of users) might contribute more (or less) than others' in improving the quality of recommendations.
% RAMA - COLLAB - The above line isn't conveying the crux relevance.
We create multiple bandit instances 
$\inst' = \{\igen, \iage, \iloc, \iocc \} $ 
by creating multiple sets of agents by dividing the set of users along the lines of attributes such as gender, age group, geographic location, and occupation. 
For each bandit instance, characterized by the attribute, an agent corresponds to the set of users who have a specific value for this attribute.
% RAMA: Clarify about grouping values into same agent, skipping some values altogether in State etc.

We consider two algorithms: (i) \mlu (Multi-agent LinUCB)---all agents run an independent copy of the LinUCB algorithm \citep{abbasi2011improved} with all data available so far in the coalition ; (ii) \greedy---all agents choose actions myopically so as to maximize their instantaneous expected reward w.r.t. the current parameter estimate (by Ordinary Least Squares) using all data available so far in the coalition.

We experimentally simulate the multi-agent bandit run for all four instances $\inst'$ with both the algorithms. Implementation details are given in \cref{appn-subsec:notes-on-implementation}.
We present the results for  $\iloc, \iocc$
% classification by occupation and geographic state 
in \cref{fig:movielesn-occu-state}, and defer the other plots to \cref{appn-subsec:ml-exp}.
% Introduce 'bandit scenario' here? The secion 'On Linearity axiom' does use it.

\subsection{Experiment Outcomes.}
In each bandit scenario, with say $M$ agents in the instance, 
we run the multi-agent bandit algorithm for all $2^M -1$ possible coalitions to get the regret of every agent in every coalition. 
And then, we explicitly compute the empirical Shapley value of the agents $\widehat{\phi}=(\widehat{\phi}_a)_{a \in \agents}$ from all the coalitional regrets (see \cref{eqn:shap-value-definition}).
And we scatter-plot for every agent $a \in \agents$ his empirical payout $\widehat{p}_a$ (in y-axis), which is the negative of the regret from the grand coalition as in \cref{eqn:gcr-payout}, against his empirical Shapley value $\widehat{\phi}_a$ (in x-axis). The results are discussed below.

\paragraph{Payouts vs Shapley value.}
We observe diverse
% contradicting
outcomes across the problem instances. 
In the $\iocc$ instance, the payouts and shapley value are reasonably close to each other (close to the orange identity line) for almost all agents, with both the algorithms (\cref{fig:ml-occu,fig:ml-occu-greedy}). 
This indicates that the regret incurred by an agent is `fair' (in a Shapley sense) and is a commensurate compensation for the value he brings to the other agents.
Further, the Shapley values and payouts show a positive correlation across the agents, more prominently so with the greedy algorithm.

For the $\iloc$ instance, for some agents, the payouts and shapley values are close to each other, but we observe a greater disparity in them for some agents (\cref{fig:ml-states,fig:ml-states-greedy}).
The agents far above (or to the left of) the orange identity line are seen to be have much higher payout compared to their Shapley value. It is interpreted that these agents benefit more from the other agents than they contribute to them.
Correspondingly, the agents far from the identity line on the bottom-right side contribute to the other agents much more than they benefit from other agents.
This mismatch highlights that the commonly used reward structure based on user engagement/satisfaction in recommender platforms does not equitably compensate some agents for their contributions to the platform, which are predominantly users from New York and Pennsylvania with $\mlu$ and New York and Maryland with $\greedy$, in this example.
Additionally, however, we recall that these payouts are \emph{not} the actual Shapley values (more discussion on this in \cref{appn-subsec:on-linearity}) and in general are not expected to equal the empirical Shapley values either.

\paragraph{Usefulness of Greedy over LinUCB.} 
Another intersting outcome we see is that the greedy algorithm performs better in minimizing regret than the much acclaimed LinUCB as seen in both $\iloc$ and $\iocc$ instances.
Conventionally, it is known that the delicate exploration-exploitation trade-off handled by LinUCB gives optimal regret guarantees, whereas greedy exploration in general is known to suffer worse regret.
However, the agent actions (even greedily chosen ones among the available options) can be diverse enough in their vector directions in $\Real^{d}$ due to the inherent heterogeneity of the agent action sets both across agents and across time.
This can result in non-deliberate/inadvertent exploration of different dimensions even when the agent is greedily trying to play actions in a specific direction that it currently thinks is optimal.
This result backs up a recent line of work \citep{kannan2018smoothed,raghavan2018externalities} where it is shown that with sufficient heterogeniety in action sets, the greedy algorithm surprisingly achieves non-trivial regret bounds.
Further, our experiments show that the greedy algorithm---in addition to helping reduce one's regret---also helps in reducing regrets of other agents better than LinUCB does, as seen from the higher Shapley values of agents following the greedy algorithm, in both the instances, $\iloc$ and $\iocc$.

\vspace{-2pt}
\section{Conclusion}
\vspace{-2pt}
In this paper, motivated by creator incentives in Recommender Systems, we considered the setting of multi-agent bandits with both fixed action set and heterogeneous action sets, and investigated the properties of the ensuing TU (transferable utility) coalition games.
We explored different bandit algorithm properties, under which we showed that the game with fixed action set is provably convex and thus has a non-empty core containing the Shapley value; and the game with heterogeneous action sets has a non-empty core but may not contain the Shapley value.
Further, we proposed a simple payout profile---where each agent's payout is his (negative) regret in the grand coalition---and we established this payout has desireable properties such as belonging to the core, and obeying Shapley axioms of efficiency, dummy-player, and symmetry.
A key contribution of our work is to show that these mild and natural assumptions are sufficient to guarantee such strong cooperative-game properties for such a simple payout.
An interesting research direction is to investigate whether more sophisticated payout schemes---particularly those that allow transfers of reward or utility at the end of bandit play, a setting well aligned with recommender systems—--can yield alternative notions of fairness.

\section*{Acknowledgements}
Arpit Agarwal was partially supported by an early career research grant awarded by ANRF.

\bibliographystyle{abbrvnat}
\bibliography{references/coll-learn,references/coop-references,references/unclass,references/math,references/coop-zotero-references}

@book{lattimore2020bandit,
  title={Bandit algorithms},
  author={Lattimore, Tor and Szepesv{\'a}ri, Csaba},
  year={2020},
  publisher={Cambridge University Press}
}

@inproceedings{wang2019distributed,
  title={Distributed Bandit Learning: Near-Optimal Regret with Efficient Communication},
  author={Wang, Yuanhao and Hu, Jiachen and Chen, Xiaoyu and Wang, Liwei},
  booktitle={International Conference on Learning Representations},
  year={2019},
  annot={This paper studies `Regret Minimization' in Multi-agent (collaborative) setup of MAB and Linear Bandits. Compared to Hillel '13 and Tao '19 who both look at BAI, this work differs as follows.
  \begin{enumerate}
    \item Communication is measured in terms of rounds in them. It is measured in terms of cost (i.e., number of integers/reals communicated between server and agents) of communication.
    \item ...
  \end{enumerate}
  The notations are $K$ arms and $M$ agents.
  The bench-mark performance is `Immediate sharing' (in other words, unbounded communication which here is $M^2. T$ cost for all peer-peer communication after every time-step) which gives a regret bound of $O(\sqrt{MKT \log T})$. 
  The paper's DEMAB algo achieves the same regret bound with a reduced $O(M \log (MK))$ communication cost. ( It's $\log(MK/\delta)$ with prob $1-\delta$.)
  As a lower bound, the paper shows that if the communication cost is reduced/constrained slightly more to $O(M)$, then the regret steeply worsens to $\Omega(M \sqrt{KT})$. 
         
  \paragraph{DEMAB algo sketch} Happens in two stages. In stage 1, every agent runs elimination algorithm \cite{} (say $X$) on \emph{all} arms for $D=T/MK$ time. This shall incur $O(M \sqrt{DK \log D}) = O(\sqrt{MT \log T})$ time. This stage is to quickly get rid of very bad arms. Say each agent $i$ is left with arm $A(i)$. Idea is good arms will be in many $A(i)$s.
  In stage 2, each agent trims his $A(i)$ to $B(i)$ such that they are disjoint while ensuring an arm that was in all $A(i)$s is retained in $B(i)$.
  In each phase $\ell$, agents run $X$ again on $B(i)_\ell$, and share empirical mean of best arm at its end. From the global knowledge, each agent drops arms that are $2^{-\ell}$-suboptimal. As gaps become smaller, the phases grow longer ($4^{\ell}$ order).
  To prevent skewed $B(i)_\ell$s (across agents), there's a load balancing step. This step has a high information overhead when done, so in expectation, its done very few times thanks to the `shared randomness' that the algo uses to create $B(i)$s initially. \\
  Similar results and algorithms are there for the Linear Bandits setup with both fixed arm-set and time-varying arm-set. 
  For the time-varying arm-set, the take Abbasi-Yadkori's algorithm as the bench-mark/basis.
 }
}

@article{zhang2021multi,
  title={Multi-agent reinforcement learning: A selective overview of theories and algorithms},
  author={Zhang, Kaiqing and Yang, Zhuoran and Ba{\c{s}}ar, Tamer},
  journal={Handbook of reinforcement learning and control},
  pages={321--384},
  year={2021},
  publisher={Springer}
}

@article{kairouz2021advances,
  title={Advances and open problems in federated learning},
  author={Kairouz, Peter and McMahan, H Brendan and Avent, Brendan and Bellet, Aur{\'e}lien and Bennis, Mehdi and Bhagoji, Arjun Nitin and Bonawitz, Kallista and Charles, Zachary and Cormode, Graham and Cummings, Rachel and others},
  journal={Foundations and trends{\textregistered} in machine learning},
  volume={14},
  number={1--2},
  pages={1--210},
  year={2021},
  publisher={Now Publishers, Inc.}
}

@article{qian2024digital,
  title={Digital content creation: An analysis of the impact of recommendation systems},
  author={Qian, Kun and Jain, Sanjay},
  journal={Management Science},
  volume={70},
  number={12},
  pages={8668--8684},
  year={2024},
  publisher={INFORMS}
}

@inproceedings{hronmodeling,
  title={Modeling content creator incentives on algorithm-curated platforms},
  author={Hron, Jiri and Krauth, Karl and Jordan, Michael and Kilbertus, Niki and Dean, Sarah},
  booktitle={The Eleventh International Conference on Learning Representations},
  year={2023}
}

@article{zhu2023online,
  title={Online learning in a creator economy},
  author={Zhu, Banghua and Karimireddy, Sai Praneeth and Jiao, Jiantao and Jordan, Michael I},
  journal={arXiv preprint arXiv:2305.11381},
  year={2023}
}

@article{ko2022survey,
  title={A survey of recommendation systems: recommendation models, techniques, and application fields},
  author={Ko, Hyeyoung and Lee, Suyeon and Park, Yoonseo and Choi, Anna},
  journal={Electronics},
  volume={11},
  number={1},
  pages={141},
  year={2022},
  publisher={MDPI}
}

@article{zhang2019deep,
  title={Deep learning based recommender system: A survey and new perspectives},
  author={Zhang, Shuai and Yao, Lina and Sun, Aixin and Tay, Yi},
  journal={ACM computing surveys (CSUR)},
  volume={52},
  number={1},
  pages={1--38},
  year={2019},
  publisher={ACM New York, NY, USA}
}

@article{shapley1953value,
  title={A value for n-person games},
  author={Shapley, Lloyd S and others},
  year={1953},
  publisher={Princeton University Press Princeton}
}

@article{bondareva1963some,
  title={Some applications of linear programming methods to the theory of cooperative games},
  author={Bondareva, Olga N},
  journal={Problemy Kibernet},
  volume={10},
  pages={119},
  year={1963}
}

@book{roth1988shapley,
  title={The Shapley value: essays in honor of Lloyd S. Shapley},
  author={Roth, Alvin E},
  year={1988},
  publisher={Cambridge University Press}
}

@article{shapley1967balanced,
  title={On balanced sets and cores},
  author={Shapley, Lloyd S},
  journal={Naval research logistics quarterly},
  volume={14},
  number={4},
  pages={453--460},
  year={1967},
  publisher={Wiley Online Library}
}

@article{shapley1971cores,
  title={Cores of convex games},
  author={Shapley, Lloyd S},
  journal={International journal of game theory},
  volume={1},
  pages={11--26},
  year={1971},
  publisher={Springer}
}

@book{osborne1994course,
  title={A course in game theory},
  author={Osborne, Martin J and Rubinstein, Ariel},
  year={1994},
  publisher={MIT press}
}

@article{howson2024quack,
  title={QuACK: A Multipurpose Queuing Algorithm for Cooperative $ k $-Armed Bandits},
  author={Howson, Benjamin and Filippi, Sarah and Pike-Burke, Ciara},
  journal={arXiv preprint arXiv:2410.23867},
  year={2024}
}

@article{baek2021fair,
  title={Fair exploration via axiomatic bargaining},
  author={Baek, Jackie and Farias, Vivek},
  journal={Advances in Neural Information Processing Systems},
  volume={34},
  pages={22034--22045},
  year={2021}
}

@inproceedings{yang2023cooperative,
  title={Cooperative Multi-agent Bandits: Distributed Algorithms with Optimal Individual Regret and Communication Costs},
  author={Yang, Lin and Wang, Xuchuang and Hajiesmaili, Mohammad and Zhang, Lijun and Lui, John CS and Towsley, Don},
  booktitle={Coordination and Cooperation for Multi-Agent Reinforcement Learning Methods Workshop},
  year={2023}
}

@inproceedings{wang2023achieving,
  title={Achieving near-optimal individual regret low communications in multi-agent bandits},
  author={Wang, Xuchuang and Yang, Lin},
  booktitle={The Eleventh International Conference on Learning Representations (ICLR)},
  year={2023}
}

@inproceedings{raghavan2018externalities,
  title={The externalities of exploration and how data diversity helps exploitation},
  author={Raghavan, Manish and Slivkins, Aleksandrs and Wortman, Jennifer Vaughan and Wu, Zhiwei Steven},
  booktitle={Conference on Learning Theory},
  pages={1724--1738},
  year={2018},
  organization={PMLR}
}

@article{kannan2018smoothed,
  title={A smoothed analysis of the greedy algorithm for the linear contextual bandit problem},
  author={Kannan, Sampath and Morgenstern, Jamie H and Roth, Aaron and Waggoner, Bo and Wu, Zhiwei Steven},
  journal={Advances in neural information processing systems},
  volume={31},
  year={2018}
}

@inproceedings{wang2023exploration,
  title={Exploration for free: how does reward heterogeneity improve regret in cooperative multi-agent bandits?},
  author={Wang, Xuchuang and Yang, Lin and Chen, Yu-Zhen Janice and Liu, Xutong and Hajiesmaili, Mohammad and Towsley, Don and Lui, John CS},
  booktitle={Uncertainty in Artificial Intelligence},
  pages={2192--2202},
  year={2023},
  organization={PMLR}
}

@inproceedings{ghorbani2019data,
  title={Data shapley: Equitable valuation of data for machine learning},
  author={Ghorbani, Amirata and Zou, James},
  booktitle={International conference on machine learning},
  pages={2242--2251},
  year={2019},
  organization={PMLR}
}

@inproceedings{lundberg2017unified,
  title={A unified approach to interpreting model predictions},
  author={Lundberg, Scott M and Lee, Su-In},
  booktitle={Advances in Neural Information Processing Systems},
  pages={4765--4774},
  year={2017}
}

@inproceedings{jia2019towards,
  title={Towards efficient data valuation based on the shapley value},
  author={Jia, Ruoxi and Dao, David and Wang, Boxin and Hubis, Frances Ann and Hynes, Nick and G{\"u}rel, Nezihe Merve and Li, Bo and Zhang, Ce and Song, Dawn and Spanos, Costas J},
  booktitle={The 22nd International Conference on Artificial Intelligence and Statistics},
  pages={1167--1176},
  year={2019},
  organization={PMLR}
}

@book{mann1960values,
  title={Values of large games, IV: Evaluating the electoral college by Montecarlo techniques},
  author={Mann, Irwin and Shapley, Lloyd S},
  year={1960},
  publisher={Rand Corporation},
}

@article{musco2024provably,
  title={Provably Accurate Shapley Value Estimation via Leverage Score Sampling},
  author={Musco, Christopher and Witter, R Teal},
  journal={arXiv preprint arXiv:2410.01917},
  year={2024}
}

@inproceedings{chessa2017cooperative,
  title={A cooperative game-theoretic approach to quantify the value of personal data in networks},
  author={Chessa, Michela and Loiseau, Patrick},
  booktitle={Proceedings of the 12th workshop on the Economics of Networks, Systems and Computation},
  pages={1--1},
  year={2017}
}

@inproceedings{kleinberg2001value,
  title={On the value of private information},
  author={Kleinberg, Jon and Papadimitriou, Christos H and Raghavan, Prabhakar},
  booktitle={Proceedings of the 8th Conference on Theoretical Aspects of Rationality and Knowledge},
  pages={249--257},
  year={2001}
}

@article{narayanam2010shapley,
  title={A shapley value-based approach to discover influential nodes in social networks},
  author={Narayanam, Ramasuri and Narahari, Yadati},
  journal={IEEE transactions on automation science and engineering},
  volume={8},
  number={1},
  pages={130--147},
  year={2010},
  publisher={IEEE}
}

@article{auer2002finite,
  title={Finite-time analysis of the multiarmed bandit problem},
  author={Auer, Peter and Cesa-Bianchi, Nicolo and Fischer, Paul},
  journal={Machine learning},
  volume={47},
  pages={235--256},
  year={2002},
  publisher={Springer}
}

@inproceedings{kaufmann2012thompson,
  title={Thompson sampling: An asymptotically optimal finite-time analysis},
  author={Kaufmann, Emilie and Korda, Nathaniel and Munos, R{\'e}mi},
  booktitle={International conference on algorithmic learning theory},
  pages={199--213},
  year={2012},
  organization={Springer}
}

@article{abbasi2011improved,
  title={Improved algorithms for linear stochastic bandits},
  author={Abbasi-Yadkori, Yasin and P{\'a}l, D{\'a}vid and Szepesv{\'a}ri, Csaba},
  journal={Advances in neural information processing systems},
  volume={24},
  year={2011}
}

@article{lai1985asymptotically,
  title={Asymptotically efficient adaptive allocation rules},
  author={Robbins, Herbert},
  journal={Advances in applied mathematics},
  volume={6},
  number={1},
  pages={4--22},
  year={1985},
  publisher={Academic Press}
}

@article{russo2018tutorial,
  title={A tutorial on thompson sampling},
  author={Russo, Daniel J and Van Roy, Benjamin and Kazerouni, Abbas and Osband, Ian and Wen, Zheng and others},
  journal={Foundations and Trends{\textregistered} in Machine Learning},
  volume={11},
  number={1},
  pages={1--96},
  year={2018},
  publisher={Now Publishers, Inc.}
}

@inproceedings{garivier2011kl,
  title={The KL-UCB algorithm for bounded stochastic bandits and beyond},
  author={Garivier, Aur{\'e}lien and Capp{\'e}, Olivier},
  booktitle={Proceedings of the 24th annual conference on learning theory},
  pages={359--376},
  year={2011},
  organization={JMLR Workshop and Conference Proceedings}
}

@article{chapelle2011empirical,
  title={An empirical evaluation of thompson sampling},
  author={Chapelle, Olivier and Li, Lihong},
  journal={Advances in neural information processing systems},
  volume={24},
  year={2011}
}

@inproceedings{ramakrishnan2024collaborative,
  title={Collaborative Learning under Strategic Behavior: Mechanisms for Eliciting Feedback in Principal-Agent Bandit Games},
  author={Ramakrishnan, K and Agarwal, Arpit and Subramanian, Lakshminarayanan and Nickel, Maximilian},
  booktitle={Agentic Markets Workshop at ICML 2024},
  year={2024}
}

@article{bastani2017exploiting,
  title={Exploiting the natural exploration in contextual bandits},
  author={Bastani, Hamsa and Bayati, Mohsen and Khosravi, Khashayar},
  journal={arXiv preprint arXiv:1704.09011},
  year={2017},
  publisher={CoRR}
}

@article{wu_agent_2019,
    title = {Agent, {Gatekeeper}, {Drug} {Dealer}: {How} {Content} {Creators} {Craft} {Algorithmic} {Personas}},
    volume = {3},
    issn = {2573-0142},
    shorttitle = {Agent, {Gatekeeper}, {Drug} {Dealer}},
    url = {https://dl.acm.org/doi/10.1145/3359321},
    doi = {10.1145/3359321},
    language = {en},
    number = {CSCW},
    urldate = {2025-11-28},
    journal = {Proceedings of the ACM on Human-Computer Interaction},
    author = {Wu, Eva Yiwei and Pedersen, Emily and Salehi, Niloufar},
    month = nov,
    year = {2019},
    pages = {1--27},
}

@inproceedings{yao_how_2023,
    title = {How {Bad} is {Top}-\${K}\$ {Recommendation} under {Competing} {Content} {Creators}?},
    url = {https://proceedings.mlr.press/v202/yao23b.html},
    language = {en},
    urldate = {2026-02-23},
    booktitle = {Proceedings of the 40th {International} {Conference} on {Machine} {Learning}},
    publisher = {PMLR},
    author = {Yao, Fan and Li, Chuanhao and Nekipelov, Denis and Wang, Hongning and Xu, Haifeng},
    month = jul,
    year = {2023},
    note = {ISSN: 2640-3498},
    pages = {39674--39701},
}

@inproceedings{yu_beyond_2025,
    title = {Beyond {Self}-{Interest}: {How} {Group} {Strategies} {Reshape} {Content} {Creation} in {Recommendation} {Platforms}?},
    shorttitle = {Beyond {Self}-{Interest}},
    url = {https://openreview.net/forum?id=q0JaH6Ukqb&noteId=1tb6U3J4pX},
    language = {en},
    urldate = {2026-02-19},
    author = {Yu, Yaolong and Yao, Fan and Pan, Sinno Jialin},
    month = jun,
    year = {2025},
}

@inproceedings{fedorova_altruistic_2025,
    title = {Altruistic {Collective} {Action} in {Recommender} {Systems}},
    url = {https://openreview.net/forum?id=wKSX4rqRSg},
    language = {en},
    urldate = {2026-02-19},
    author = {Fedorova, Ekaterina and Kitch, Madeline Celi and Podimata, Chara},
    month = nov,
    year = {2025},
}

@inproceedings{ben-porat_game-theoretic_2018,
    address = {Red Hook, NY, USA},
    series = {{NIPS}'18},
    title = {A game-theoretic approach to recommendation systems with strategic content providers},
    url = {https://dl.acm.org/doi/10.5555/3326943.3327046},
    urldate = {2026-02-22},
    booktitle = {Proceedings of the 32nd {International} {Conference} on {Neural} {Information} {Processing} {Systems}},
    publisher = {Curran Associates Inc.},
    author = {Ben-Porat, Omer and Tennenholtz, Moshe},
    month = dec,
    year = {2018},
    pages = {1118--1128},
}

@inproceedings{jagadeesan_supply-side_2023,
    address = {Red Hook, NY, USA},
    series = {{NIPS} '23},
    title = {Supply-side equilibria in recommender systems},
    urldate = {2026-02-22},
    booktitle = {Proceedings of the 37th {International} {Conference} on {Neural} {Information} {Processing} {Systems}},
    publisher = {Curran Associates Inc.},
    author = {Jagadeesan, Meena and Garg, Nikhil and Steinhardt, Jacob},
    month = dec,
    year = {2023},
    pages = {14597--14608},
}

@article{bolton_strategic_1999,
    title = {Strategic {Experimentation}},
    volume = {67},
    issn = {1468-0262},
    url = {https://onlinelibrary.wiley.com/doi/abs/10.1111/1468-0262.00022},
    doi = {10.1111/1468-0262.00022},
    language = {en},
    number = {2},
    urldate = {2026-02-19},
    journal = {Econometrica},
    author = {Bolton, Patrick and Harris, Christopher},
    year = {1999},
    note = {\_eprint: https://onlinelibrary.wiley.com/doi/pdf/10.1111/1468-0262.00022},
    pages = {349--374},
}

@misc{aridor_competing_2024,
    title = {Competing {Bandits}: {The} {Perils} of {Exploration} {Under} {Competition}},
    shorttitle = {Competing {Bandits}},
    url = {http://arxiv.org/abs/2007.10144},
    doi = {10.48550/arXiv.2007.10144},
    urldate = {2025-11-18},
    publisher = {arXiv},
    author = {Aridor, Guy and Mansour, Yishay and Slivkins, Aleksandrs and Wu, Zhiwei Steven},
    month = oct,
    year = {2024},
    note = {arXiv:2007.10144 [cs]},
}

@book{preparata2012computational,
  title={Computational geometry: an introduction},
  author={Preparata, Franco P and Shamos, Michael I},
  year={2012},
  publisher={Springer Science \& Business Media},
  
}

@article{pearl,
  author       = {Zheqing Zhu and
                  Rodrigo de Salvo Braz and
                  Jalaj Bhandari and
                  Daniel Jiang and
                  Yi Wan and
                  Yonathan Efroni and
                  Liyuan Wang and
                  Ruiyang Xu and
                  Hongbo Guo and
                  Alex Nikulkov and
                  Dmytro Korenkevych and
                  {\"{U}}r{\"{u}}n Dogan and
                  Frank Cheng and
                  Zheng Wu and
                  Wanqiao Xu},
  title        = {Pearl: {A} Production-ready Reinforcement Learning Agent},
  journal      = {CoRR},
  volume       = {abs/2312.03814},
  year         = {2023}
}

@article{HarperKo15,
author = {Harper, F. Maxwell and Konstan, Joseph A.},
title = {{The MovieLens Datasets: History and Context}},
year = {2015},
issue_date = {January 2016},
publisher = {Association for Computing Machinery},
address = {New York, NY, USA},
volume = {5},
number = {4},
issn = {2160-6455},
url = {https://doi.org/10.1145/2827872},
doi = {10.1145/2827872},
journal = {ACM Trans. Interact. Intell. Syst.},
month = {dec},
articleno = {19},
numpages = {19}
}

\newpage
\appendix
\onecolumn
\aistatstitle{
% Instructions for Paper Submissions to AISTATS 2026: \\
Supplementary Materials}

\inappendixtrue
\tableofcontents

\newpage

\appendix

\section{Fundamentals}
\label{sec:fundamentals}

\subsection{Game theoretic aspects}

In this section, we recap some fundamental definitions and results from co-operative game theory that are referred to in the main paper.

\noindent Consider a Transferable Utility (TU) game $(N,v)$, where $N$ is the finite set of players, and $v~:~2^N~\mapsto~\Real$ is the characteristic or value function with $v(\emptyset) = 0$.

\begin{definition}[Convexity of a TU game]
A transferable utility game $(N,v)$ is said to be \emph{convex} if marginal contributions of agents are non-decreasing with coalition growth.
That is, for all players $i \in N$, and for all coalitions $C,D$ such that $C \subseteq D \subseteq N \setminus \{i\}$, it holds that 
\begin{align}
    v(D \cup \{i\}) - v(D) \geq v(C \cup \{i\}) - v(C).
\end{align}
\end{definition}

\begin{definition}[Balanced game]
\label{defn:balanced-game}
For a set of players $N$, a mapping $w : 2^N \mapsto [0,1]$ is said to be \emph{`balancing'} if for every player $a \in N$, it holds that 
$$\sum_{S \subseteq N : a \in S} w(S) = 1.$$

\noindent A transferable utility game $(N,v)$ is said to be \emph{balanced} if for every balancing mapping $w$ it holds that
\begin{align*}
\sum_{S \subseteq N : S \neq \emptyset} w(S) v(S) \leq v(N).
\end{align*}
\end{definition}

\begin{definition}[Core of a TU game]
For a transferable utility game $(N,v)$ with $\cardinality{N}=n$, it's core is defined as the set of payout profiles that are feasible and can not be improved upon by any coalition. Precisely,
\begin{align}
    \text{Core}(N,v) := \left\{ p = (p_1, \dots, p_n) \in \Real^n : 
    \sum_{i=1}^n p_i = v(N)
    ; \forall C \subseteq N, \sum_{i \in C} p_i \geq v(C)  \right\}.
\end{align}
\end{definition}

In general, the core of a game is not guaranteed to be non-empty. A necesary and sufficient for it is given next:
\begin{theorem}[\citep{bondareva1963some, shapley1967balanced}]
\label{thm:bondareva-shapley}
A transferable utility game $(N,v)$ has a non-empty core if and only if it is balanced.
\end{theorem}

\subsubsection{Shapley Value}
\label{appn-subsec:shapley-value}
Shapley value is a point solution concept for distributing the value of a coalition among its members in a `fair' way \citep{shapley1953value,roth1988shapley}.

\begin{definition}[Shapley Value]
For a TU game $(N,v)$ with $n = \cardinality{N}$ players, the shapley value $\phi=(\phi_i)_{i \in [n]}$ is given by
\begin{align}
    \phi_i(v) = \frac{1}{n} \sum_{S \subseteq N \setminus \{i\}} {n-1 \choose \cardinality{S}}^{-1} \left( v(S \cup \{i\}) - v(S) \right). \label{eqn:shap-value-definition}
% \binom{\cardinality{S}}{n-1}
\end{align}
\end{definition}

\noindent The Shapley value is the unique solution that satisfies the following axioms.

\begin{axiom}[Symmetry, Anonymity]
\label{axio:shap-sym}
Let $\pi : N \mapsto N$ be any (bijection) permutation of the set of players. 
Overload notation to write $\pi(C) := \{ \pi(a) : a \in C\}$ for any coalition $C \subseteq N$.
Consider a modified value function $\pi v$ that is defined as $\pi v( \pi(C) ) = v(C)$ for all $C \subseteq N$.
Then, the axiom mandates that the Shapley values of the two games obey
\begin{align}
\phi_i(v) = \phi_{\pi(i)}(\pi v)
\end{align}
for all $i \in [N]$. 
\end{axiom}

\begin{axiom}[Carrier]
\label{axio:shap-carr}
A subset $C \subseteq N$ of agents is said to be a \emph{`carrier'} when $ v(D) = v(C \cap  D)$ for all $D \subseteq N$.
Then, the axiom mandates that the cumulative Shapley values of carriers equals the value of grand coalition, i.e.,
\begin{align}
 \sum_{a \in C} \phi_a(v) = v(N).   
\end{align}
\end{axiom}

\begin{axiom}[Linearity]
\label{axio:shap-linearity}
Consider two different TU games $(N,v_1)$ and $(N,v_2)$. Define the combined game $(N,v_1+v_2)$ such that its value function equals $[v_1+v_2](C) = v_1(C) + v_2(C)$ for all coalitions $C \subseteq N$.
Then, the axiom mandates the shapley values of these games shall obey
\begin{align}
    \phi_i(v_1+v_2) = \phi_i(v_1) + \phi_i(v_2),
\end{align}
for all players $i \in N$.
\end{axiom}

These \cref{axio:shap-sym,axio:shap-carr,axio:shap-linearity} were used in the original paper (see Chapter 2 of \cite{shapley1967balanced}) that introduced Shapley value.
However, there are some restatements of some of these axioms as given below.

\begin{axiom}[Symmetry, ``equal treatment of equals"]
\label{axio:shap-sym-eqtoeq}
For any two players $i,j \in N$ such that $v(S \cup \{i\}) = v(S \cup \{j\})$ for all $S \in N \setminus \{i,j\}$, 
their Shapley values $\phi_i(v) = \phi_j(v)$ are equal.
\end{axiom}

\begin{axiom}[Null player]
\label{axio:shap-dummy}
A player $i \in N$ is said to be a \emph{`dummy'} player (or a null player) if he adds no value to any coalition beyond his individual value, i.e., $v(S \cup \{i\}) = v(S) + v(\{i\})$ for all $S \subseteq N$.
The Shapley value of a dummy player is $\phi_i(v) = v(\{i\})$, his individual value.
\end{axiom}
\noindent The Null player axiom is also popularly stated by additionally assuming $v(\{i\})=0$ in the definition of a dummy player.

% RAMA : Some sources \thought{wikipedia.. and ...} use \cref{axio:shap-sym-eqtoeq,axio:shap-dummy,axio:shap-linearity}.

% \rama{How are the Symmetry \cref{axio:shap-sym,axio:shap-sym-eqtoeq} equivalent???}

\begin{theorem}[Theorem 7 of \cite{shapley1971cores}]
\label{lem:shap-val-in-core-cvx-game}
The Shapley value of a convex game belongs to the core.
\end{theorem}

The above Theorem also immediately implies that the core of a convex game is non-empty (and the game admits a stable grand coalition).

\subsection{Mathematical aspects}
\begin{definition}[Convex set]
\label{dfn:convex-set}
Let $V$ be a Euclidean space. A set $C \subseteq V$ is said to be \emph{convex} if for every $x,y \in C$ and $\lambda \in [0,1]$, the element $\lambda x + (1-\lambda) y \in C$.
\end{definition}

We now relax the notion of convexity to the follwing weaker notion:
\begin{definition}[Star-shaped set]
\label{dfn:star-shaped-set}
Let $V$ be a Euclidean space. A set $S \subseteq V$ is said to be \emph{star-shaped} if there exists some point/element $x \in S$ such that for every $y \in S$,
the line segment $\overline{xy}$ lies within $S$, i.e.,
$\forall \lambda \in [0,1] : \lambda x + (1-\lambda) y \in S$.

The set of all such $x$ is called the \emph{kernel} of set $S$.
\end{definition}

\noindent A convex set is also star-shaped, and it can be seen that the entire convex set is its own kernel.

\noindent These star-shaped sets in the two-dimensional Euclidean plan are known as star-shaped polygons and appear to be well studied in the computational geometry literature (see e.g., the 1985 textbook of \cite{preparata2012computational}).

\begin{claim}[Kernel-centred Gaussian distributions over star-shaped polytopes]
\label{clm:star-polytope-gaussian-probs}
Consider a star-shaped polytope $S$ with $\theta^*$ in its kernel. Let $A \sim \gaussian{\theta^*}{V_A^{-1}}$ and $B \sim \gaussian{\theta^*}{V_B^{-1}}$ be gaussian random variables with $V_B^{-1} \succeq V_A^{-1}$, i.e., $A$ is tightly centred around $\theta^*$ than $B$ is, in every direction.

Then, $\prob{A \in S} \geq \prob{B \in S}$.
\end{claim}
\begin{proof}
Write $f_A$ (and $f_B$) to be the p.d.f of $\gaussian{\theta^*}{V_A^{-1}}$ (sim. $\gaussian{\theta^*}{V_A^{-1}}$) in $d$ dimensions.
We show the Claim by comparing the corresponding probability integrals in the polar system.
We start with the L.H.S. 
\begin{align*}
    \prob{A \in S} & = \int_{s \in S} f_A(s) .ds = \int_{s \in S}  (2 \pi)^{d/2} \det(V_A^{-1})^{-1/2} \exp{\frac{-1}{2} \norm{s - \theta^*}_{V_A} } .ds \\
    & \labelrel{=}{i:10-a} \int_{\phi} \int_{r = 0}^{r_\phi} (2 \pi)^{d/2} \det(V_A)^{1/2} \exp{\frac{-1}{2} \norm{r}_{V_A} } . dr . d \phi \\
    & \labelrel{\geq}{i:10-b} \int_{\phi} \int_{r = 0}^{r_\phi} (2 \pi)^{d/2} \det(V_B)^{1/2} \exp{\frac{-1}{2} \norm{r}_{V_B} } . dr . d \phi \\
    & \labelrel{=}{i:10-c} \int_{s \in S} f_A(s) .ds = \prob{B \in S},
\end{align*}
where \eqref{i:10-b} is due to $V_A \succeq V_B$.
% \thought{and because of Prékopa–Leindler inequality?}
Further, crucially, the conversion to and from the polar systems in \eqref{i:10-a} and \eqref{i:10-c} over a continuous $[0,r_\phi]$ is possible because $S$ is star-shaped with $\theta^*$ in its kernel, i.e., a ray originating from $\theta^*$ in any direction $\phi$ exits the polytope $S$ at most once (and never enters again), and we call $r_\phi$ to be the distance of this exit point from $\theta^*$ (which is $\infty$ if it doesn't exit). 
\end{proof}

\section{Bandit Algorithm Assumptions}
\label{app-sec:assumptions}

In our work, we made a string of assumptions about the nature and performance of the bandit algorithms to be used. First, we made \cref{assm:regret-strictly-concave} about the black-box algorithm that our  \cref{alg:ident-meta-alg} used for the Homogenous setting (\cref{sec:ident-agents}).
Seond, we made \cref{assm:samples-pooled,assm:large-coalition-less-regret} about the behaviour of the multi-agent algorithm and showed that such an algorithm enjoys nice theoretical properties (\cref{sec:non-ident-agents}).

In this section, we discuss those three assumptions and their practicality in \cref{appn-subsec:justify-strictly-concave-assm,appn-subsec:justify-samples-pooled-assm,appn-subsec:justify-lclr-assm}.
And finally, in \cref{appn-subsec:metc-for-assm} we give a multi-agent bandit algorithm for the heterogeneous setting that obeys the two assumptions assumed for the theoretical results. 

\subsection{On \cref{assm:regret-strictly-concave}}
\label{appn-subsec:justify-strictly-concave-assm}

\assmSingleAgentConcaveRegret*

% THE BELOW WAS THE INITIAL FORUMULATION OF THE ASSUMPTION.
% \begin{enumerate}
%     \item The algorithm $\singalg$ is a `zero-regret' strategy. For all problem instances, $R_a(\singalg, I, T) = o(T)$.
%     \item  With algorithm $\singalg$, the (expected) regret $R_a(\singalg, I, T)$ is strictly concave. \thought{above a certain time-step $T>..$ . Initially, it can be linear or worse.}
% \end{enumerate}

The Assumption constraints the nature of regret the single agent bandit algorithm $\singalg$ to be used shall incur.
Namely, the change of instantaneous regret with time (or the second derivative of the cumulative regret) is upper bounded and lower bounded.

\paragraph{Upper Bound.}
\cref{eqn:assm-sing-upper} implies that the total regret $R_.(\singalg,I,T)$ is strictly concave as a function of time horizon $T$, i.e., the instantaneous regret (or the first derivative of the cumulative regret) shall decrease with time.
In other words, the algorithm $\singalg$ `improves with time'.

Basic bandit algorithms such as Explore-Then-Commit (also known as Explore-First) and $\varepsilon$-Greedy approaches are known to achieve a regret of $O(T^{2/3})$. More sophisticated algorithms, such as Successive Elimination, UCB1 \citep{auer2002finite}, and Thompson Sampling \citep{kaufmann2012thompson}, all achieve $O(T^{1/2})$ regret rates in finite time.
All these functional forms are strictly concave, albeit these bounds are worst-case in nature.

Further, there is ample empirical evidence to suggest that the assumption of concavity is natural. 
It can be seen that empirical cumulative regret (typically averaged over a few repetitions/runs) follows a strictly concave growth with time.
For visual plots, we refer the reader to \cite{russo2018tutorial,garivier2011kl,chapelle2011empirical} that demonstrate the aforementioned nature of shape of regret curve. 
% and we do not perform any experiments to justify this.

\paragraph{Lower Bound.} It is also known that any bandit algorithm can not perform better than a certain level, i.e., the algorithm will have to incur some minimum rate of (expected) regret. 
In other words, the cumulative regret curve `can not flatten too soon'
% or instantaneous regret curve `can not vanish too soon'.
The seminal work of \cite{lai1985asymptotically} establishes an asymptotic lower bound of $\Omega(\log t)$ to the minimizable regret.
And the second derivative of the members of the family of logarithmic curves is of the form $-ct^{-2}$ for some constant $c$ which is used in \cref{eqn:assm-sing-lower} with a small $\varepsilon$ margin. 

% of the multi-agent learning algorithm does not degrade with addition of more agents in \cref{assm:large-coalition-less-regret}.
% We try to justify that this "The more the merrier"-style assumption is mild and very natural.

% \cite{wang2019distributed} consider the multi-agent linear bandit setting that we consider, where the agents are assumed to be not strategic and the focus is on minimizing communication among the agents. They design $\textsc{DisLinUCB}$, a distributed algorithm based on the OFUL (Optimism in the Face of Uncertainty) principle.

\subsection{On \cref{assm:samples-pooled}}
\label{appn-subsec:justify-samples-pooled-assm}

We recap the space of histories of agent arm play and observations.
$H_{a,t}:=(x_{a,s},y_{a,s})_{s=1}^{t}$ 
denotes the single-agent history, i.e., the sequence of actions played $x_{a,s} \in X_{a,s}$ and rewards observed $y_{a,s} \in \Real$, by agent $a$ upto time $t$.
Further, the coalition history $H^C_{t-1}:=\{H_{b,t-1} : b \in C\}$ comprises of single-agent histories corresponding to all agents in the coalition $C$ upto the time-step $t-1$.

Let the space of all histories be defined as follows: $\hist_{a,t} = \bigtimes_{s=1}^t \left( X_{a,s} \times \Real \right) \subseteq \left( \Real^d \times \Real \right)^t$ be the family of single-agent histories of length $t$. And, let 
$$\hist^C_{t} \in \bigtimes_{a \in C} \bigtimes_{s=1}^t \left( X_{a,s} \times \Real \right) \subseteq \left( \Real^d \times \Real \right)^{\cardinality{C} \times t}$$ be the family of all coalition histories of length $t$.
And our problem setting permits that a multi-agent algorithm $\alg$ can take as input this coalition history to come up with actions to play for all the agents in the coalition.
Precisely, for all $t \in [T]$,
$\alg^t : \hist^C_{t-1} \mapsto \left( \cross_{a \in C} X_{a,t}  \right)$.
\footnote{If the algorithm is non-deterministic, the co-domain becomes a distribution over the joint action space, $\Delta \left( \cross_{a \in C} X_{a,t} \right)$, but this distinction is not important to the point we make.}

However, we constrain the multi-agent algorithm to use a different quantity as follows:
\assmSamplesPooled*
This Assumption states that the multi-agent algorithm shall not use coalition history $H^C_t \in \hist^C_{t}$ directly, but shall instead use a further processed quantity $P^C_{t}$, which is the union of samples observed so far by all agents. 
It is easy to observe that this pool of samples $P^C_{t}$ is a deterministic function of the coalition history $H^C_t$.
By data processing inequality, this Assumption doesn't make the algorithm use improved information in making the choice of actions to play. 
On the contrary, information is lost. Specifically, the information about the identity of the agent and the time of play is lost in the union operation.

However, it can be argued that this loss of information doesn't affect the performance of the algorithm. 
This is squarely due to the i.i.d. nature of rewards of the multi-agent linear bandit setting, when conditioned on the action.
When agent $a \in \agents$ at time $t \in [T]$ plays action $x_{a,t} = x \in X_{a,t}$, he observes a reward $y_{a,t} = \inprod{\theta^*}{x} + \eta_{a,t}$. Here, the sequence of additive noises $(\eta_{a,t})_{a \in \agents, t \in [T]}$ are independent and identically distributed across agents $a \in \agents$ and time $t \in [T]$. 
In other words, the reward depends only on the action played and does not depend on the agent who plays it or the time at which it is played.
Thus, it is reasonable to lose this information about the identity of agent and time from the samples in the coalition history without impacting the strength of an algorithm.
\paragraph{Examples from literature.}
In fact, this is a common approach that several multi-agent algorithms use, satisfying this Assumption.
For example, the \emph{DisLinUCB} algorithm \citep{wang2019distributed} makes all agents run an Upper Confidence Bound (UCB) based algorithm, wherein the parameter estimate is computed by using the aggregate `design' matrix $V_{a,t} = \sum_{a \in \agents} \sum_{s=1}^t x_{a,s} x_{a,s}^\top$, the sum of outer product of all actions played,
and the `result' matrix $B_{a,t} = \sum_{a \in \agents} \sum_{s=1}^t x_{a,s} y_{a,s}$, the sum of product of action and reward pairs.
Both these quantities do not make use of the identity of agents and time, that information is lost in the summation, and these quantities can be seen to be functions of the pool/union of samples $P^{\agents}_t$.
% RAMA. \thought{look for more algos to cite. Multi-agent Thompson Sampling?}

\subsection{On \cref{assm:large-coalition-less-regret}}
\label{appn-subsec:justify-lclr-assm}
\assmLargeCoalitionLessRegret*

% First, we show that this assumption can be satisfied in a minimax sense, on a sub-class of the problem instances.
First, we give examples from the literature that argue that given `enough heterogeneity', collaboration does lower regret of agents.
Second, we point to some numerical simulations to verify that the assumption is satisfied in most cases in our MovieLens problem instance experiments.

\paragraph{Related Work - Heterogeneity helps implicitly explore.} 
There is a line of work \citep{bastani2017exploiting,kannan2018smoothed,wang2023achieving} which studies how heterogeneity in action sets inherently helps exploration and thus helps minimize regret.
Specifically, \cite{kannan2018smoothed} consider the stochastic linear bandit setting as ours, where the action sets can be set by an adversary, but is then perturbed component-wise by i.i.d gaussian noise before being presented to the agent.
That is, at time $t$, for any given arbitraty action set $X'_{t} = \{ x'_{1,t}, \dots, x'_{k,t} \}$ with actions in $\Real^d$, they are perturbed with i.i.d. noise vectors $\eta_i \sim \gaussian{0}{\sigma^2 I_d}$ for $i \in [k]$ as follows:
$ X_t = \{ x_{i,t} = x'_{i,t} + \eta_i \}_{i \in [k]}.$
Thus, at every time $t$, the action set is made `diverse' by perturbing it in a random direction in $\Real^d$.
More importantly, as this perturbation is independent over time, the action sets are also diverse across time.
Under such a condition, the authors consider the greedy algorithm, where an action is played from $\argmax_{x \in X_t} \thetahat_t^\top x$ to myopically maximize current expected reward based on current parameter estimate $\thetahat_t$. 
They surprisingly show that with high probability, this greedy algorithm attains a regret upper bound of $\bigO{\sqrt{dT}/\sigma^2}$ where $\sigma^2$ is the variance of the component-wise perturbation added to the actions.

Comparing this to our setting, if an agent $a$ has action sets that have good representation in the ambient space $\Real^d$, and the other agents' action sets are not correlated with that of agent $a$, then, agent $a$ benefits from more data as more agents are in the coalition. 

% RAMA- sampath claims some instance condition that's suficient (that this perturbation also achieves) for less regret.

\paragraph{Empirical validation.} From our MovieLens experiments, we exhaustively check if the \cref{assm:large-coalition-less-regret} holds.
We observe that it holds for most agent-coalition pairs. 
We describe these in \cref{appn-subsec:on-satisfying-more-the-merrier}.

\subsection{Algorithm that satisfies \cref{assm:samples-pooled,assm:large-coalition-less-regret}}
\label{appn-subsec:metc-for-assm}

In \cref{thm:non-ident-stable-game,thm:shapley-axioms}, we showed that any multi-agent bandit algorith that satisfies \cref{assm:samples-pooled,assm:large-coalition-less-regret} enjoys desirable properties such as the grand coalition being stable, and the payout obeying all but one of the Shapley value axioms.
In this section, we give an algorithm (\cref{alg:multi-agent-etc}) that provably satisfies these assumptions (\cref{clm:metc-sats-assm-samples-pooled,thm:mul-etc-obeys-lclr}).

Consider the algorithm $\muletc$ (\cref{alg:multi-agent-etc}) based on the Explore-Then-Commit (or Explore-First) paradigm extended to accommodate multiple agents interacting with the bandit environment. 

$\muletc$ runs in two stages---an exploration phase, followed by a commit phase.
First, the exploration phase happens for a set duration of $T'$ time steps. In each time-step in it, the algorithm uses an exploration routine to come up with the determinant-maximizing action $x_{a,t}$ for each agent $a$ to play,
\begin{align}
x_{a,t} \assign \argmax_{x \in X_{a,t}} \det \left( I + V_{a,t-1} + x x^\top \right), \label{eqn:metc-explore}
\end{align}
where $V_{a,t-1} := \sum_{s<t} x_{a,s} x_{a,s}^\top$. All ties are broken in some deterministic manner agnostic of the agent $a$.
Notably, the exploration action shall depend on the actions recommended to (and played by) the agent so far, but shall be independent of the rewards observed by the agent $a$, or the actions and rewards of any other agent.

After $T'$ time-steps, with the actions and rewards of all agents in $A$, the algorithm computes an estimate of the linear parameter using Ordinary Least Squares (OLS), 
\begin{align}
\thetahat_A \assign V^{-1}_A B_A, \text{ where } V_A = \sum_{a \in A} \sum_{t \leq T'} x_{a,t} x^\top_{a,t} \text{, and } 
B_A = \sum_{a \in A} \sum_{t \leq T'} x_{a,t} y_{a,t}. \label{eqn:metc-ols}
\end{align}
Second, in the commit phase, each agent plays the action that maximizes his expected reward as if estimate $\thetahat_A$ is the true parameter value.
In other words, the agents `commit' to estimate $\thetahat_A$ for the rest of play.

\begin{algorithm}[h]  %[h]
\small
\caption{ $\muletc$, an algorithm for multi-agent linear bandits.}
\label{alg:multi-agent-etc}
{\bf Input : } A set of collaborating agents $A$, time horizon $T$, action sets $X_{a,t}$ for all agents $a \in A$ and time-steps $t \in [T]$, an exploration threshold $T' < T$.
\begin{algorithmic}[1]
\For{time-step $t = 1,2,\dots,T'$}
    \Comment{Exploration phase.}
    \StatePar{Each agent $a \in A$ plays action to maximize his exploration `volume' \\ as in \cref{eqn:metc-explore}
    and observes rewards $y_{a,t}$.}
 \EndFor
\StatePar{Compute parameter estimate $\thetahat_A$ using all agents' statistics until $T'$ as in \cref{eqn:metc-ols}.}
\label{aline:metc-ols}
\For{time-step $t = T'+1,\dots,T$}
    \Comment{Commit phase.}
    \StatePar{Each agent $a \in A$ plays action 
    $ x_{a,t} \assign \argmax_{x \in X_{a,t}} x^\top \thetahat_A.$}
\EndFor
\end{algorithmic}
\end{algorithm}

\begin{claim}
\label{clm:metc-sats-assm-samples-pooled}
$\muletc$ obeys \cref{assm:samples-pooled}.  
\end{claim}
\begin{proof}
We prove the claim by showing that at each time $t \in [T]$, the action to be played $x_{a,t}$ depends on the action-reward sequences from self-play or from other agents' play, $P^a_{t-1}$ and $P^{-a}_{t-1}$.

First, in the exploration phase, we have that the action played 
\begin{align*}
x_{a,t} =& \argmax_{x \in X_{a,t}} \det \left(  I + V_a^{t-1} + x x^\top \right) 
= \argmax_{x \in X_{a,t}} \det \left(  I + x x^\top + \sum_{s=1}^{t-1} x_{a,s} x_{a,s}^\top \right) \\
= & \argmax_{x \in X_{a,t}} \det \left(  I + x x^\top + \sum_{(x',y') \in P^a_{t-1}} x' x'^\top \right). \numberthis \label{eqn:explore-from-pool}
\end{align*}
Second, in the commit phase, the action played at any time $t > T'$ depends on the estimate $\thetahat_A$ which in turn is 
\begin{align*}
\thetahat_A =& V_A^{-1} B_A = \left( \sum_{b \in A} \sum_{s \in [T']} x_{b,s} x_{b,s}^\top \right)^{-1} \left( \sum_{b \in A} \sum_{s \in [T']} x_{b,s} y_{b,s} \right) \\
=& \left( \sum_{(x,y) \in P^a_{T'}} x x^\top + \sum_{(x,y) \in P^{-a}_{T'}} \right)^{-1} \left( \sum_{(x,y) \in P^A_{T'}} x y \right), \numberthis \label{eqn:commit-from-pool}
\end{align*}
where $P^a_{T'}$ (sim. $P^{-a}_{T'}$) is a prefix (a function) of $P^a_{t-1}$ (sim. $P^{-a}_{t-1}$) as $t>T'$.

From \cref{eqn:explore-from-pool,eqn:commit-from-pool}, 
it is seen that for any agent $a \in A$, time $t \in [T]$,
the action played $x_{a,t}$ is a function of the $X_{a,t}$, $P^a_{t-1}$, and $P^{-a}_{t-1}$, satisfying the Assumption.
\end{proof}

\begin{theorem}
\label{thm:mul-etc-obeys-lclr}
$\muletc$ obeys \cref{assm:large-coalition-less-regret}. 
That is, for any problem instance $I$, two coalitions $S \subseteq Q \subseteq \agents$, and any agent $a \in S$, it holds that $R_a^Q(\muletc, I, T) \leq R_a^S(\muletc, I, T)$.
\end{theorem}
\begin{proof}
Write the regret of agent $a \in A$ (for any generic coalition $A \subseteq \agents$) to be the sum of regret in the exploration phase and the regret in the commit phase.
\begin{align}
R^A_a(I,\muletc,T) = R^A_a(I, \muletc, T') + R^A_a(I, \muletc, [T'+1,T]). \label{eqn:muletc-regret-in-two-phase}
\end{align}

\paragraph{Regret in Exploration phase.} For every agent $a \in \agents$, the exploration routine in \cref{eqn:metc-explore} chooses actions $x_{a,t}$ independent of the presence of the other agents.
Thus, the actions played and thus the regret incurred by agent $a$ as a part of any coalition is identical.
That is, for coalitions $S \subseteq Q$, and any agent $a \in S$, 
$R^Q_a(I, \muletc, T') = R^S_a(I, \muletc, T')$.

\paragraph{Regret in Commit phase.} To show the Theorem, what remains to be shown is that the regrets in the commit phase obey 
\begin{align}
R^Q_a(I, \muletc, [T'+1,T]) \leq R^S_a(I, \muletc, [T'+1,T]) 
\end{align}

We shall show this by arguing that the estimate $\thetahat_Q$ is `more accurate' than $\thetahat_S$ and with these estimates, the probability of playing the sub-optimal actions in the commit phase is lesser when the agent is part of the bigger coalition $Q$ (\cref{lem:subopt-less-probable}).

We setup some notations.
W.l.o.g., number the actions in $X_{a,t}$ by decreasing order of optimality as $x_1$ (optimal), $x_2, \dots, x_k$, where $k = \cardinality{X_{a,t}}$.
Let $\pq{\cdot}$ and $\ps{\cdot}$ be the probability measures of the algorithmic trajectory (or history) induced by running $\muletc$ with coalitions $Q$ and $S$ respectively.

\begin{lemma}[Sub-optimal action play probabilities]
\label{lem:subopt-less-probable}
At any time $t \in [T'+1,T]$, for agent $a \in S \subseteq Q$ and for any $1 \leq i \leq k$, the probability of choosing action at least as inferior as $i$ when $\muletc$ is run with bigger coalition $Q$ is at most the corresponding probability when run with smaller coalition $S$. That is,
\begin{align*}
    \pq{x_{a,t} \in \{x_i, x_{i+1}, \dots, x_k \} } \leq \ps{x_{a,t} \in \{x_i, x_{i+1}, \dots, x_k \} }.
\end{align*}
\end{lemma}

\begin{proof}
Towards proving this, we develop some constructs using a generic coalition $A$ and shall later instantiate them using coalitions $S$ and $Q$.
At any time $t \in [T'+1,T]$, in a generic coalition $A$, introduce the following collection of sets $\{C_{A,x,t} \}_{x \in X_{a,t}}$, where we define
\begin{align}
C_{A,x,t} = \{\theta' \in \Real^d : x = \argmax_{x \in X_{a,t}} x^\top \theta' \} \label{eqn:play-param-set}
\end{align}
to be the set of values that estimate $\thetahat_A$ should belong to for action $x$ to be played by agent $a$ at time $t$ by $\muletc$.
% In the definition of $C_{A,x,t}$, the ties in $\argmax$ are broken as per $\muletc$.
Precisely, 
\begin{align}
    \thetahat_A \in C_{A,x,t} \iff x_{a,t} = x, \label{eqn:action-prob-estimate-prob}
\end{align}
and we shall study the probability that a certain action $x$ is played by studying the probability that the computed estimate $\thetahat_A$ lies in set $C_{A,x,t}$ corresponding to action $x$.

\begin{claim}[Convex cone]
\label{clm:convex-cone}
For any action $x \in X_{a,t}$ of agent $a$ at time $t$, 
the set $C_{A,x,t}$ is a convex cone.
\end{claim}
\begin{proof}
The Claim follows from the linearity of the $x^\top \theta'$ dot-product function being optimized.

First, it is a cone since for all $\theta' \in C_{A,x,t}$, we have $c \theta' \in C_{A,x,t}$ for any positive constant $c>0$, as
\begin{align*}
x^\top \theta' > z^\top \theta' \iff x^\top (c \theta') > z^\top (c \theta')
\end{align*}
for all other actions $z \neq x$.

Second, it is convex since for any $\theta', \theta'' \in C_{A,x,t}$, we have that $c_1 \theta' + c_2 \theta'' \in C_{A,x,t}$ for positive constants $c_1, c_2 \geq 0$, as
\begin{align*}
(x^\top \theta' > z^\top \theta') \land (x^\top \theta'' > z) & \implies (c_1. x^\top \theta' + c_2 . x^\top \theta'' > c_1 . z^\top \theta' + c_2 . z^\top \theta'') \\
& \implies (x^\top (c_1+c_2) \theta' > z^\top (c_1+c_2) \theta'),
\end{align*}
for all other actions $z \neq x$.
\end{proof}
Geometrically, the collection $\{C_{A,x,t} \}_{x \in X_{a,t}}$ partitions $\Real^d$ into pie slices with apex/centre at origin.

Next, we claim that the estimate from the bigger coalition $\thetahat_Q$ has a lesser variance than that of the smaller coalition $\thetahat_S$ in all directions:
\begin{claim}[Bigger coalition $\Rightarrow$ Tighter estimate]
\label{clm:q-less-var-than-s}
The estimates obey the gaussian distributions
$\thetahat_S \sim \gaussian{\theta^*}{V_S^{-1}}$ and $\thetahat_Q \sim \gaussian{\theta^*}{V_Q^{-1}}$, with $V_S^{-1} \succeq V_Q^{-1}$.
\end{claim}
\begin{proof}
Under the two coalitions $S$ and $Q$, the OLS estimates computed (\cref{eqn:metc-ols}) obey the gaussian distributions $\thetahat_S \sim \gaussian{\theta^*}{V_S^{-1}}$ and $\thetahat_Q \sim \gaussian{\theta^*}{V_Q^{-1}}$. And since $S \subseteq Q$, we have that from \cref{eqn:metc-ols} that $V_Q \succeq V_S$ and thus, the covariances $V_S^{-1} \succeq V_Q^{-1}$.  
\end{proof}

To show the Lemma, as observed in \cref{eqn:action-prob-estimate-prob}, 
we shall show that, for all $1 \leq i \leq k$, $\thetahat_Q$ falls in the corresponding union of sets with higher probability than $\thetahat_S$ does in the following claim:
\begin{claim}
    $\pq{\thetahat_Q \in \bigcup_{j=1}^i C_{A,x_j,t}} \geq \ps{\thetahat_S \in \bigcup_{j=1}^i C_{A,x_j,t}}$ for all $1 \leq i \leq k$.
\end{claim}
\begin{proof}
With the nature of distributions of $\thetahat_Q, \thetahat_S$ as in \cref{clm:q-less-var-than-s},
due to \cref{clm:star-polytope-gaussian-probs},
it is sufficient to show that the membership sets $C_{A,x_.,t}$s are star-shaped polytopes (\cref{dfn:star-shaped-set}) with its kernel containing $\theta^*$, the mean of the distributions.

To start, for $i=1$, $C_{A,x_1,t}$ is a convex set (\cref{clm:convex-cone}) and is thus star-shaped, so the entire set is its own kernel, and $\theta^*$ belongs to this set and is thus in its kernel.

For $i \in [2,k]$, we show this using proof by contradiction.
Assume for some $i$ that $\bigcup_{j=1}^i C_{A,x_j,t}$ is not a star-shaped polytop w.r.t kernel point $\theta^*$.
Then, there exists some ray originating at $\theta^*$ and traversing some $\theta_b$ and $\theta_a$ (in that order) such that 
$\theta_a \in C_{A,x_{a},t}$ and $\theta_b \in C_{A,x_{b},t}$ with $b > i \geq a$, i.e., action $x_a$ is more optimal than action $x_b$.
In other words, using observation \cref{eqn:action-prob-estimate-prob}, there is some direction in which as the estimate $\thetahat_A$ moves away from the true $\theta^*$, the action picked by the algorithm ceases to be the optimal action $x_1$ and changes to the sub-optimal $x_b$ as the estimate gets to $\theta_b$, and then changes to a relatively optimal arm $x_a$ as the estimate moves further away to $\theta_a$. 

As $\theta_b \in C_{A,x_{b},t}$, by its definition \cref{eqn:play-param-set}, 
\begin{align*}
    & \theta_b x_b^\top \geq \theta_b x_a^\top \\
    \labelrel{\implies}{i:7-a} & (c \theta^* + (1-c) \theta_a) x_b^\top \geq (c \theta^* + (1-c) \theta_a) x_a^\top \\
    \implies & c \theta^* x_b^\top \geq c \theta^* x_a^\top + (1-c) \left( \theta_a x_a^\top - \theta_a x_b^\top \right) \\
    \labelrel{\implies}{i:7-b} & c \theta^* x_b^\top \geq c \theta^* x_a^\top \\
    \implies & \theta^* x_b^\top \geq \theta^* x_a^\top, \numberthis \label{eqn:ray-contra}
\end{align*}
where, \eqref{i:7-a} is by linear interpolation for some $0 < c < 1$, \eqref{i:7-b} uses that $\theta_a \in C_{A,x_{a},t}$ to have $ \theta_a x_a^\top - \theta_a x_b^\top > 0$.

Finally, \cref{eqn:ray-contra} implies $x_b$ is more optimal than $x_a$ with $b>a$, which is a contradiction.
This completes the proof of the Claim.
\end{proof}
The above Claim, in conjunction with \cref{eqn:action-prob-estimate-prob}, completes the proof of the Lemma.
\end{proof}

Finally, to complete the proof of the Theorem, we show in \cref{clm:insta-regret-inequality}, for any time $t$, the instantaneous regret of $\muletc$ with coalition $Q$ is no greater than that of $\muletc$ with coalition $S$.

\begin{claim}[Instantaneous regret inequality]
\label{clm:insta-regret-inequality}
For any agent $a \in S \subseteq Q$, at any time $t \in [T'+1, T]$,
\begin{align*}
    \sum_{j=2}^k \pq{x_{a,t} = x_j} \Delta_j \leq \sum_{j=2}^k \ps{x_{a,t} = x_j} \Delta_j,     
\end{align*}
where $\Delta_j := \max_{x \in X_{a,t}} \inprod{\theta^*}{x} - \inprod{\theta^*}{x_j}$ is the sub-optimality `gap' of action $x_j \in X_{a,t}$.
\end{claim}
\begin{proof}

Write short-hands $p^Q_i = \pq{x_{a,t} = x_i}$, $p^Q_{i,k} = \pq{x_{a,t} \in \{x_i, x_{i+1}, \dots, x_k \}}$ and similarly define $p^S_{i}, p^S_{i,k}$.

We show the Claim by induction on action index $i$.  \paragraph{Hypothesis.} 
\begin{equation}
H(i): \sum_{j=i}^k \left( p^Q_j - p^S_j \right) \Delta_j \leq \left( p^Q_{i,k} - p^S_{i,k} \right) \Delta_i.
\end{equation}
\paragraph{Base Case.} Statement $H(k)$ holds as
\begin{align*}
    p^Q_k \cdot \Delta_k - p^S_{k} \cdot \Delta_k = \left(p^S_{k,k} - p^Q_{k,k} \right) \Delta_k.
\end{align*}

\paragraph{Induction Step.} Let $H(i+1)$ be true for some $i+1 \leq k$. We show $H(i)$ is true.
\begin{align*}
    \sum_{j=i}^{k} \left( p^Q_j - p^S_j \right) \Delta_j = & \left( p^Q_i - p^S_{i} \right) \Delta_i + \sum_{j=i+1}^{k} \left( p^Q_j - p^S_j \right) \Delta_j \\
    \labelrel{\leq}{i:6-a} & \left( p^Q_i - p^S_{i} \right) \Delta_i + \left( p^Q_{i+1,k} - p^S_{i+1,k} \right) \Delta_{i+1} \\
    \labelrel{\leq}{i:6-b} & \left( p^Q_i - p^S_{i} \right) \Delta_i + \left( p^Q_{i+1,k} - p^S_{i+1,k} \right) \Delta_{i} = \left( p^Q_{i,k} - p^S_{i,k} \right) \Delta_i. \numberthis \label{eqn:insta-reg-ineq}
\end{align*}
Here, \eqref{i:6-a} upper bounds the second term by using the induction assumption $H(i+1)$, 
then \eqref{i:6-b} is due to $p^Q_{i+1,k} - p^S_{i+1,k} \leq 0$ by \cref{lem:subopt-less-probable} and $\Delta_i \leq \Delta_{i+1}$.
And \cref{eqn:insta-reg-ineq} shows $H(i)$. By mathematical induction, we have that $H(2)$ holds. \\

\noindent Along with the observation that $p^Q_{i,k} - p^S_{i,k} \leq 0$ from \cref{lem:subopt-less-probable}, statement $H(2)$ implies the Claim.
\end{proof}

\noindent This completes the proof of the Theorem.
\end{proof}

\section{Missing Proofs from \cref{sec:ident-agents}}
\label{appn-sec:fixed-proofs}

\subsection{Proof of \cref{thm:identical-agent-actions-convex-game}}

\fixedActionsConvexGame*
\begin{proof}
The result shall be shown in two steps.
First, in \cref{lem:idenalg-single-multi-regret-inequality}, we shall neatly bound the cumulative regret of the agents in a coalition running $\idenalg$ for time $T$ using the analytical regret of the single-agent algorithm $\singalg$ (that $\idenalg$ internally uses) run for a longer time of $mT$, where $m$ is the size of the coalition.
Second, we shall use this neat bound to show that the value function of the collaboration game $v_{\idenalg, I, T}$ is supermodular for large enough values of $T$ to complete the proof. 

\multSingBounds*

\begin{proof}
We describe the two different regret quantities mentioned in the Lemma statement and introduce a third quantity to connect them.
\begin{enumerate}
    \item The realized cumulative regret of all agents, $\sum_{a \in C} R^C_a(\idenalg,I,\cdot)$, the quantity we try to bound from both sides.
    \item The analytical/hypothetical single-agent regret $R_a(\singalg,I,\cdot)$ that is used in the bound terms.
    \item To compare the above two quantities, an intermediate quantity (introduced below in \cref{eqn:buffer-read-regret}) :
    $\overline{R}(\singalg, B,\cdot)$, the `regret' of the black-box decision/action sequence $\overline{x}_\tau$s on interacting with the reward buffer $B$. 
    This is a quantity internal to $\idenalg$ algorithm that depends on actions chosen by $\singalg$, and not a quantity inherent to the bandit problem.
\end{enumerate}

We set up some notations: Let $\taumax$ be the final value of $\tau$ after $\idenalg$ terminates, which is also the number of times $\singalg$ has been fed with rewards fetched/removed from the buffer (in line \ref{aline:get-from-buffer}).
% Write $\overline{y^*_\tau} := \max_{x \in X} \expund{B}{\overline{y}_\tau | \overline{x}_\tau = x}$ to be the optimal (expected) reward achievable.
% 
Now, we define the `regret' of the black-box decision maker upto some step $\tau' \leq \taumax$ to be 
\begin{align}
\overline{R}(\singalg, B, \tau') := \sum_{\tau=1}^{\tau'} \max_{x \in X} \expectationover{B}{\overline{y}_\tau | \overline{x}_\tau = x} - \expectationover{B, \singalg}{ \overline{y}_\tau }, \label{eqn:buffer-read-regret}
\end{align}
where the expectation is over any randomness in the black-box decision-making and in the buffer rewards.

%  RAMA - Should $\overline{R}$ notation be used without $B,\singalg$ params as it's already a part of the $\idenalg$ meta-algo?

The following Lemma equates this quantity $\overline{R}(\singalg, B,\tau')$ to the hypothetical single-agent regret for a similar time period:
\begin{lemma}[Lemma 1 of \cite{howson2024quack}]
\label{lem:howson-quack}
When the bandit rewards obtained depend only on the chosen action (and not the time or agent):
for all time-steps $t \in [T]$, agents $a \in \agents$, actions $x \in X$, the observed rewards $ y_{a,t} \mid (x_{a,t} = x) \stackrel{i.i.d.}{\sim} \inprod{\theta^*}{x} + \eta_{a,t}$.
Then, for any time $\tau' \leq \taumax$, 
% the `regret' of the black-box decisions maker interacting with the buffer is identical to the regret that he would incur 
\begin{align*}
    R_a(\singalg, I, \tau') = \overline{R}(\singalg, B,\tau') .
\end{align*}   
\end{lemma}
\noindent We present a proof in \cref{subsec:proof-of-howson-lemma} for the sake of completion. 
And, our problem setting shares the assumption of the Lemma
wherein the reward of an agent at any time depends only on the action played by him at that time, and not on the history, or the identity of agent, or the time of play.

Next, we compare the cumulative regret of the agents in the coalition, $\sum_{a \in C} R^C_a$ to the black-box regret $\overline{R}(\cdot)$ by studying how many steps the black-box decision maker performs.
After $m$ agents in the coalition have played the bandit instance $I$ for $T$ time-steps, the the number of steps the black-box decision maker interacts with the buffer, $\taumax$, is bounded as follows:
\begin{claim}
\label{clm:tau-max-bounds}
In all algorithmic runs/trajectories, $mT - mK \leq \taumax \leq mT.$
\end{claim}
\begin{proof}
As per the algorithm $\idenalg$, the black-box decision maker necessarily reads (and removes) one sample from the buffer at every step. There are in total $mT$ samples observed by the agents and filled into the buffer, and that is the maximum number of samples the black-box decision maker could've read out of the buffer. This shows the upper bound.

When black-box decision maker wants reward for any action $x$, the agents play this action on the bandit instance \emph{only} when the buffer $B_x$ is empty, and fill it with $m$ samples. No more samples are added to the buffer (the agents don't play this action $x$ again) until the black-box decision maker reads all the samples and empties the buffer for this action.
Thus, at any time, the size (number of yet-to-be-used available samples) of the buffer is at most $mK$, where $K$ is the number of actions in the action set $X$.
Any sample that has been removed from the buffer has necessarily been consumed by the black-box decision maker.
With a total $mT$ samples filled into the buffer and at most $mK$ unused samples remaining at any time, the black-box decision maker has interacted with the buffer for at least $mT-mK$ steps. This gives the lower bound.
\end{proof}

In all trajectories/runs of $\idenalg$, for every sample $y_{a,t}$ obtained by every agent $a \in \agents$ at time $t \in [T]$, one of the below two statements hold.
\begin{enumerate}
    \item It remains in the buffer at the end of time horizon $t=T$. 
    Let $B'$ be the set of agent-time tuples $(a,t)$-s whose samples $y_{a,t}$-s remain in the buffer at the end. Or,
    \item It was consumed by the black-box decision maker at some step $\tau \in [\taumax]$.
    Let $f : [\taumax] \mapsto (\agents \times [T]) \setminus B'$ denote this bijective mapping.
\end{enumerate}

Now, we are ready to bound the cumulative regret of all agents in $C$ as follows:
\begin{align*}
    \sum_{a \in C} R^C_a(\idenalg, I, T) \labelrel{=}{i:2-a} & \sum_{a \in C} \sum_{t=1}^{T} y^* - \expectationover{I, \idenalg}{\inprod{\theta^*}{x_{a,t}}} \\
    \labelrel{=}{i:2-b} & \sum_{\tau=1}^{\taumax} y^* - \inprod{\theta^*}{x_{f(\tau)}} + \sum_{(a,t) \in B'}  y^* - \inprod{\theta^*}{x_{a,t}} \\
    \labelrel{= \ }{i:2-c} & \sum_{\tau=1}^{\taumax} y^* - \inprod{\theta^*}{\overline{x}_\tau}  + \sum_{(a,t) \in B'}  y^* - \inprod{\theta^*}{x_{a,t}}  \numberthis \label{eqn:coll-reg-intermediate} \\
    \labelrel{\leq}{i:2-d} & \sum_{\tau=1}^{\taumax} y^* - \inprod{\theta^*}{\overline{x}_\tau}  + mK \\
    = \  & \overline{R}(\singalg, B, \taumax) + mK \\
    \labelrel{= \ }{i:2-e} & R_a(\singalg, I, \taumax) + mK \labelrel{\leq }{i:2-f} R_a(\singalg, I, mT) + mK. \numberthis \label{eqn:coll-reg-ub}
\end{align*}
Here, \eqref{i:2-a} uses $X=X_{a,t}$ for all $a,t$ from problem instance constraint, and introduces short-hand $y^*:= \max_{x \in X} \inprod{\theta^*}{x}$, 
% X_{a,t}=X$.
and in \eqref{i:2-b}, we split the regret into two terms based on whether the rewards obtained were consumed by the black-box or not.
Then, \eqref{i:2-c} uses the algorithm property that black-box is given a reward for an action $x$ from the buffer which was originally filled with rewards for action $x$.
Then, \eqref{i:2-d} uses \cref{clm:tau-max-bounds} to bound $\cardinality{B'} = mT - \taumax \leq mK$ while liberally upper bounding per-time-step regret with $1$. 
Then, \eqref{i:2-e} is by \cref{lem:howson-quack}, 
and finally \eqref{i:2-f} uses \cref{clm:tau-max-bounds} again.

Continuing again from \cref{eqn:coll-reg-intermediate}, we have

\begin{align*}
    \sum_{a \in C} R^C_a(\idenalg, I, T) \geq &  \sum_{\tau=1}^{\taumax} y^* - \inprod{\theta^*}{\overline{x}_\tau} = \overline{R}(\singalg, B,\taumax)  \\
    \labelrel{=}{i:2-g} & R_a(\singalg, I, \taumax) \labelrel{\geq}{i:2-h}  R_a(\singalg, I, mT) - mK, \numberthis \label{eqn:coll-reg-lb}
\end{align*}
where \eqref{i:2-g} is by \cref{lem:howson-quack}, 
and \eqref{i:2-h} uses $\taumax \geq mT - mK$ from \cref{clm:tau-max-bounds} 
and liberally upper bounds per-time-step regret with $1$.

\noindent \cref{eqn:coll-reg-ub,eqn:coll-reg-lb} together show the Lemma.
\end{proof}

Note that the upper and lower bounds in the Lemma only differ by an additive term that is independent of time horizon $T$.

We next try to show supermodularity of the value function.
On any instance $I$, for any two sets of agents $S \subset Q \subseteq M$, with $\cardinality{S}=s, \cardinality{Q}=q$, and any agent $a \in M \setminus Q$, we want to show the inequality 
\begin{align*}
    &  & v(S \cup \{a\}) - v(S)  &  \leq v(Q \cup \{a\}) - v(Q) \\
    & \iff & -v(S) - v(Q \cup \{a\}) & \leq -v(S \cup \{a\}) -v(Q) \\
    & \iff & \sum_{b \in S} R^S_b(\idenalg, T) + \sum_{b \in Q \cup \{a\}} R^{Q \cup \{a\}}_b(\idenalg, T) & \leq \sum_{b \in S \cup \{a\}} R^{S \cup \{a\}}_b(\idenalg, T) + \sum_{b \in Q} R^Q_b(\idenalg, T) \\
    & \labelrel{\Longleftarrow }{i:1-a} & R_a(\singalg, sT) + sK + R_a(\singalg, (q{+}1)T) + (q{+}1)K & \leq R_a(\singalg, (s{+}1)T) - (s{+}1)K + R_a(\singalg, qT) - qK \\
    & \labelrel{\Longleftarrow }{i:1-b} & R_a(\singalg, qT{+}T) - R_a(\singalg, qT) & \leq R_a(\singalg, sT{+}T) - R_a(\singalg, sT) - 4MK \\
    \numberthis \label{eqn:cvx-sing-ag-regret-requirement}
\end{align*}
Here, \eqref{i:1-a} is due to \cref{lem:idenalg-single-multi-regret-inequality}, and \eqref{i:1-b} uses $s+q+1 \leq 2M$.

What remains to be shown is that \cref{eqn:cvx-sing-ag-regret-requirement} holds,
which postulates that for a given problem instance,
the regret of the single-agent algorithm over a time period of $T$ that begins after time $qT$ 
be lesser than 
the regret over a similar duration period that begins after time $sT$
by a margin of $4MK$.

Using the notations introduced in \cref{assm:regret-strictly-concave}, we try to show \cref{eqn:cvx-sing-ag-regret-requirement} as follows:
\begin{align*}
& R_a(\singalg, sT{+}T) - R_a(\singalg, sT) - \left( R_a(\singalg, qT{+}T) - R_a(\singalg, qT) \right) \\
= \, & T \cdot R'(sT, T) - T \cdot R'(qT, T) = T(qT - sT) \cdot - R^{''}(sT, qT-sT, T) \\
\labelrel{\geq}{i:1-c} \, & T^2 \cdot \upsilon_{sT}  \labelrel{\geq}{i:1-f} \frac{4MK}{\min_t \upsilon_t} \cdot \upsilon_{sT} \geq 4MK. \numberthis \label{eqn:convexity-4mk-bound}
\end{align*}
Here \eqref{i:1-c} uses \cref{eqn:assm-sing-upper} (\cref{assm:regret-strictly-concave}), 
and \eqref{i:1-f} uses the constraint on $T$ from the Theorem statement.
We further ascertain that such a sufficiently large $T$ (dependent on $\upsilon_t$s, $M$, and $K$) is feasible as follows:
\begin{align*}
T \geq \sqrt{\frac{4MK}{\min_t \upsilon_t}} \labelrel{\geq}{i:1-g} \sqrt{\frac{4MK}{cT^{-2+\varepsilon}}} = \sqrt{4MK/c} \cdot T^{1-\varepsilon/2}
\iff T > \left( \frac{4MK}{c} \right)^{\nicefrac{1}{\varepsilon}},\\
\end{align*}
a lower bound (r.h.s.) independent of $T$. Here, \eqref{i:1-g} uses \cref{eqn:assm-sing-lower} (\cref{assm:regret-strictly-concave}). 
\cref{eqn:convexity-4mk-bound} shows the Theorem.
\end{proof}

\subsection{Proof of \cref{lem:howson-quack}}
\label{subsec:proof-of-howson-lemma}

This Lemma is originally shown by \cite{howson2024quack}, and we provide here a proof for the sake of completeness.
\begin{proof}
The L.H.S. term can be written from \cref{eqn:coalition-regret} as follows:
\begin{align*}
    R_a(\singalg, I, \tau') & = \sum_{t=1}^{\tau'} \max_{x \in X_t} \inprod{\theta^*}{x} - \expectationover{I, \singalg}{ \inprod{\theta^*}{x_{a,t}} } \\ 
    & \labelrel{=}{i:2} \sum_{x \in X} \left(\max_{x' \in X} \inprod{\theta^*}{x'} - \inprod{\theta^*}{x} \right) \cdot \expectationover{I,\singalg}{\sum_{t=1}^{\tau'} \indicator{x_{a,t} = x}}
    \labelrel{=}{i:3} \sum_{x \in X} \Delta_x \cdot \expectationover{I,\singalg}{\sum_{t=1}^{\tau'} \indicator{x_{a,t} = x}} \numberthis \label{eqn:lhs-1}\\ 
\end{align*}
Here, \eqref{i:2} is due to action set $X$ being fixed across time, then \eqref{i:3} introduces $\Delta_x = \max_{x' \in X} \inprod{\theta^*}{x'} - \inprod{\theta^*}{x}$ to be the time-independent `sub-optimality' gap of action $x$.

The R.H.S term can be written from \cref{eqn:buffer-read-regret} as
\begin{align*}
\overline{R}(\singalg, B, \tau') &\labelrel{=}{i:4} \sum_{\tau=1}^{\tau'} \max_{x \in X_\tau} \expectationover{B}{\overline{y}_\tau | \overline{x}_\tau = x} - \expectationover{B, \singalg}{ \overline{y}_\tau } \\
& \labelrel{=}{i:5} \sum_{\tau=1}^{\tau'} \max_{x \in X_\tau} \inprod{\theta^*}{x} - \expectationover{B, \singalg}{\inprod{ \theta^*}{\overline{x}_\tau} } \\
& \labelrel{=}{i:6} \sum_{x \in X} \Delta_x \cdot \expectationover{B, \singalg}{ \sum_{\tau=1}^{\tau'} \indicator{\overline{x}_\tau = x}}. \numberthis \label{eqn:rhs-1}
\end{align*}
Here, \eqref{i:4} is by \cref{eqn:buffer-read-regret}, and  \eqref{i:5} comes from the nature of algorithm $\idenalg$ as follows: the reward $\overline{y}_\tau | \overline{x}_\tau = x$ is arbitrarily picked from the buffer $B_x$, and any reward that was put into this buffer $B_x$ was from playing the action $x$ on bandit instance $I$ whose expected reward is $\inprod{\theta^*}{x}$.
Then, \eqref{i:6} is from action set $X$ being fixed across time and uses $\Delta_x$ defined in the L.H.S.

From \cref{eqn:lhs-1,eqn:rhs-1}, what remains to be shown is the following claim:
\begin{claim}[Equivalence of expected play-counts]
For all actions $x \in X$, and $\tau' \leq \tau^{max}$,
$$ \expectationover{I,\singalg}{\sum_{t=1}^{\tau'} \indicator{x_{a,t} = x}} = \expectationover{B, \singalg}{ \sum_{\tau=1}^{\tau'} \indicator{\overline{x}_\tau = x}}. $$
\end{claim} 
\begin{proof}
Note that the (random variable) quantities, whose expectations are compared in the Claim, are both functions of the history/trajectory of bandit play.
\sloppy Recollect $H_{a,\tau'} = (x_{a,1}, y_{a,1}, x_{a,2}, y_{a,2}, \cdots, x_{a,\tau'}, y_{a,\tau'})$ (resp. $\overline{H}_{\tau'} = (\overline{x}_1, \overline{y}_q, \cdots, \overline{x}_{\tau'}, \overline{y}_{\tau'})$) to be the histories of hypothetical run of $\singalg$ on instance $I$ (resp. $\singalg$ on buffer $B$).

The action spaces $x_{a,\cdot}, \overline{x}_{\cdot} \in X$ are identical between the two. And the space of rewards $y_{a,\cdot}, \overline{y}_{\cdot} \in \Real$ in instance $I$ and buffer $B$ are also identical as the buffer $B$ is filled with rewards obtained from multi-agent interaction of agent $I$. So, the sample spaces of the two histories are the same, say $\mathcal{H} = (X \times \Real)^{\tau'}$.

We use $h = (x_1,y_2,x_2,y_2,\cdots,x_s,y_s, \cdots, x_{\tau'},y_{\tau'})$ to denote some complete history in $\mathcal{H}$, and $h_s = (x_1, y_1, \cdots, x_s,y_s)$ for $s \leq \tau'$ to denote a prefix of the history.

We start from R.H.S.
\begin{align*}
    & \expectationover{B, \singalg}{\sum_{t=1}^{\tau'} \indicator{\overline{x}_t = x}} \labelrel{=}{i:7} \int_{h \in \mathcal{H}} f(h) \cdot \probunder{B,\singalg}{\overline{H}_{\tau'}= h} dh \\
    \labelrel{=}{i:8} & \int_{h \in \mathcal{H}} f(h) \cdot \prod_{s=1}^{\tau'} \probunder{B}{ \overline{y}_s = y_s \mid \overline{x}_s = x_s, 
    \overline{H}_{s-1}= h_{s-1}} \times \probunder{\singalg}{\overline{x}_s = x_s \mid \overline{H}_{s-1} = h_{s-1}}   \cdot dh \\
    \labelrel{=}{i:9} & \int_{h \in \mathcal{H}} f(h) \cdot \prod_{s=1}^{\tau'} \probunder{B}{ \overline{y}_s = y_s \mid \overline{x}_s = x_s} \times \probunder{\singalg}{\overline{x}_s = x_s \mid \overline{H}_{s-1} = h_{s-1}}   \cdot dh \\
    \labelrel{=}{i:10} & \int_{h \in \mathcal{H}} f(h) \cdot \prod_{s=1}^{\tau'} \probunder{I}{ y_{a,s} = y_s \mid x_{a,s} = x_s} \times \probunder{\singalg}{\overline{x}_s = x_s \mid \overline{H}_{s-1} = h_{s-1}}   \cdot dh \\
    \labelrel{=}{i:11} & \int_{h \in \mathcal{H}} \sum_{t=1}^{\tau'} \indicator{x_{a,t} = x} \cdot \prod_{s=1}^{\tau'} \probunder{I}{ y_{a,s} = y_s \mid x_{a,s} = x_s} \times \probunder{\singalg}{x_{a,s} = x_s \mid H_{a,s-1} = h_{s-1}}   \cdot dh \\
    = & \int_{h \in \mathcal{H}} \sum_{t=1}^{\tau'} \indicator{x_{a,t} = x} \cdot \probunder{I, \singalg}{H_{a,\tau'}=h} dh = \expectationover{I,\singalg}{\sum_{t=1}^{\tau'} \indicator{x_{a,t}=x}}. \numberthis \label{eqn:equiv-play-counts}
\end{align*}

Here, \eqref{i:7} introduces short-hand $f(h):=\sum_{\tau=1}^{\tau'} \indicator{\overline{x}_\tau = x}$.
Then, \eqref{i:8} splits and introduces two measures: $\probunder{B}{.}$ as the randomness in the reward given the action comes from only the buffer $B$ and not the action playing algorithm; and similarly $ \probunder{\singalg}{.}$ as the randomness in the choice of action to play given the history comes from only the algorithm $\singalg$ and not the buffer environment $B$.
Then, \eqref{i:9} is due to the fact that at any time $s$, the reward obtained $\overline{y}_s$ from buffer is conditionally independent to the history $\overline{H}_{s-1}$ of actions and rewards, given the action played $\overline{x}_s$ at that time.
Then, \eqref{i:10} is from the nature of the buffer usage as follows: As the reward $\overline{y}_s | \overline{x}_s = x_s$ picked/removed from buffer $B_{x_s}$ was originally filled by some agent at some time, say agent $b$ at time $s'$, by interacting with bandit instance $I$, 
we have $\probunder{B}{\overline{y}_s = y_s | \overline{x}_s = x_s} =  \probunder{I}{y_{b,s'} = y_s | x_{b,s'} = x_s}$; 
further, as the rewards only depend on the action played and not the agent who played it or the time at which it was played, we have $\probunder{I}{y_{b,s'} = y_s | x_{b,s'} = x_s} = \probunder{I}{y_{a,s} = y_s | x_{a,s} = x_s}$.
In \eqref{i:11}, we replace running variable $\tau$ with $t$, and rename history variable $\overline{H}_s=(\overline{x}_1, \overline{y}_1, \cdots, \overline{x}_s, \overline{y}_s)$ to $H_{a,t} = (x_{a,s}, y_{a,s}, \cdots, x_{a,s}, y_{a,s})$.
Finally, \cref{eqn:equiv-play-counts} shows the Claim.
\end{proof}

This completes the proof of the Lemma.
\end{proof}

\section{Missing Proofs from \cref{sec:non-ident-agents}}
\label{appn-sec:hetero-proofs}

\subsection{Proof of \cref{thm:shapley-axioms}}

\payoutShapleyAxioms*
\begin{proof}
We show that the allocation obeys the different axioms in a series of claims below. We write $v$ and $R^C_a$ in short to denote $v_{\alg, I,T}$ and $R^C_a(\alg, I, T)$ for any coalition $C \subseteq M$.

\begin{claim}[Efficiency]
\label{clm:p-effi}
Payout profile $p$ is efficient. $\sum_{a \in \agents} p_a = v(\agents)$. 
\end{claim}
\begin{proof}
The Claim is immediate as $\sum_{a \in \agents} p_a = - \sum_{a \in \agents} R^M_a = v(\agents)$.
\end{proof}

\begin{claim}[Membership in core]
\label{clm:p-core}
Payout profile $p$ belongs to the core of the collaboration game.
\end{claim}
\begin{proof}
A payout profile belongs to the core if it is efficient and it is coalitionally rational. 
\cref{clm:p-effi} shows $p$ is efficient.
Next, we show coalitional rationality. 
Consider any coalition $S \subseteq \agents$, we have 
\begin{align*}
    \sum_{a \in S} p_a = - \sum_{a \in S} R^M_a \labelrel{\geq}{i:2a} - \sum_{a \in S} R^S_a = v(S),
\end{align*}
where \eqref{i:2a} is due to \cref{assm:large-coalition-less-regret}.
\end{proof}

\begin{claim}[Dummy player]
Payout profile $p$ obeys the dummy player (or null player) axiom, i.e., for all agents $a \in \agents$,
% scavenge below definition, point to appendix.
% Replace "We start with the LHS:" with "When $a$ is a null player",
\begin{align*}
   \left( \forall S \not \ni a : v(S \cup \{a\}) - v(S) = v(\{a\}) \right) \implies p_a = v(\{a\}). 
\end{align*}
\end{claim}
\begin{proof}
We start with the LHS: for all coalitions $S \not \ni a$,
\begin{align*}
    & v(S \cup \{a\}) - v(S) = v(\{a\}) \iff  \left( \sum_{b \in S} R^{S}_b - R^{S \cup \{a\}}_b \right) - R^{S \cup \{a\}}_a = - R_a \\
    \labelrel{\implies}{i:3a} &  - R^{S \cup \{a\}}_a \leq - R_a \implies R_a \leq R^{S \cup \{a\}}_a
    \labelrel{\implies}{i:3b} R_a = R^{S \cup \{a\}}_a, \numberthis \label{eqn:dummy-agent-regret}
\end{align*}
where \eqref{i:3a}, \eqref{i:3b} use $R^{S}_b - R^{S \cup \{a\}}_b \geq 0$ and $R_a \geq R^{S \cup \{a\}}_a$ from \cref{assm:large-coalition-less-regret}.
We conclude the proof by showing the RHS: $p_a = - R^M_a \labelrel{=}{i:3c} -R_a = v(\{a\})$,
% \begin{align*}
%     p_a = - R^M_a \labelrel{=}{i:3c} -R_a = v(\{a\}),
% \end{align*}
where \eqref{i:3c} is due to \cref{eqn:dummy-agent-regret} with a choice of $S = \agents \setminus \{a\}$.
\end{proof}

\begin{restatable}{claim}{claimsymmetry}[Symmetry]
\label{clm:payout-symmetry}
Payout profile $p$ obeys the symmetry axiom.
% , i.e., 
% the shapley value of an agent is the same in the original game and the permuted game.
% For all permutations $\pi$, $\phi_i(v) = \phi_{\pi(i)}(\pi v)$.
% \thought{It's better to not state under-introduced notations and definitions here. Point to the Appendix.} RAMA
\end{restatable}
\begin{proof}[Proof sketch]
Deferring the formal proof to \cref{appn-subsec:symmetry-proof}, we provide a sketch here.
The symmetry axiom mandates that the payout of an agent---in our context, the regret of an agent $a$ in the grand coalition---doesn't change when the agents are relabeled. 
% We shall establish this by showing that the distribution over trajectories of bandit play remains unaltered when agents are relabeled.
We shall rely on \cref{assm:samples-pooled} and argue that all the agent information is removed during the union operation to pool the data, and it is this pool of (anonymized) data that the algorithm uses to determine the actions to play. Conditioned on this pool, the action choice only depends on the agent's action set and not the agent's identity.
As a result, the distribution of trajectories remains unchanged under relabeling of agents, and thus the regret, and by extension, his payout $p_a = -R^M_a$ remains unchanged under relabeling of agents.
\end{proof}
This completes the proof of the Theorem.    
\end{proof}

\subsection{Proof of \cref{clm:payout-symmetry}}
\label{appn-subsec:symmetry-proof}
\claimsymmetry*

\begin{proof}
We setup some notations.
Consider a permutation/bijection $\pi : [M] \mapsto [M]$.
Enumerate the set of agents to be $\agents = \{a_1,a_2,\dots, a_M\}$.
For the agent $a_i \in \agents$ with label/id $i$, let his permuted label/id be $\pi(i) \in [M]$.
Then, construct the permuted set of agents to be $\tilde{\agents} = \{ \tilde{a}_{\pi(i)} = a_i : i \in [M] \}$, and extend all notations defined earlier using an overhead \~{} for this permuted set of agents.
Given some $i$, write short-hands $a = a_i$, $\tilde{a} = \tilde{a}_{\pi(i)}$.
Then, denote by $\tilde{H}_{\tilde{a},t} = \{(\tilde{x}_{\tilde{a},1}, \tilde{y}_{\tilde{a},1}), \dots, (\tilde{x}_{\tilde{a},t}, \tilde{y}_{\tilde{a},t}) \}$ the history/trajectory of actions played and rewards observed by agent $\tilde{a} \in \tilde{\agents}$ upto time $t$ when algorithm $\alg$ is run with the set of agents $\tilde{\agents}$. 

We want to show that the allocation $p_a$ for any agent $a_i \in \agents$ doesn't change when he participates in the collaboration game under a different label as a part of a permuted set.
Specifically, it requires to be shown that for any permuted agent set $\tilde{M}$ (induced by permutation $\pi$), for any agent $a_i \in M$, 
it holds that $R^M_{a_i} = R^{\tilde{M}}_{\tilde{a}_{\pi(i)}}$.
We shall show that this condition follows when the algorithm $\alg$ facilitates (carries out) collaboration among agents as per \cref{assm:samples-pooled}.

First, we claim that the trajectory of an agent $a_i \in \agents$ when $\alg$ is run on set of agents $\agents$ is identically distributed to the trajectory of the corresponding relabeled/permuted agent when $\alg$ is run on the relabeled/permuted set of agents $\tilde{\agents}$. Let $(s_a)_a = ((\bar{x}_{a,1}, \bar{y}_{a,1}), \dots, (\bar{x}_{a,t}, \bar{y}_{a,t}))_{a \in \agents} \in S$ be a collection of action-reward sequences, with $S$ the set of all such possible collection of sequences.
\begin{claim}
\label{clm:traj-iden-dist}
For any $t \in [T]$, for any collection $(s_a)_a \in S$ of action-reward sequences,
% = ((\bar{x}_{a,1}, \bar{y}_{a,1}), \dots, (\bar{x}_{a,t}, \bar{y}_{a,t})))_{a \in \agents}$ of action-reward sequences,
it holds that
\begin{align}
    \probover{I, \alg}{\bigcap_{a \in \agents} H_{a,t} = s_a} = \probover{I, \alg}{\bigcap_{\tilde{a} \in \tilde{\agents}} \tilde{H}_{\tilde{a},t} = s_a}. \label{eqn:identical-trajectories}
\end{align}
\end{claim}
\begin{proof}
We show this by induction on time-step variable $t$. 
Let $h(t)$ be the hypothesis that \cref{eqn:identical-trajectories} holds for $t$.
\paragraph{Base Case.} At $t=0$, the only valid action-reward sequence, for any agent $a \in \agents$ is the empty sequence $s_a=()$, and both $H_{a,0} = s$ and $\tilde{H}_{\tilde{a},0} = s$ with probability $1$. Thus $h(0)$ holds.
\paragraph{Induction Step.} We assume $h(t-1)$ holds and shall show $h(t)$ holds.
Let $s_{a,t-1}:= ((x_1,y_1), \dots, (x_{t-1},y_{t-1}))$ be the sequence of first $t-1$ action-reward tuples in $s_a$.
We start from the L.H.S:  
\begin{align*}
     \probover{I, \alg}{\bigcap_{a \in \agents} H_{a,t} ~\shorteq~ s_a} = & \probover{I, \alg}{\bigcap_{a \in \agents} y_{a,t} ~\shorteq~ \bar{y}_{a,t} | \bigcap_{a \in \agents} x_{a,t} ~\shorteq~ \bar{x}_{a,t}, \bigcap_{a \in \agents} H_{a,t \shortminus 1} ~\shorteq~ s_{a,t \shortminus 1}} \\
     & \cdot \probover{I, \alg}{\bigcap_{a \in \agents} x_{a,t} ~\shorteq~ \bar{x}_{a,t} | \bigcap_{a \in \agents} H_{a,t \shortminus 1} ~\shorteq~ s_{a,t \shortminus 1}} \cdot \probover{I, \alg}{\bigcap_{a \in \agents} H_{a,t \shortminus 1} ~\shorteq~ s_{a,t \shortminus 1}}.    \numberthis \label{eqn:traj-H-prob} 
\end{align*}
Next, we analyse each of these three terms. 
First,
\begin{align*}
    & \probover{I, \alg}{\bigcap_{a \in \agents} y_{a,t} ~\shorteq~ \bar{y}_{a,t} | \bigcap_{a \in \agents} x_{a,t} ~\shorteq~ \bar{x}_{a,t}, \bigcap_{a \in \agents} H_{a,t \shortminus 1} ~\shorteq~ s_{a,t \shortminus 1}} \\
    \labelrel{=}{i:4a} & \probover{I, \alg}{\bigcap_{a \in \agents} y_{a,t} ~\shorteq~ \bar{y}_{a,t} | \bigcap_{a \in \agents} x_{a,t} ~\shorteq~ \bar{x}_{a,t}} 
    \labelrel{=}{i:4b} \prod_{a \in \agents} \probover{I, \alg}{ y_{a,t} ~\shorteq~ \bar{y}_{a,t} | x_{a,t} ~\shorteq~ \bar{x}_{a,t}} \\ 
    \labelrel{=}{i:4c} & \prod_{\tilde{a} \in \tilde{\agents}} \probover{I, \alg}{ \tilde{y}_{\tilde{a},t} ~\shorteq~ \bar{y}_{a,t} | \tilde{x}_{\tilde{a},t} ~\shorteq~ \bar{x}_{a,t}} 
    \labelrel{=}{i:4d} \probover{I, \alg}{\bigcap_{\tilde{a} \in \tilde{\agents}} \tilde{y}_{\tilde{a},t} ~\shorteq~ \bar{y}_{a,t} | \bigcap_{\tilde{a} \in \tilde{\agents}} \tilde{x}_{\tilde{a},t} ~\shorteq~ \bar{x}_{a,t}} \\
    \labelrel{=}{i:4e} & \probover{I, \alg}{\bigcap_{\tilde{a} \in \tilde{\agents}} \tilde{y}_{\tilde{a},t} ~\shorteq~ \bar{y}_{a,t} | \bigcap_{\tilde{a} \in \tilde{\agents}} \tilde{x}_{\tilde{a},t} ~\shorteq~ \bar{x}_{a,t}, \bigcap_{\tilde{a} \in \tilde{\agents}} \tilde{H}_{\tilde{a},t \shortminus 1} ~\shorteq~ s_{a,t \shortminus 1}}. \numberthis \label{eqn:traj-prob-reward}
\end{align*}
The above is due to the independence of the rewards. Specifically, \eqref{i:4a},\eqref{i:4e} use that given the current action, rewards are independent of the historical actions/rewards, \eqref{i:4b},\eqref{i:4d} are from the independence of rewards across different agent-play at the same time, 
and \eqref{i:4c} is from the fact that the reward is a function of the specific action $\bar{x}_{a,t}$ played, and independent of the agent who plays the action (or the time at which it is played) for any problem instance $I$. 

Second,
\begin{align*}
& \probover{I, \alg}{\bigcap_{a \in \agents} x_{a,t} ~\shorteq~ \bar{x}_{a,t} \given \bigcap_{a \in \agents} H_{a,t \shortminus 1} ~\shorteq~ s_{a,t \shortminus 1}} \\
\labelrel{=}{i:4f} & \probover{I, \alg}{\bigcap_{a \in \agents} \left( x_{a,t} ~\shorteq~ \bar{x}_{a,t} \given \left( P^a_{t-1} = ((\bar{x}_{a,s}, \bar{y}_{a,s} ))_{s \in [t-1]} \right) \land \left(
P^{-a}_{t-1} = \left( \cup_{b \in \agents \setminus \{a\}} \{ (\bar{x}_{b,s}, \bar{y}_{b,s} ) \} \right)_{s \in [t-1]} \right) \right) } \\
\labelrel{=}{i:4ff} & \probover{I, \alg}{\bigcap_{a \in \agents} \left( x_{a,t} ~\shorteq~ \bar{x}_{a,t} \given \left( P^a_{t-1} = ((\bar{x}_{\tilde{a},s}, \bar{y}_{\tilde{a},s} ))_{s \in [t-1]} \right) \land \left(
P^{-a}_{t-1} = \left( \cup_{\tilde{b} \in \tilde{\agents} \setminus \{ \tilde{a} \}} \{ (\bar{x}_{\tilde{b},s}, \bar{y}_{\tilde{b},s} ) \} \right)_{s \in [t-1]} \right) \right) } \\
\labelrel{=}{i:4g} & \probover{I, \alg}{\bigcap_{\tilde{a} \in \tilde{\agents}} \tilde{x}_{\tilde{a},t} ~\shorteq~ \bar{x}_{a,t} \given \bigcap_{\tilde{a} \in \tilde{\agents}} \tilde{H}_{\tilde{a},t \shortminus 1} ~\shorteq~ s_{a,t \shortminus 1}}, \numberthis \label{eqn:traj-prob-action}
\end{align*}
where \eqref{i:4f} is due to \cref{assm:samples-pooled} (first condition), and \eqref{i:4ff} is due to $\tilde{M}$ is a permutation of $M$, and \eqref{i:4g} uses that action set of an agent $X_{a,t} = X_{\tilde{a},t}$
% RAMA - Is $\tilde{X}$ needed for the action set after permutation? I think not.
is the same after permuting/relabeling, 
and that the algorithm decides the action $\tilde{x}_{\tilde{a},t}$ for each agent $\tilde{a}$ as a function of 
action-reward sequence over time from self-play $P^{\tilde{a}}_{t-1}$ and an action-reward sequence from an anyonymized union from other agents' play $P^{-\tilde{a}}_{t-1}$
(from second and third conditions of \cref{assm:samples-pooled}).
Third,
\begin{align}
\probover{I, \alg}{\bigcap_{a \in \agents} H_{a,t \shortminus 1} ~\shorteq~ s_{a,t \shortminus 1}} = \probover{I, \alg}{\bigcap_{\tilde{a} \in \tilde{\agents}} \tilde{H}_{\tilde{a},t \shortminus 1} ~\shorteq~ s_{a,t \shortminus 1}} \label{eqn:traj-prob-history}
\end{align}
by the induction assumption. Finally, substituting \cref{eqn:traj-prob-action,eqn:traj-prob-history,eqn:traj-prob-reward} back into \cref{eqn:traj-H-prob} shows us that $h(t)$ holds.
\end{proof}

Now, we are ready to show that symmetry holds. 
The regret of agent $a$ when $\alg$ is run with agents $M$ is given by
\begin{align*}
 R^M_a(\alg,I,T) = & \sum_{t=1}^T \inprod{\theta^*}{x^*_{a,t}} - \expectationover{I, \alg}{ \inprod{\theta^*}{x_{a,t}}} \\
= & \sum_{t=1}^T \inprod{\theta^*}{x^*_{a,t}} - \sum_{(s_a)_a \in S} \left( \sum_{t=1}^T \inprod{\theta^*}{\bar{x}_{a,t}} \right) \probover{I, \alg}{\bigcap_{a \in \agents} H_{a,t} = s_a } \\
\labelrel{=}{i:5b} & \sum_{t=1}^T \inprod{\theta^*}{\tilde{x}^*_{\tilde{a},t}} - \sum_{(s_a)_a \in S} \left( \sum_{t=1}^T \inprod{\theta^*}{\bar{x}_{a,t}} \right) \probover{I, \alg}{\bigcap_{\tilde{a} \in \agents} \tilde{H}_{\tilde{a},t} = s_a } \\
= & R^{\tilde{M}}_{\tilde{a}}(\alg,I,T), \numberthis \label{eqn:regret-symmetry-equivalence}
\end{align*}
where \eqref{i:5b} uses \cref{clm:traj-iden-dist}. \cref{eqn:regret-symmetry-equivalence} completes the proof of the Claim.
\end{proof}

\subsection{On the Linearity axiom.}
\label{appn-subsec:on-linearity}
In \cref{thm:shapley-axioms}, we showed that the payout profile of $p=(p_a = -R^{\agents}_a(\alg, I, T)$ obeys the Shapley's axioms of efficiency, null-player, and symmetry. However, we mentioned that the linearity axiom can not be satisfied. We discuss this further in this sub-section.
% (sometimes called Additivity)

If the linearity axiom is satisfied by our payout $p$, then in conjunction with the other satisfied axioms, $p$ indeed \textbf{\emph{is}} the Shapley value. 
However, given that the Shapley value, by definition, captures the marginal utility of an agent to all possible $2^{M-1}$ different coalitions, it is highly unlikely that a simple payout structure like ours, $p_a = -R^M_a$, which only captures the dynamism associated with the grand coalition, could actually be the Shapley value. 
Thus, we believe the Linearity axiom will not be satisfied by $p$. \\

\paragraph{What would Linearity look like?} 
Consider two bandit scenarios $B_1 = (\alg, I, T_1)$ and $B_2 = (\alg, I, T_2)$.
Consider the corresponding value functions $v_1(C):=v_{\alg,I,T_1}(C) = - \sum_{a \in C} R^C_a(\alg,I,T_1)$ and $v_2:=v_{\alg,I,T_2}(C) = - \sum_{a \in C} R^C_a(\alg,I,T_2)$.

The sum of two value functions is traditionally obtained by the `superposition' of the value functions, i.e., the arithmetic sum of the two individual value function (as if the two games are happening in parallel and do not influence each other) is used to traditionally define the sum of two games, say $w=v_1+v_2$.
Namely, the game $(M,v_1+v_2)$ has the value/characteristic function
\begin{align}
    w(C) = v_1(C) + v_2(C) = - \sum_{a \in C} R^C_a(\alg,I,T_1) + R^C_a(\alg,I,T_2) \label{eqn:sum-of-games}
\end{align}
for all $C \subseteq M$.
Now, how does our definition of payouts $p$ for game $(M,v_1+v_2)$ relate to the underlying bandit scenario of $B_1 + B_2$? What does `+' mean in this context? 
There appears to be (at least) two formulations.
\begin{enumerate}
    \item $p$ is the regret going through bandit scenario $B_1$ (the specific instance $I$, with the specific algorithm $\alg$, for the specific time-steps) and \emph{then} going through $B_2$ with outcome of $B_1$ in memory. 
    In that case, 
    \begin{align}
        p_a(v_1+v_2) = - R^C_a(\alg,I,T_1 + T_2) \label{eqn:lin-p-1}
    \end{align}
    is the regret through total time $T_1 + T_2$.
    Note that this preserves the commutativity of $+$ operation (at least when $\alg$ remains the same).
    \item $p$ is the regret of going through bandit scenario $B_1$ summed up with the regret of going through scenario $B_2$ \emph{independent} of the execution of $B_1$. 
    In that case, 
    \begin{align}
        p_a(v_1+v_2) = - R^C_a(\alg,I,T_1) - R^C_a(\alg,I,T_2) \label{eqn:lin-p-2}
    \end{align}
    Note that this preserves the commutativity of $+$ operation always, also appears more aligned to the notion `superposition'.
    However, here, the payout definition assumes it is possible to decompose $w$ back into $v_1$ and $v_2$ which may not be permissible.
\end{enumerate}
Between $v_1$ and $v_2$, if the instances therein differ with model parameter $\theta^*_1 \neq \theta^*_2$, then the first case above matches the second case.

Next, we ask if these two definitions of $p$ satisfy the Linearity axiom. First,
\begin{align*}
    \cref{eqn:lin-p-1} \implies p_a(v_1+v_2) =& - R^C_a(\alg,I,T_1 + T_2) \\
    \neq & - R^C_a(\alg,I,T_1) - R^C_a(\alg,I,T_2) = p_a(v_1)+p_a(v_2),
\end{align*}
in general. Thus Linearity axiom does not hold in this case.
Second,
\begin{align*}
    \cref{eqn:lin-p-2} \implies p_a(v_1+v_2) = - R^C_a(\alg,I,T_1) - R^C_a(\alg,I,T_2) = p_a(v_1)+p_a(v_2),
\end{align*}
which shows $p$ (the version in \cref{eqn:lin-p-2}) satisfies the Linearity axiom. 

% Which of these definitions, if either, is correct/appropriate? 
But, as discussed above, it is not possible to implement (or realize) such a payout $p$ that requires the decomposition of the sum of two games $v_1+v_2$ back into the original constituent games $v_1, v_2$. 
Thus, we feel that our payout, in its current form, can not satisfy the Linearity axiom.

\section{Additional Numeric Simulations}
\label{appn-sec:experiments}

\begin{figure}[ht]
    \centering
    \begin{subfigure}{.51\linewidth}
        \centering
        \includegraphics[width=\linewidth]{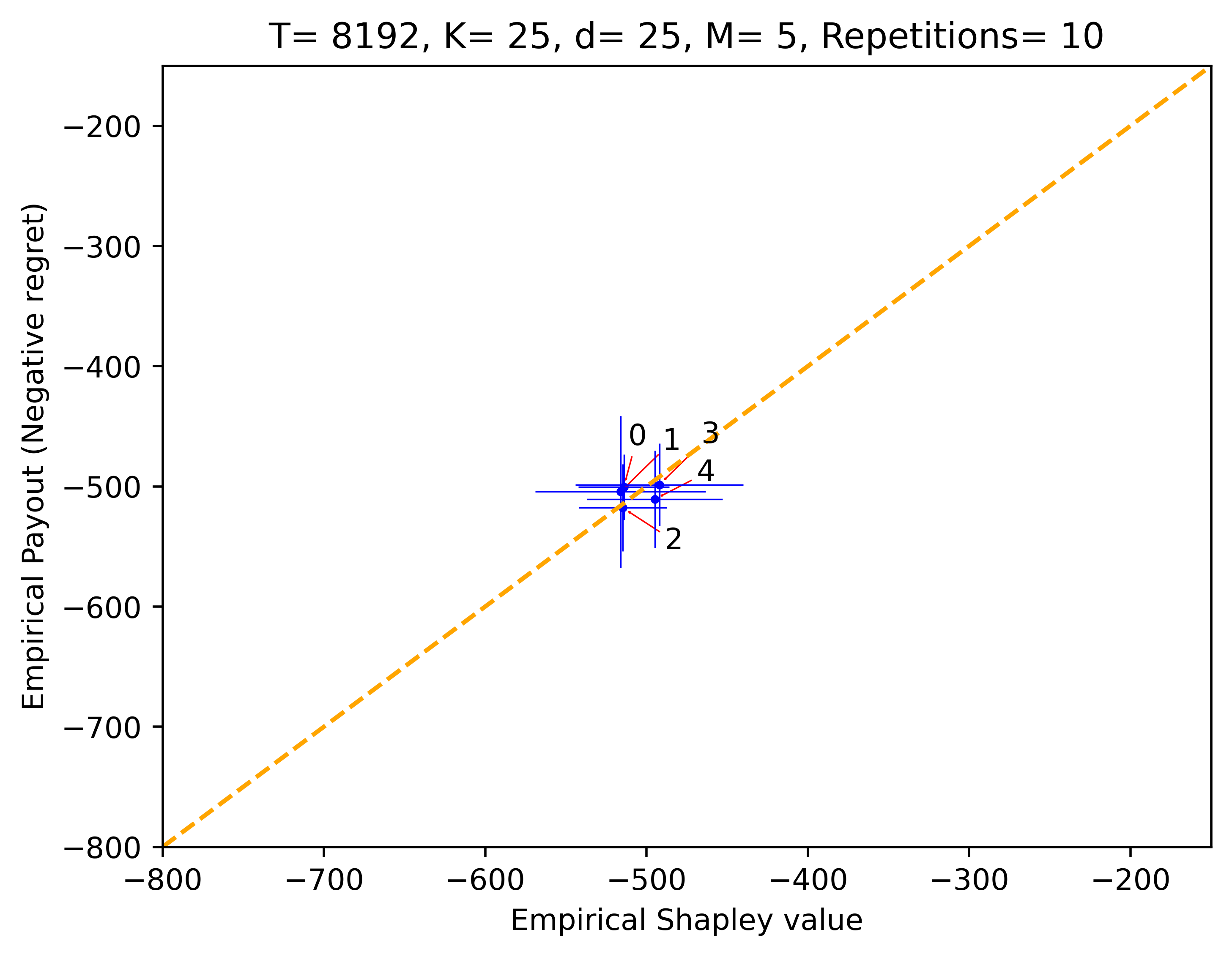}
        \caption{}\label{fig:synth-iden}
    \end{subfigure}%
    \begin{subfigure}{.51\linewidth}
        \centering
        \includegraphics[width=\linewidth]{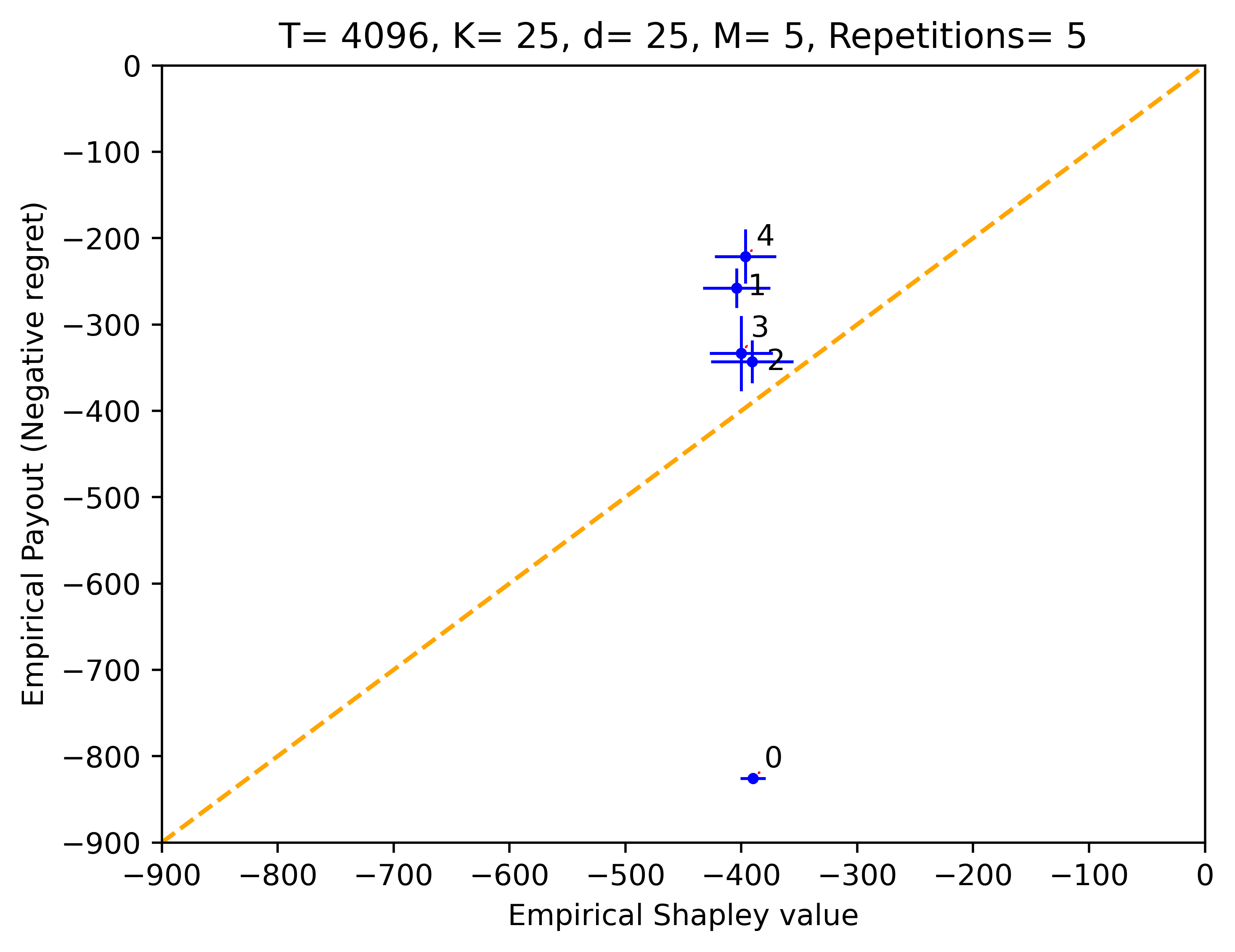}
        \caption{}\label{fig:synth-explo}
    \end{subfigure}
    
    \begin{subfigure}{.49\linewidth}
        \centering
        \includegraphics[width=\linewidth]{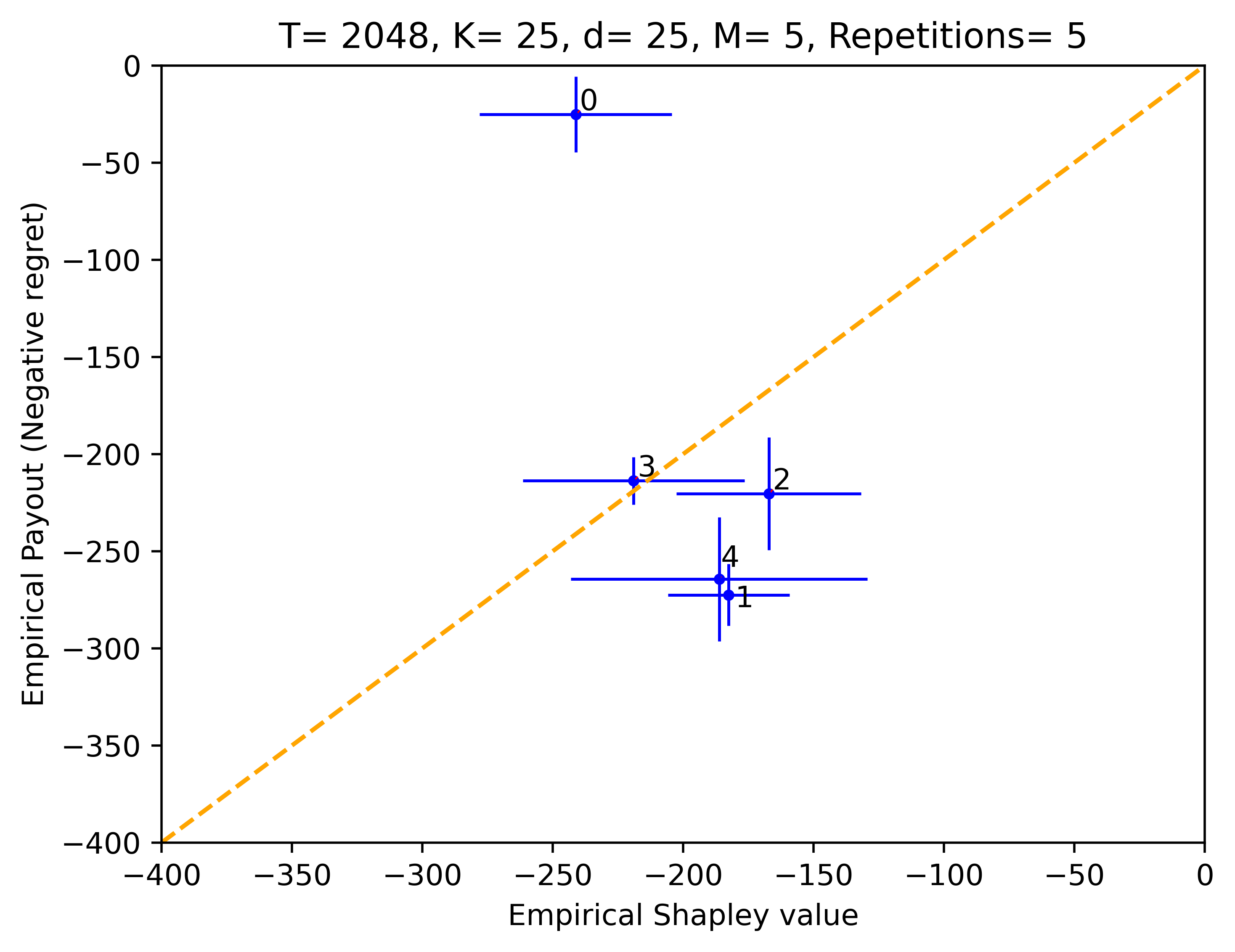}
        \caption{}\label{fig:synth-free}
    \end{subfigure}
    \begin{subfigure}{.49\linewidth}
        \centering
        \includegraphics[width=\linewidth]{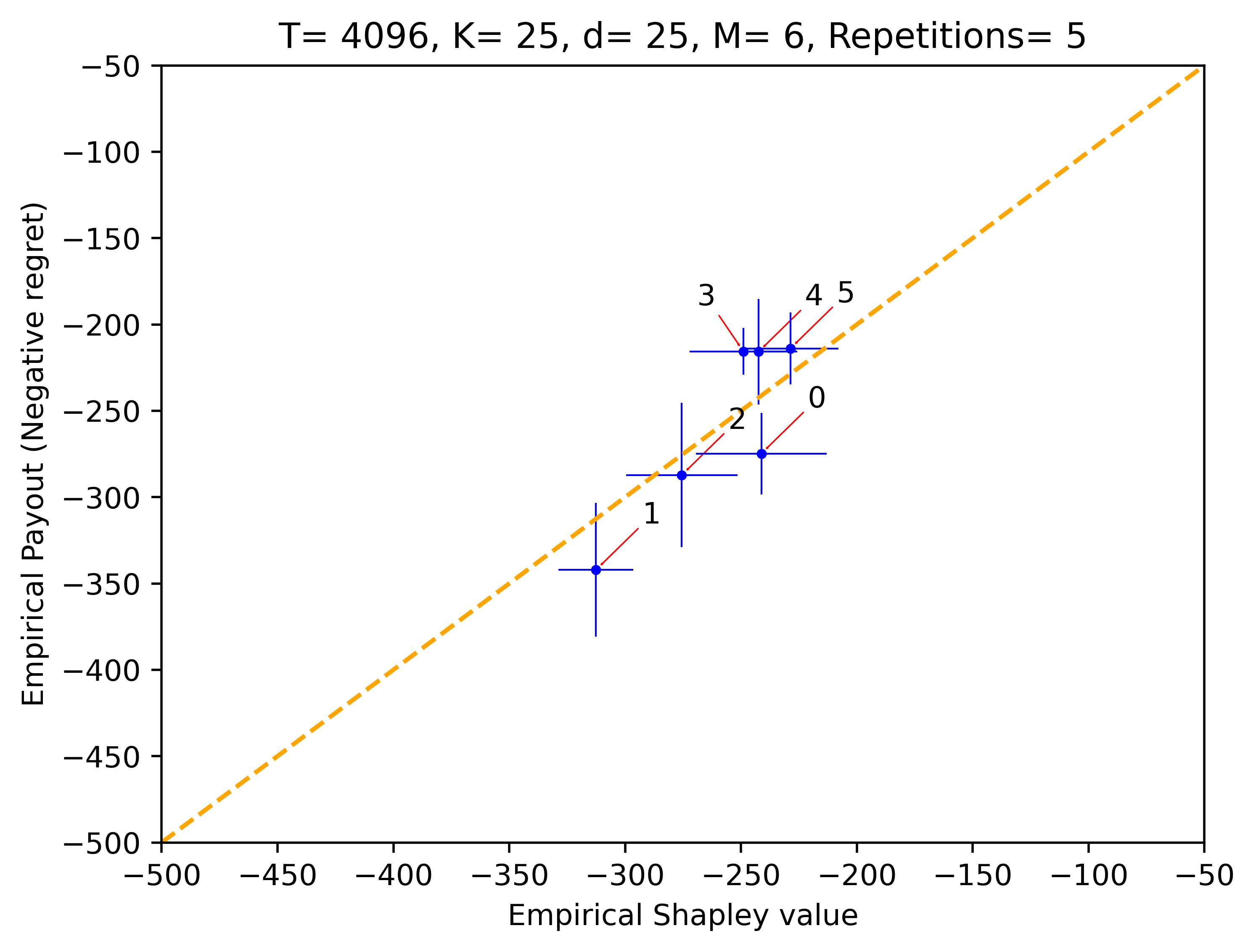}
        \caption{}\label{fig:synth-diag}
    \end{subfigure}
    \caption{\textbf{Synthetic experiments}}
    \label{fig:synth-experiments}
\end{figure}

\subsection{Notes on Implementation}
\label{appn-subsec:notes-on-implementation}
In each bandit instance $I_{attr} \in \inst$ (for some attribute), the action sets $X_{a,1}, X_{a,2}, \dots, X_{a,T}$ for an agent $a$ are generated by jointly embedding users and movies in a $d=100$ dimensional space such that each $X_t$ corresponds to the set of movies that the platform can recommend to a user with attribute value corresponding to that of agent $a$. 
We have $\cardinality{X_{a,t}} = 1682$ for all $t$ for all agents $a$ over all instances $I$. 
The parameter $\theta^* \in \Real^{100}$ for the linear bandits setup in all instances is the minimizer of the least square error between the predicted ratings and the true ratings. 
In all experiments, the reward for playing any arm $x \in X_{a,t}$ is $ \inprod{x}{\theta^*} + \eta$ where $\eta \sim \mathcal{N}(0,1)$ independently drawn every time.
All experiments are run for a time horizon of $T=4096$ for $5$ repetitions, with mean and error bars (cross-hairs) plotted. 
For the \mlu, we build upon the LinUCB implementation provided in the Pearl library \citep{pearl} and extend it to our multi-agent setting.

\subsection{Synthetic Experiments}
\label{appn-subsec:synth-exp}
We perform some synthetic experiments to better elucidate our assumptions and results.

\paragraph{Instance Setup}
We consider the following `symmetric' (not to be confused with the Shapley axiom) bandit instance in terms of the action sets of the agents, whereby all agents are identical.
With ambient dimension $d=25$, there are a total of $K=25$ actions, say $\{1,2,\dots, 25\}$, where each action is a unit vector pointing in a unique dimension. This can be equivalently understood as the vanilla multi-armed bandit with $d$ arms.
The $\theta^* $ parameter  is given by
\begin{align*}
    \theta^*_i = 
    \begin{cases}
        \nicefrac{7}{10} & \text{ if } i \equiv 1 \pmod{5}, \\
        \nicefrac{6}{10} & \text{ if } i \equiv 2 \pmod{5}, \\
        \nicefrac{5}{10} & \text{ if } i \equiv 3 \pmod{5}, \\
        \nicefrac{4}{10} & \text{ if } i \equiv 4 \pmod{5}, \\
        \nicefrac{3}{10} & \text{ if } i \equiv 0 \pmod{5} \numberthis \label{eqn:synth-instance}
    \end{cases}
\end{align*}
for all $i \in [d]$.
Each $\theta^*_i$ can be thought of as the reward means of the action $i$ in the vanilla MAB formulation.
From the expression, one can see that these action subsets $A_1 = \{1,2,\dots, 5\}, A_2 = \{6, \dots, 10 \}, \dots, A_5 = \{21, \dots, 25\}$ all are mutually identical.
We then have each agent posses (or be allocated) $2$ subsets of actions at all times : for all $t$, $X_{0,t} = A_1 \cup A_2, X_{1,t} = A_2 \cup A_3, \dots, X_{4,t} = A_5 \cup A_1$.
Individually, each agent posseses an identical set of $10$ actions, and viewing the agents arranged in a circle, each agent shares a half of the actions with the neighbour on the left, and shares the other half with the neighbour on the right.
All agents are identical w.r.t. their action sets and sharing of action sets with agents.

\paragraph{Outcomes.}
We run three experiments in this set-up: (i) all agents run LinUCB algorithm; (ii) Agent 0 deviates and plays greedily at all times (maximizing expected reward w.r.t. current parameter estimate without exploration term) while other agents continue to run LinUCB; and (iii) Agent 0 deviates and plays actions uniformly at random while other agents continue to run LinUCB. 
The comparison of payouts and Shapley value are plotted in \cref{fig:synth-experiments}.

With all identical agents (\cref{fig:synth-iden}), every agent's payout is remarkably close to his Shapley value. 
Further, the payouts (and Shapley values) are also very similar across all agents.
In \cref{fig:synth-free}, when there is one agent who plays greedily (that agent is called a `free-rider' in some contexts), 
it is seen that the greedy player 0 enjoys incredibly less regret, and has a lower Shapley value than other agents, which implies that he contributes lesser to minimizing the group's regret than other agents do.
% This can be attributed to the 
Further, 1 and 4 are the agents who share actions with free-rider 0, and they have a they have higher regret compared to 2 and 3 who don't share actions with 0.
From \cref{fig:synth-explo}, the explorer suffers very high regret (very low payout). Agents 1 and 4 benefit by having common actions with this explorer, they enjoy lower regret (higher payout) compared to agents 2 and 3 who don't share actions with the explorer. It is also seen that the Shapley values of all agents appear close to each other, the explorer's shapley value is statistically indistinguishable from that of a normal agent.

Finally, in \cref{fig:synth-diag}, we consider a slightly different asymmetric instance. 
Each agent $a \in \{1,2,3,4,5\}$ has an identical set of actions $X_{a,t} = A_a$, with no action overlap among each other.
And we introduce an asymmetric agent, labeled $0$, who shares exactly one action with each of the other five agents. 
Specifically, $X_{0,t} = \{1, 7, 13, 19, 25\}$.
It can be observed from \cref{eqn:synth-instance} that agent 0 shares with agent 1 their respective optimal arms, with agent 2 their respective second optimal arm, and so forth, and finally with agent 5 their respective least optimal arm.

It is seen from \cref{fig:synth-diag} that the empirical payouts and shapley values are well correlated for the agents. 
The asymmetry in agents 1 to 5 offers an interesting insight. Agent 1 has the least shapley and most regret.
This indicates that sharing and receiving information about the optimal action is not of much use to himself or the receiving agent (which in this case is 0). The reason is that each agent could have very well explored this optimal arm by himself by incurring no regret instead of receiving from other agents.
This phenomenon gradually lightens as we move through agents 2,3,4, and 5, who share lesser and lesser optimal arms, and have larger and larger Shapley values and payouts (negative regrets).

\subsection{MovieLens Experiments}
\label{appn-subsec:ml-exp}

\begin{figure}[ht]
    \centering
    \begin{subfigure}{.51\linewidth}
        \centering
        \includegraphics[width=\linewidth]{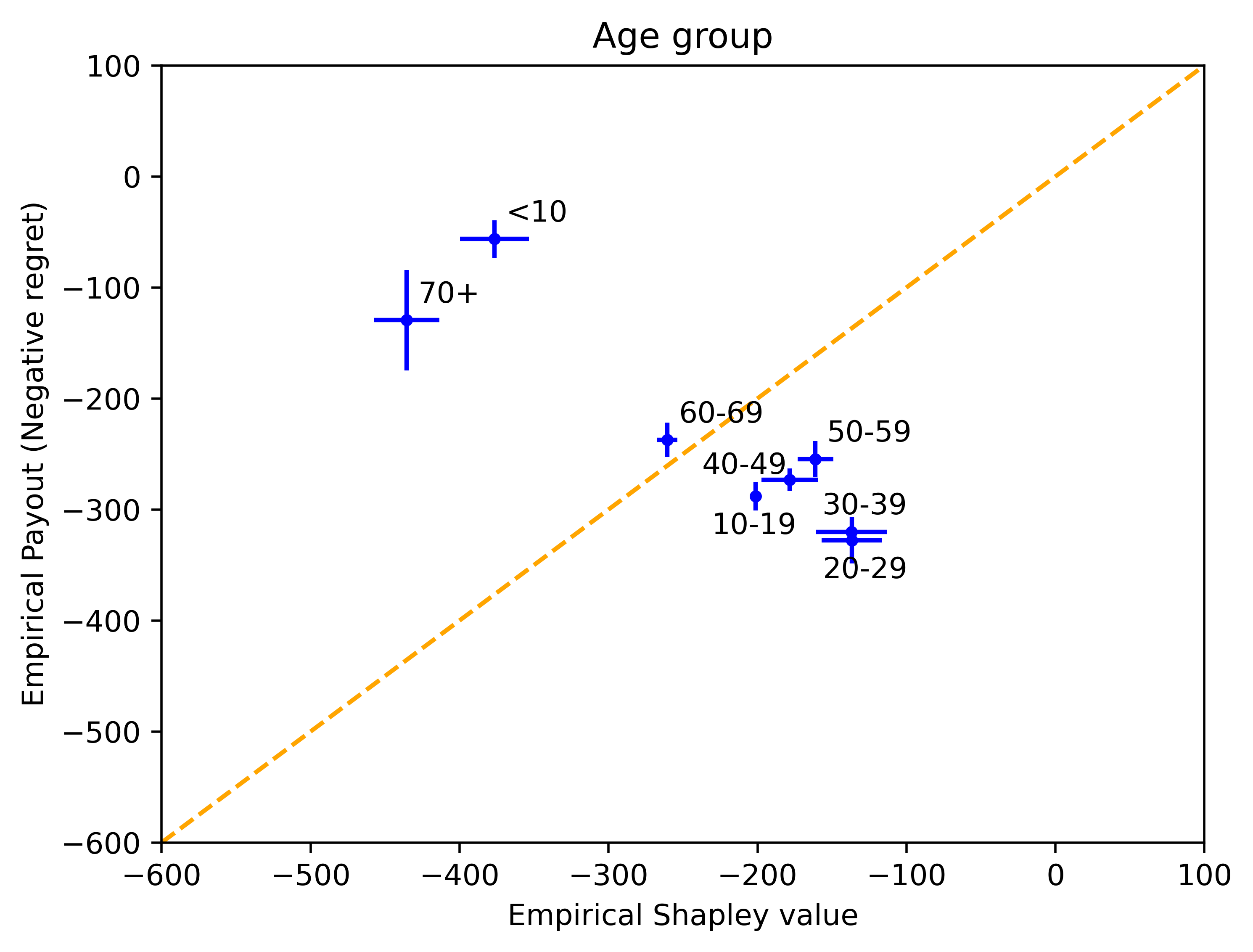}
        \caption{$(\iage, \mlu)$}\label{fig:ml-age}
    \end{subfigure}%
    \begin{subfigure}{.51\linewidth}
        \centering
        \includegraphics[width=\linewidth]{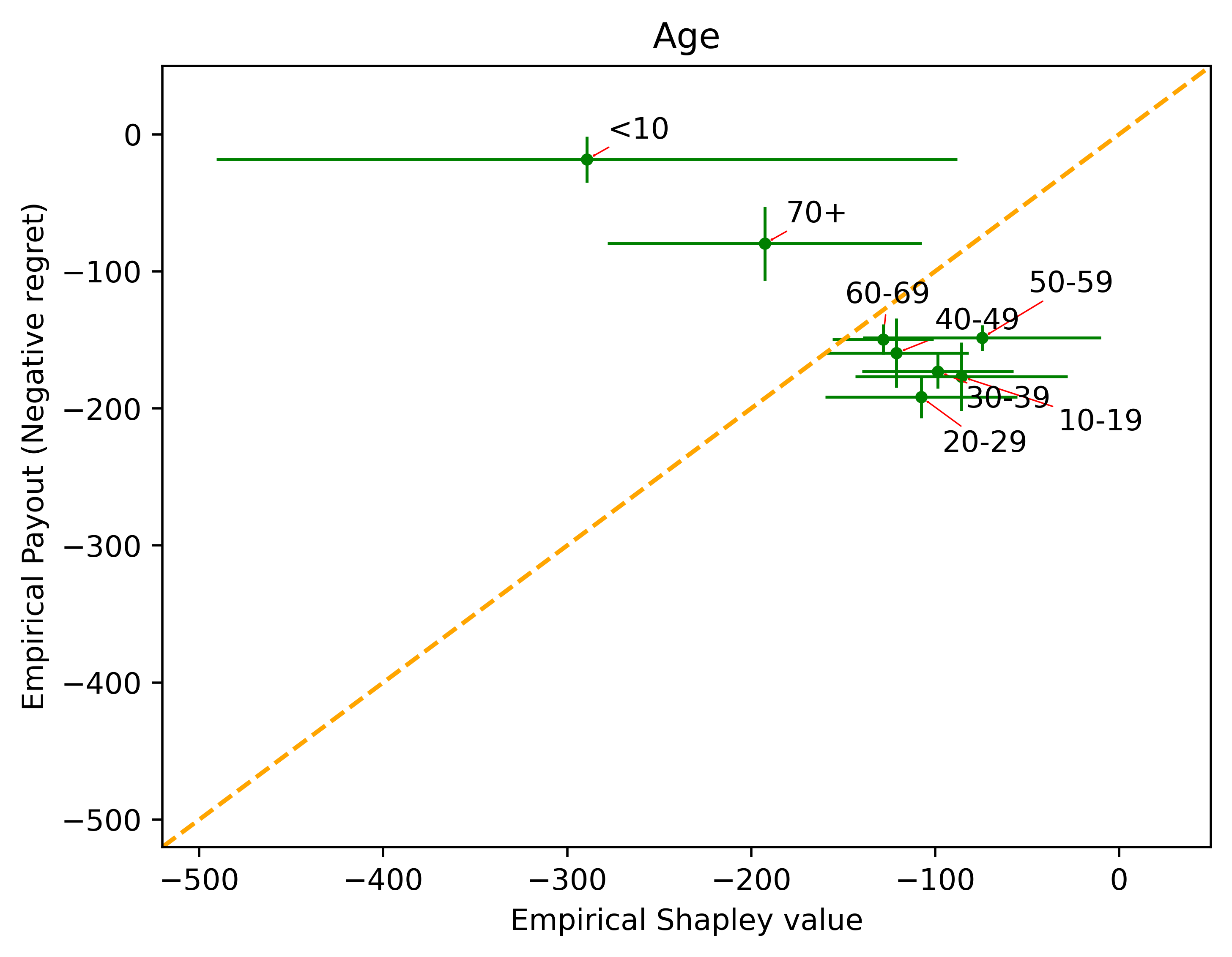}
        \caption{$(\iage, \greedy)$}\label{fig:ml-age-greedy}
    \end{subfigure}
    
    \begin{subfigure}{.49\linewidth}
        \centering
        \includegraphics[width=\linewidth]{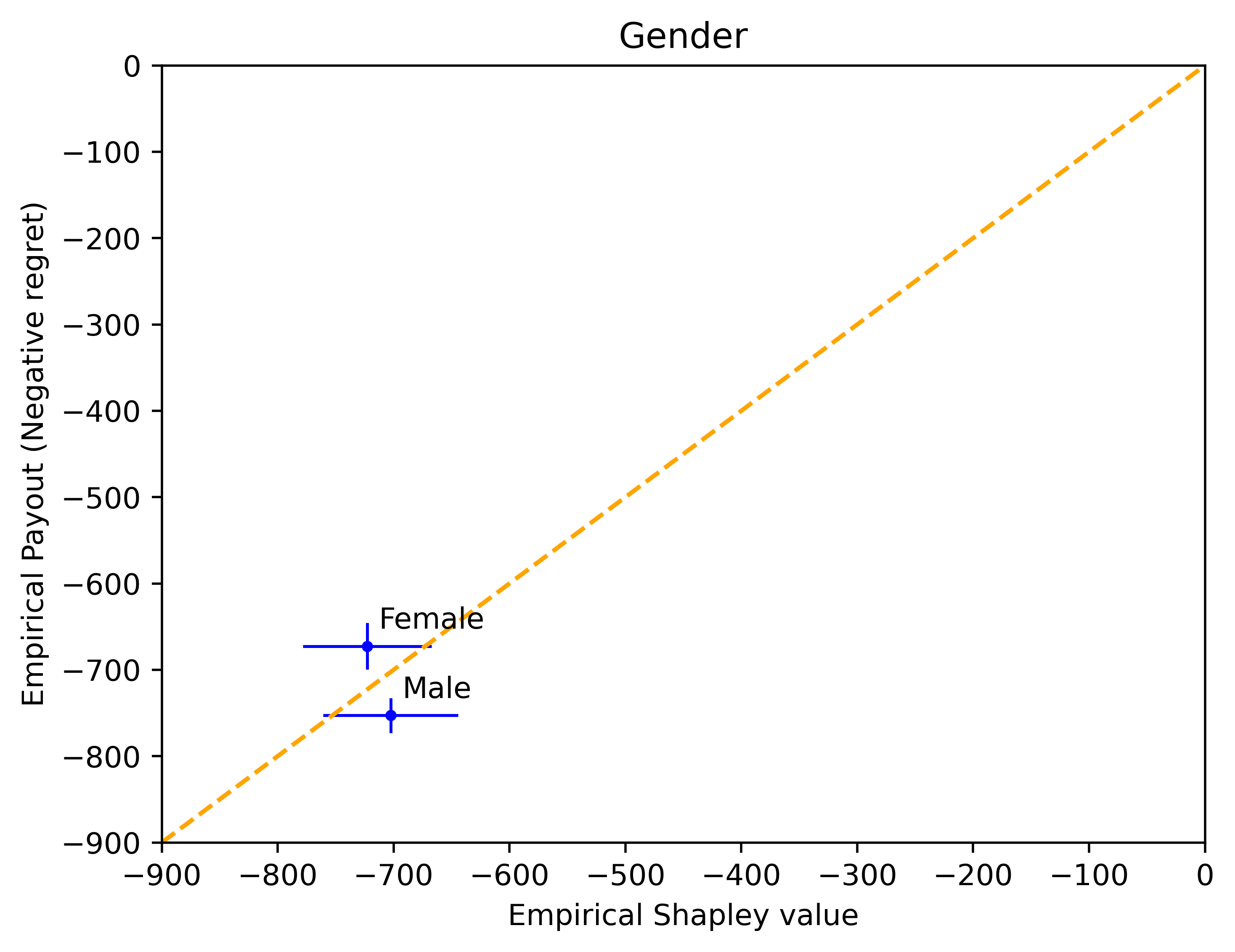}
        \caption{$(\igen, \mlu)$}\label{fig:ml-gender}
    \end{subfigure}
    \begin{subfigure}{.49\linewidth}
        \centering
        \includegraphics[width=\linewidth]{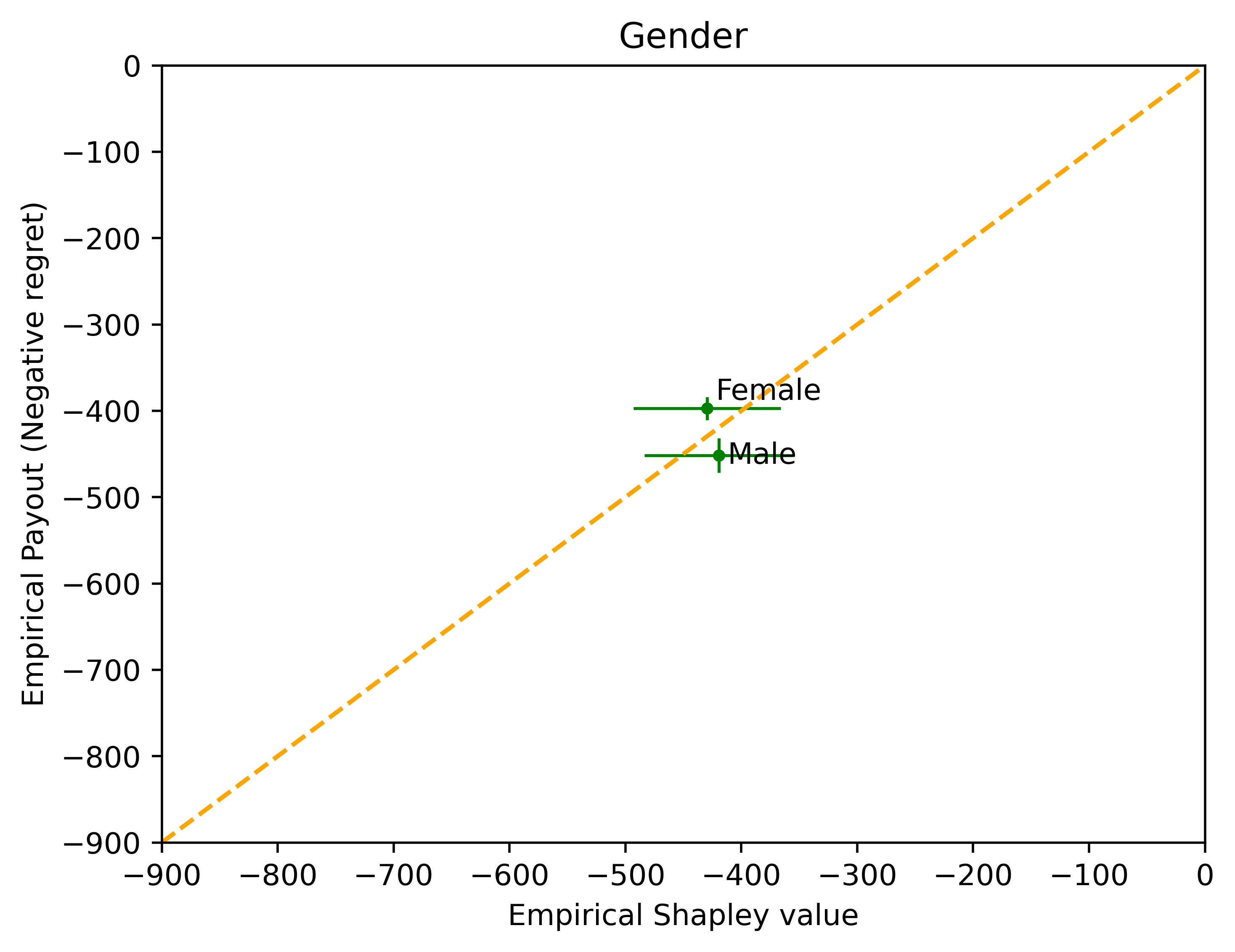}
        \caption{$(\igen, \greedy)$}\label{fig:ml-gender-greedy}
    \end{subfigure}
    \caption{\textbf{MovieLens experiments} - Instances based on classification by age and gender.}
    \label{fig:movielens-gender-age}
\end{figure}

In this subsection, we present the experimental results on MovieLens problem instances generated by classifying by gender and age group. Note the main paper presents experiments on instances based on classification by occupation and geography.

\cref{fig:movielens-gender-age} draws similar insights to the ones mentioned in the main paper. 
The payouts are close to the Shapley values empirically for several agents, but there exist outlier agents who disproportionately gain or contribute to the other agents.

On another note, the greedy algorithm appears to outpeform the LinUCB, and we attribute this to the inherent heterogeneity of the action space (the user representations, here).

\subsection{On satisfying \cref{assm:large-coalition-less-regret}}
\label{appn-subsec:on-satisfying-more-the-merrier}

We check if the Assumption---that the regret of an agent doesn't increase if more agents join the coalitoin---holds for our MovieLens experiments. 
We observe that it is mostly satisfied, with some rare cases where it is not. We present the results for $\iocc$ and $\igen$ problem instances with both algorithms $\mlu$ and $\greedy$.

In the $\iocc$ instance, there are $2816$ distinct agent-coalition pairs of $(a \in Q, Q \subseteq \agents)$ with $\cardinality{\agents} = 8$. Among them, all but $3$ (sim. $22$) pairs satisfy the Assumption when run with algorithm $\mlu$ (sim. $\greedy$), and the details of the pairs that do \emph{not} satisfy are plotted in \cref{tab:occupation} (sim. \cref{tab:occupation-greedy}).
All other pairs satisfy the Assumption and are not presented due to large volume of data.
Read agents from $0-7$ in order ['student', 'technical', 'management', 'creative', 'academic', 'business', 'healthcare', 'non-professional'].

\begin{table}
\caption{$(\iocc, \mlu)$}
\label{tab:occupation}
\begin{tabular}{|p{3cm}|c|c|c|c|c|}
\hline
\textbf{Coalition $Q$} & \textbf{Agent $a \in Q$} & \textbf{Regret $R^Q_a$} & \textbf{Sub-coalition $S$} & \textbf{Reg. $R^S_a$} & \textbf{$\nicefrac{R^Q_a}{R^S_a} \geq 1$} \\
\hline

\begin{filecontents*}{coop-content/mtm-table/occ-full-1-occupation.txt}
(0, 1, 3, 4, 6, 7)|1|354|18|(1, 3, 4, 6, 7)|352|47|1.006
(0, 1, 3, 4, 6, 7)|1|354|18|(1, 3, 4, 6, 7)|352|47|1.006
(2, 3, 4, 5, 6, 7)|5|208|23|(3, 4, 5, 6, 7)|208|14|1.0
(0, 1, 2, 3, 5, 6, 7)|5|189|29|(1, 2, 3, 5, 6, 7)|182|4|1.038
\end{filecontents*}
\csvreader[
  separator=pipe,
  late after line=\\ \hline
]{coop-content/mtm-table/occ-full-1-occupation.txt}{}%
{%
  \csvcoli & \csvcolii & \csvcoliii~$\pm$~\csvcoliv & \csvcolv & \csvcolvi~$\pm$~\csvcolvii & \csvcolviii
}
\end{tabular}
\end{table}

\begin{table}
\caption{$(\iocc, \greedy)$}
\label{tab:occupation-greedy}
\begin{tabular}{|p{3cm}|c|c|c|c|c|}
\hline
\textbf{Coalition $Q$} & \textbf{Agent $a \in Q$} & \textbf{Regret $R^Q_a$} & \textbf{Sub-coalition $S$} & \textbf{Reg. $R^S_a$} & \textbf{$\nicefrac{R^Q_a}{R^S_a} \geq 1$} \\
\hline

\begin{filecontents*}{coop-content/mtm-table/occ-greedy-full-1-occupation.txt}
(0, 2, 3, 6)|6|268|37|(2, 3, 6)|252|23|1.063
(0, 2, 3, 6)|6|268|37|(2, 3, 6)|252|23|1.063
(0, 3, 6, 7)|7|266|20|(3, 6, 7)|256|27|1.039
(2, 3, 5, 7)|7|230|15|(3, 5, 7)|226|15|1.018
(0, 1, 2, 5, 7)|5|181|25|(1, 2, 5, 7)|172|13|1.052
(0, 1, 4, 5, 7)|7|195|11|(1, 4, 5, 7)|176|19|1.108
(0, 1, 4, 6, 7)|6|205|49|(1, 4, 6, 7)|195|21|1.051
(0, 2, 5, 6, 7)|5|204|45|(2, 5, 6, 7)|184|20|1.109
(0, 4, 5, 6, 7)|6|243|13|(4, 5, 6, 7)|215|24|1.13
(1, 2, 4, 5, 6)|5|193|37|(2, 4, 5, 6)|176|14|1.097
(1, 4, 5, 6, 7)|6|227|20|(4, 5, 6, 7)|215|24|1.056
(0, 1, 2, 3, 6, 7)|6|179|25|(1, 2, 3, 6, 7)|165|27|1.085
(0, 1, 3, 4, 5, 6)|6|190|18|(1, 3, 4, 5, 6)|165|5|1.152
(0, 1, 3, 5, 6, 7)|5|158|10|(1, 3, 5, 6, 7)|152|19|1.039
(0, 3, 4, 5, 6, 7)|7|173|40|(3, 4, 5, 6, 7)|172|17|1.006
(1, 2, 3, 4, 6, 7)|3|210|19|(2, 3, 4, 6, 7)|198|20|1.061
(1, 3, 4, 5, 6, 7)|5|165|34|(3, 4, 5, 6, 7)|154|19|1.071
(1, 3, 4, 5, 6, 7)|6|201|10|(3, 4, 5, 6, 7)|184|54|1.092
(1, 3, 4, 5, 6, 7)|7|174|21|(3, 4, 5, 6, 7)|172|17|1.012
(0, 1, 2, 3, 5, 6, 7)|6|165|33|(1, 2, 3, 5, 6, 7)|145|21|1.138
(0, 1, 2, 4, 5, 6, 7)|6|156|26|(1, 2, 4, 5, 6, 7)|140|19|1.114
(1, 2, 3, 4, 5, 6, 7)|2|162|15|(2, 3, 4, 5, 6, 7)|162|13|1.0
(1, 2, 3, 4, 5, 6, 7)|7|156|18|(2, 3, 4, 5, 6, 7)|152|19|1.026
\end{filecontents*}
\csvreader[
  separator=pipe,
  late after line=\\ \hline
]{coop-content/mtm-table/occ-greedy-full-1-occupation.txt}{}%
{%
  \csvcoli & \csvcolii & \csvcoliii~$\pm$~\csvcoliv & \csvcolv & \csvcolvi~$\pm$~\csvcolvii & \csvcolviii
}
\end{tabular}
\end{table}

Next, we present the full result for $\igen$ instance with both algorithms. It can be seen that the Assumption is satisfied as in \cref{tab:gender,tab:gender-greedy}.
Read agents from $0-1$ in order ['Male', 'Female'].

\begin{table}
\caption{$(\igen, \mlu)$}
\label{tab:gender}
\begin{tabular}{|p{3cm}|c|c|c|c|c|}
\hline
\textbf{Coalition $Q$} & \textbf{Agent $a \in Q$} & \textbf{Regret $R^Q_a$} & \textbf{Sub-coalition $S$} & \textbf{Reg. $R^S_a$} & \textbf{$\nicefrac{R^Q_a}{R^S_a}$} \\
\hline

\begin{filecontents*}{coop-content/mtm-table/gender-full-1-gender.txt}
(0,)|0|1085|84|()|NA|NA|NA
(0,)|0|1085|84|()|NA|NA|NA
(1,)|1|1106|48|()|NA|NA|NA
(0, 1)|0|752|20|(0,)|1085|84|0.693
(0, 1)|1|672|27|(1,)|1106|48|0.607
\end{filecontents*}
\csvreader[
  separator=pipe,
  late after line=\\ \hline
]{coop-content/mtm-table/gender-full-1-gender.txt}{}%
{%
  \csvcoli & \csvcolii & \csvcoliii~$\pm$~\csvcoliv & \csvcolv & \csvcolvi~$\pm$~\csvcolvii & \csvcolviii
}
\end{tabular}
\end{table}

\begin{table}
\caption{$(\igen, \greedy)$}
\label{tab:gender-greedy}
\begin{tabular}{|p{3cm}|c|c|c|c|c|}
\hline
\textbf{Coalition $Q$} & \textbf{Agent $a \in Q$} & \textbf{Regret $R^Q_a$} & \textbf{Sub-coalition $S$} & \textbf{Reg. $R^S_a$} & \textbf{$\nicefrac{R^Q_a}{R^S_a}$} \\
\hline

\begin{filecontents*}{coop-content/mtm-table/gender-greedy-full-1-gender.txt}
(0,)|0|673|96|()|NA|NA|NA
(0,)|0|673|96|()|NA|NA|NA
(1,)|1|683|53|()|NA|NA|NA
(0, 1)|0|451|20|(0,)|673|96|0.67
(0, 1)|1|397|13|(1,)|683|53|0.581
\end{filecontents*}
\csvreader[
  separator=pipe,
  late after line=\\ \hline
]{coop-content/mtm-table/gender-greedy-full-1-gender.txt}{}%
{%
  \csvcoli & \csvcolii & \csvcoliii~$\pm$~\csvcoliv & \csvcolv & \csvcolvi~$\pm$~\csvcolvii & \csvcolviii
}
\end{tabular}
\end{table}

% \input{coop-content/crap}

% \input{coop-content/apdx-mapping-model-to-application}

% \parttoc % Insert the appendix TOC

\end{document}